\documentclass[opre,nonblindrev]{OR-src/informs4}

\RequirePackage{tgtermes}
\RequirePackage{newtxtext}
\RequirePackage{newtxmath}
\RequirePackage{endnotes}

\OneAndAHalfSpacedXI

\TheoremsNumberedThrough 
\EquationsNumberedThrough

\newcommand{\PP}{\mathbb{P}}
\newcommand{\RR}{\mathbb{R}}

\newcommand{\myprob}[2][]{\mathbb{P}\ifthenelse{\not\equal{}{#1}}{_{#1}}{}\!\left[{\def\givenn{\middle|}#2}\right]}

\newcommand{\primed}{^{\dagger}}

\usepackage{float}
\usepackage{mathtools}
\usepackage{graphicx}
\usepackage[titletoc]{appendix}
\usepackage{xcolor}
\usepackage{comment}
\usepackage{nicefrac}
\usepackage{xfrac}
\usepackage{enumitem}
\usepackage{ifthen}
\usepackage{stackengine}
\usepackage{csquotes}
\usepackage[colorlinks=true, linkcolor=blue!70!black, citecolor=blue!70!black, urlcolor=blue, breaklinks=true]{hyperref}
\usepackage{cleveref}
\usepackage{tikz}
\usetikzlibrary{arrows.meta,patterns}
\usepackage{pgfplots}
\pgfplotsset{compat=newest}
\usepgfplotslibrary{groupplots}
\usepackage{aliascnt}
\usepackage{setspace}
\hypersetup{
	colorlinks=true,
	linkcolor=blue!70!black,
	citecolor=blue!70!black,
	linkbordercolor={white}
}
\usepackage{algorithm}
\usepackage{algpseudocode}
\usepackage{booktabs}
\usepackage{thm-restate}
\usepackage{tcolorbox}
\DeclareRobustCommand{\myboxtwo}[2][gray!15]{
\begin{tcolorbox}[ 
        colback=white,      
        colframe=gray,  
        boxrule=0.2pt,      
        arc=2pt,outer arc=2pt,
        left=12pt,
        right=12pt,
        top=5pt,
        bottom=5pt,
        width=1.07\linewidth,
        enlarge left by=-0.55cm,
        before upper=\renewcommand{\baselinestretch}{1.3}\selectfont,
        after upper=\normalfont
        ]
 #2
 \end{tcolorbox}
}

\usepackage[suppress]{color-edits}
\addauthor{wt}{red}
\addauthor{pp}{blue}
\addauthor{ai}{magenta}
\addauthor{jpred}{red}

\usepackage{float} 
\usepackage{algorithm}
\usepackage{algpseudocode}
\algrenewcommand{\algorithmiccomment}[1]{{\ttfamily/* #1 */}}

\Crefname{algocf}{Algorithm}{Algorithms}
\crefname{claim}{claim}{claims}
\crefname{assumption}{assumption}{assumptions}
\Crefname{assumption}{Assumption}{Assumptions}

\newcommand{\xhdr}[1]{\vspace{6pt}\noindent{\bf {#1.}}}

\newcommand{\bs}[1]{\boldsymbol{#1}}

\newcommand{\condition}{\;\middle|\;}

\newcommand{\myA}{{\cal A}}

\newcommand{\myexp}[1]{\exp\left(#1\right)}
\newcommand{\prob}[2][]{\mathbb{P}\ifthenelse{\not\equal{}{#1}}{_{#1}}{}\!\left({\def\givenn{\middle|}#2}\right)}

\newcommand{\armf}[1]{\mathcal{A}^f_{#1}}
\newcommand{\armg}[1]{\mathcal{A}^g_{#1}}
\newcommand{\optf}{1^f}
\newcommand{\optg}{1^g}
\newcommand{\myDelta}[1]{\widetilde{\Delta}_{#1}}
\newcommand{\Deltaf}[1]{\Delta^f_{#1}}
\newcommand{\Deltag}[1]{\Delta^g_{#1}}
\newcommand{\empiricalmeanf}[1]{\hat{\mu}^f_{#1}}
\newcommand{\empiricalmeang}[1]{\hat{\mu}^g_{#1}}
\newcommand{\numround}{\widetilde{L}}
\newcommand{\roundlen}[2]{m^{#1}_{#2}}
\newcommand{\marginroundf}[1]{L_{#1}^f}
\newcommand{\marginroundg}[1]{L_{#1}^g}
\newcommand{\fgdiff}{\delta}
\newcommand{\epochfactor}{\alpha}
\newcommand{\epochfactorg}{\beta}

\newcounter{wtcounter}

\newcounter{jpcounter}

\newcommand{\expect}[2][]{\mathbb{E}\ifthenelse{\not\equal{}{#1}}{_{#1}}{}\!\left[{\def\givenn{\middle|}#2}\right]}

\newcommand{\cc}[1]{\ensuremath{\mathsf{#1}}}

\newcommand{\Reg}[2][]{\text{\bf REG}\ifthenelse{\not\equal{}{#1}}{_{#1}}{}\!\left[{\def\givenn{\middle|}#2}\right]}
\newcommand{\Rega}[2][]{\text{\bf REG}_1\ifthenelse{\not\equal{}{#1}}{_{#1}}{}\!\left[{\def\givenn{\middle|}#2}\right]}
\newcommand{\Regb}[2][]{\text{\bf REG}_2\ifthenelse{\not\equal{}{#1}}{_{#1}}{}\!\left[{\def\givenn{\middle|}#2}\right]}
\newcommand{\Regc}[2][]{\text{\bf REG}_3\ifthenelse{\not\equal{}{#1}}{_{#1}}{}\!\left[{\def\givenn{\middle|}#2}\right]}

\newcommand{\Rew}[2][]{\text{\bf REW}\ifthenelse{\not\equal{}{#1}}{_{#1}}{}\!\left[{\def\givenn{\middle|}#2}\right]}

\newcommand{\OPT}{\cc{OPT}}

\newcommand{\armSpace}{\myA}
\newcommand{\armNum}{K}

\newcommand{\feaVec}{\boldsymbol{x}}
\newcommand{\feaSpace}{\mathcal{X}}
\newcommand{\feaDimen}{d}

\newcommand{\rewFn}{f}
\newcommand{\rew}{R}
\newcommand{\feaDist}{\PP}

\newcommand{\pickArm}{I}
\newcommand{\switchRound}{N}
\newcommand{\unknownVec}{\boldsymbol{\theta}}

\newcommand{\commitArm}{J}
\newcommand{\rewFnCommit}{g}
\newcommand{\lipschitzConst}{L_\rewFnCommit}

\newcommand{\policy}{\pi}
\newcommand{\expPolicy}{\pi_{\cc{exp}}}
\newcommand{\comPolicy}{\pi_{\cc{com}}}
\newcommand{\commitPolicy}{\pickArm_{\switchRound+1}}

\newcommand{\optArm}{I^*}

\renewcommand{\emph}[1]{\textit{#1}}

\newcommand{\history}{\mathcal{H}}
\newcommand{\timeHorizon}{T}

\newcommand{\UCB}{\text{UCB}}

\newcommand{\RAEC}{\text{RAEC}}
\newcommand{\ROSCOC}{\text{ROSCOC}}

\newcommand{\instance}{\mathcal{I}}

\crefname{appendix}{Appendix}{Appendices}
\Crefname{appendix}{Appendix}{Appendices}

\makeatletter
\renewcommand{\theARTICLETOP}{}
\renewcommand{\theARTICLEABSTRACT}{%
  \par\HOOKb\vspace*{10pt}%
  \noindent\begin{minipage}[t]{\textwidth}
    \parindent1em\ABSfont
    \noindent\theABSTRACT\par
    \vskip5pt
    \theFUNDING
    \theKEYWORDS
    \theSUBJECTCLASS
    \theAREAOFREVIEW
    \theMSCCLASS
    \theORMSCLASS
    \if@BLINDREV\else\theHISTORY\fi
    \noindent\hrulefill
  \end{minipage}\par
}
\makeatother
\RRHSecondLine{}
\LRHSecondLine{}

\usepackage{natbib}
 \bibpunct[, ]{(}{)}{,}{a}{}{,}%
 \def\bibfont{\small}%
 \def\bibsep{\smallskipamount}%

\makeatletter
\DeclareRobustCommand{\qed}{%
  \ifmmode 
  \else \leavevmode\unskip\penalty9999 \hbox{}\nobreak\hfill
  \fi
  \quad\hbox{\qedsymbol}}
\newcommand{\qedsymbol}{\openbox}
\renewenvironment{proof}[1][\proofname]{\par
  \normalfont
  \topsep6\p@\@plus6\p@ \trivlist
  \item[\hskip\labelsep\itshape
    #1.]\ignorespaces
}{%
  \qed\endtrivlist
}
\newcommand{\proofname}{Proof}
\makeatother

\MANUSCRIPTNO{}

\begin{document}

\RUNAUTHOR{Jiang, Tang, and Zhang}
\RUNTITLE{Experimentation and Commitment under Reward Shifts}

\TITLE{Experimentation and Commitment under Reward Shifts}

\ARTICLEAUTHORS{
\AUTHOR{Puping Jiang}
\AFF{Shanghai Jiao Tong University, \EMAIL{jiang.p@wustl.edu}}
\AUTHOR{Wei Tang}
\AFF{The Chinese University of Hong Kong, \EMAIL{weitang@cuhk.edu.hk}}
\AUTHOR{Renyu (Philip) Zhang}
\AFF{The Chinese University of Hong Kong, \EMAIL{philipzhang@cuhk.edu.hk}} 
} 

\ABSTRACT{
Decision-makers in learning environments face a dilemma when their short-term optimal actions may not favor their long-term benefits the most. To understand the fundamental tradeoff behind the dilemma, we study adaptive experimentation with post-commitment reward shifts. During an experiment phase, the decision-maker may adaptively test multiple options; during a subsequent commitment phase, the decision-maker must commit to a single option, whose reward may differ from its pre-commitment reward. 
We propose the Reserved Arm Eliminations for Commitment (RAEC) algorithm, which reserves a predetermined portion of the experiment phase to identify the best post-shift option while using the remaining rounds to minimize short-run regret. We establish regret upper bounds for {RAEC} across all parameter regimes and matching minimax lower bounds, providing a tight characterization of the cost of balancing short-term performance and long-term commitment. 
A key implication is that deciding in advance how much of the experiment phase to reserve for the commitment decision is sufficient to achieve the best possible worst-case regret rate; adapting this amount as more data are observed does not improve the rate.
We further study extensions with structural knowledge of reward shifts and with concave commitment rewards and portfolio choice. 
Numerical experiments confirm that our proposed algorithms achieve the regret predicted by our theory and outperform other baselines.
}%





\maketitle


\section{Introduction}
\label{sec:intro}


Balancing short‐term benefits against long‐term goals, especially in an uncertain decision-making environment, has long challenged decision-makers.
The dilemma is especially acute when the decision that maximizes near‐term rewards conflicts with what is most desirable in the long run.
It is often the case that the decision-maker can experiment with multiple options to identify which might be most suitable for the current decision task, yet they must eventually commit to an option for long-run goals.

Practitioners and researchers alike have debated how to reconcile these competing objectives in applications ranging from online platform experimentation to technology development and product manufacturing. 
For example, online platforms have experienced significant regulatory environment policy shifts with the EU's General Data Protection Regulation (GDPR), which was announced in 2016 and enacted on May 25, 2018 \citep{JSG-23,ACS-23}.
During the two-year implementation window before GDPR took effect,
a platform needs to test out new strategies, preparing for a platform-wide transition after the environment shift.
Examples of pre-GDPR testing and staged rollouts include Criteo's pre-GDPR consent-interface tests \citep{tiku2018wired}, IAB Europe's launch of the Transparency \& Consent Framework in April 2018 \citep{iabeurope_tcf_v1_2018}, and Google's March 2018 GDPR policy updates \citep{google2018gdpr_ads_policy}.
Ranging from environmental regulation (e.g., the Clean Air Act; EU ETS) to trade and geopolitical tensions, the anticipated policy shifts have also repeatedly altered firms' supply chain, production, and investment decisions.
For example, \cite{SW-18} show U.S.\ environmental regulation (Clean Air Act) drove large changes in manufacturing emissions, which is clear evidence that actual policy shifts reshape firm manufacturing decisions; \cite{CMMW-25} provide firm-level evidence that carbon pricing under the EU ETS changes production choices and investments.
Firms have to experiment with substitute production plans in order to achieve a resilient transition \citep{T-09, AIS-17}.

At a high level, the operational problem that firms encounter in the environment discussed above has the following structure:
before the anticipated environment shift, firms would typically reserve a relatively small portion of their capacity (for online platforms, this represents a portion of users and computing resources; for manufacturing firms, this represents a portion of production capacity and material resources) for experiments.
The experiment capacity can be used not only to test strategies that might attain promising performance after the environment shift, but also to explore and roll out strategies that could deliver high payoffs under the current operational environment.
While the experiment capacity is flexible for strategy switches, the strategy running on the remaining capacity is generally kept stable due to cost and other practical concerns.
For example, online platforms prefer a smooth and stable experience for the majority of the customers. 
A whole-capacity-wide strategy transition, however, is inevitable after the environmental shift, since the old strategy could become highly unfavorable or even infeasible in the new environment.
The key decisions firms have to make are the experiment design within the experiment capacity and the new strategy to commit to after the shift. 
The decision maker's objective that we focus on in this paper is the aggregated payoff generated by the experiment capacity before the shift and the new committed strategy run on the remaining capacity after the shift.
The fundamental question we aim to address is:
\begin{displayquote} 
\myboxtwo{\emph{What is the rule of thumb for the optimal experiment design when there is an anticipated environment shift?}} 
\end{displayquote}
Such a decision dilemma is becoming more and more relevant as many industrial sectors are facing more and more anticipated economic and political environment shifts, and aim to build more resilient and sustainable supply chains \citep{DR-2019}.
Earlier academic work derives qualitative observations on firms' optimal exploration effort under changing environments (see, e.g., \citealp{PL-2012}).
Our paper advances the understanding of this problem from an adaptive learning perspective via a parsimonious model.
Before we briefly mention the high-level framework of our model, we elaborate more on the aforementioned examples here.
\begin{example}[Platform regulation shift: GDPR example]
The EU announced GDPR on April 27th, 2016, which was scheduled to be enacted on May 25th, 2018. 
During this implementation window, an online platform could experiment with candidate interface designs on a relatively small group of users (e.g., beta users). 
Empirical evidence has reported significant industry-wide strategy changes following the enactment of GDPR (see, e.g., \citealp{PBBK2022,ACS-23}).
Routing a share of beta traffic to a candidate design yields a random outcome, such as \((\textsc{conversion-rate},\textsc{privacy-level})\), whose distribution is tied to the design. 
The platform's payoff can be simplified as \(\textsc{profit}-\textsc{legal-cost}\), where \(\textsc{profit}\) and \(\textsc{legal-cost}\) depend on conversion rate and privacy level, respectively. 
At the enforcement date, the legal-cost function may change \citep{wsj_gdpr_ccpa_privacy_strategies}, so the same privacy level may imply a different cost than before. 
Thus, a design that performs well during experimentation may not be the best design after the regulation takes effect, which creates a tradeoff between short-term gains during the window and learning which design to adopt for the long run.
\end{example}

There are two key features behind these examples:
$(i).$ \textit{Two-phase horizon.}
The time horizon consists of two phases, an \textit{experiment phase} and a \textit{commitment phase}.
The commitment date is known in advance and is typically marked by an exogenous environment shift.\footnote{In other settings, the environment-shift date can instead be interpreted as a strategy commitment date. 
For example, semiconductor foundries use a risk-production window to test process recipes before the scheduled launch of volume production. Since risk production also generates financial returns, foundries trade off short-term gains against learning which process performs well in volume production, where the reward structure may differ. 
TSMC announced in July 2026 that A14 risk production would begin in 2027, followed by volume production in 2028 \citep{tsmc_a14_schedule_2026}. 
See \Cref{exam:manufacture} in Appendix \ref{apx:example} for additional manufacturing examples.}
Before the commitment date, experiments can be conducted due to the flexible nature of the experiment capacity.
A long-term commitment must be made after the commitment date due to the inflexible nature of the global strategy transition, the construction of mass production capacity, etc.
$(ii).$ \textit{Predictable reward shift.}
The reward shift after the policy environment shift is largely predictable.
For example, the legal costs associated with a given interface design are quantifiable in advance under GDPR, as long as the regulations are specified clearly.
Similarly, the taxation, export restrictions on key components, fines on carbon emissions, cost changes in volume production, etc., are also quantifiable in advance once the laws are announced and contracts are specified.

To capture these features, 
we consider an adaptive experiment problem with post-commitment reward shift: 
The decision-maker faces a sequential decision-making problem over a total of $\timeHorizon$ periods and must choose from  $\armNum$ candidate treatments.
The whole process consists of two phases:
In the first phase (the first $\switchRound$ rounds referred to as the {\em experiment phase}), she can experiment with any different treatments.
The decision-maker initially has no prior knowledge of the mean reward of each treatment but can adaptively choose one of the treatments to observe a corresponding random outcome generated from this treatment and collect the corresponding reward.
Importantly, this experimentation budget $\switchRound$ is exogenously determined by the environment.
In the second phase (the remaining $\timeHorizon - \switchRound$ rounds, referred to as the {\em commitment phase}), the decision-maker must commit to a single treatment and implement it throughout this phase.
However, due to anticipated environment changes after $\switchRound$ periods, 
a reward shift may occur for each treatment in the commitment phase.
As a result, the treatment that appeared to be optimal during the experiment phase may become suboptimal in the commitment phase.\footnote{Our framework can also capture settings in which each treatment has an additional attribute that matters only for the final commitment decision. For example, in adaptive clinical trials, the decision-maker may care about both the cumulative welfare of patients participating in the trial and the performance of the treatment selected after the trial; see \Cref{exam:clinical} in Appendix \ref{apx:example}.}

Under this framework, our main research question can now be stated as:
\begin{displayquote} 
\myboxtwo{\emph{What is the optimal way to address the exploration-exploitation dilemma in the presence of such a post-commitment reward shift?}} 
\end{displayquote}
We seek to answer the above question in two aspects: 
(1) whether an efficient online learning algorithm can achieve provable performance, and (2) what the optimal performance guarantee is for any online algorithm.

We measure the performance of online learning algorithms by the notion of \textit{Regret}, which is the opportunity cost of our algorithm compared against the optimal clairvoyant strategy -- one that knows the best treatments for both the experiment and commitment phases.
Formally, it is defined as the difference between the total reward accumulated by an optimal clairvoyant strategy and the reward achieved by our online learning algorithm.

\subsection{Main Results}
\label{sec:contribution}
Our first main result introduces a simple yet effective online learning algorithm, termed the \textit{Reserved Arm Eliminations for Commitment} algorithm ({\RAEC}), which achieves provable performance across all parameter regimes of the experimentation budget $\switchRound$.
At a high level, {\RAEC} operates as an elimination-based algorithm with phased learning.
During the experiment phase, it adaptively eliminates suboptimal arms using increasingly stringent hypothesis tests.
In particular, in Algorithm {\RAEC}, a predetermined portion of the experiment phase is reserved for identifying the optimal arm for the commitment phase, while the remaining rounds will focus on standard regret minimization.
A key feature of Algorithm {\RAEC} is that the confidence intervals used for eliminating suboptimal arms are constructed to account for the commitment constraint after $\switchRound$ rounds.
We use (pseudo) \textit{regret} to measure the performance of our algorithm, which is defined as the difference between the expected total reward of a clairvoyant benchmark that knows the optimal arms in advance and our algorithm.
Throughout this work, we use $\widetilde O$ and $\widetilde \Omega$ to denote upper and lower bounds on the growth rate (up to logarithmic factors), and $\widetilde \Theta$ to characterize matching rates (up to logarithmic factors).
{\RAEC}'s simplicity enables us to derive a full spectrum of regret upper bounds across different parameter settings ($\armNum\ll\switchRound,\timeHorizon-\switchRound$):
\begin{itemize}
    \item 
    {\em Short experiment scenario} (i.e., 
    $\switchRound\le (\timeHorizon - \switchRound)^{\sfrac{2}{3}}(\armNum\log(\timeHorizon - \switchRound))^{\sfrac{1}{3}}$): 
    The regret bound is $\widetilde{O}(\sqrt{\sfrac{\armNum\timeHorizon^2}{\switchRound}})$;
    \item 
    \textit{Balanced scenario} (i.e., 
    $\switchRound\geq(\timeHorizon - \switchRound)^{\sfrac{2}{3}}(\armNum\log(\timeHorizon - \switchRound))^{\sfrac{1}{3}}$ and $(\timeHorizon - \switchRound)\geq\armNum^{\sfrac{1}{4}}\timeHorizon^{\sfrac{3}{4}}/(\log(\timeHorizon - \switchRound))^{\sfrac{1}{2}}$):
    The regret bound is $\widetilde{O}(\armNum^{\sfrac{1}{3}}(\timeHorizon-\switchRound)^{\sfrac{2}{3}})$;
    \item 
    \textit{Short commitment scenario} (i.e.,  $(\timeHorizon - \switchRound)<\armNum^{\sfrac{1}{4}}\timeHorizon^{\sfrac{3}{4}}/(\log(\timeHorizon - \switchRound))^{\sfrac{1}{2}}$):
    The regret bound is $\widetilde{O}(\sqrt{\armNum\timeHorizon})$.
\end{itemize}



\begin{table}[h] 
\begin{center}
    \renewcommand{\arraystretch}{2.2}
    \setlength{\tabcolsep}{3mm}{
    \begin{tabular}{c|c|c|c}
    \toprule 	 %
    Experiment budget $\switchRound$ & $\switchRound \lesssim \armNum^{\sfrac{1}{3}}\timeHorizon^{\sfrac{2}{3}}$ & $\armNum^{\sfrac{1}{3}}\timeHorizon^{\sfrac{2}{3}} \lesssim \switchRound \lesssim \timeHorizon-\armNum^{\sfrac{1}{4}}\timeHorizon^{\sfrac{3}{4}}$  & $\switchRound \gtrsim \timeHorizon-\armNum^{\sfrac{1}{4}}\timeHorizon^{\sfrac{3}{4}}$                        \\ \hline
    Optimal regret $\Reg{\switchRound, \timeHorizon}$  & $\displaystyle \widetilde \Theta\left(\sqrt{\frac{\armNum\timeHorizon^2}{\switchRound}}\right)$               & $\widetilde  \Theta(\armNum^{\sfrac{1}{3}}\left(\timeHorizon-\switchRound\right)^{\sfrac{2}{3}})$                        & $\widetilde\Theta(\sqrt{\armNum\timeHorizon})$                       \\   \bottomrule
    \end{tabular}
    }
\end{center}
\caption{The optimal regret for different parameter regimes. Here, we use $A \lesssim B$, $A\gtrsim B$ to denote $A = \widetilde O(B)$, $A = \widetilde\Omega(B)$, respectively.}
    \label{table:opt regret}
\end{table}
Beyond the algorithmic results, we establish an information-theoretic lower bound that matches the upper bound achieved by {\RAEC}, which proves that the regret bound cannot be further improved by other algorithms (in a certain sense). 
We define the instance-independent regret lower bound (i.e., minimax lower bound) across different parameter regimes as follows:
\begin{itemize}
    \item 
    \textit{Short experiment scenario}:
    The experiment phase is relatively short, and the optimal strategy is to focus on pure exploration. The minimax lower bound is $\Omega(\sqrt{\armNum\timeHorizon^2/\switchRound})$;
    \item 
    \textit{Balanced scenario}:
    The experiment phase is sufficiently long, but so is the commitment phase, which makes this case particularly more interesting. Here, the minimax lower bound is $\Omega(\armNum^{\sfrac{1}{3}}(\timeHorizon-\switchRound)^{\sfrac{2}{3}})$;
    \item 
    \textit{Short commitment scenario}:
    The commitment phase is short, making the problem resemble a standard $\armNum$-armed bandit setting. The minimax lower bound is the well-known $\Omega(\sqrt{\armNum\timeHorizon})$.
\end{itemize}
Our results highlight that while asymptotic optimality (sublinear regret) remains achievable across different parameter regimes of $\switchRound$, the problem is inherently more challenging due to the reward shift between the experiment and commitment phases. 
Interestingly, the regret lower bound grows more slowly in $\armNum$ compared to traditional MAB problems. This is because the presence of a reward shift encourages the optimal strategy to explore more arms in the experiment phase. Consequently, the regret lower bound becomes less sensitive to increases in $\armNum$.
In summary, our work provides a tight characterization of the regret landscape in adaptive experimentation with post-commitment reward shift. 
Our results provide a quantitative framework for decision-makers balancing the short-term vs.\ long-term trade-off while offering guidance on how much effort should be invested in exploring optimal solutions for long-term gains and what costs can be expected in the process.
Another interesting takeaway is that our algorithm predetermines the effort allocated to explore good arms in the commitment phase, and it is shown to be good enough.
More sophisticated methods that adaptively determine the exploration effort while learning the instance structure will not fundamentally improve the performance further.


Finally, we advance our discussion to two extensions.
The first extension is that, with a more nuanced ex-ante understanding of the reward shift structure, one can improve regret performance.
Our results clarify that, in terms of regret minimization, correctly identifying the ranking-changing effect in reward shifts is far more critical than specifying the absolute magnitude of these shifts.
The second extension is that, instead of committing to a single arm, the decision-maker can commit to a portfolio of arms in the commitment phase, and the total reward in the commitment phase is a concave function of the sum of the observed outcomes.
This model strictly generalizes our baseline model.
We propose a new algorithm, \textit{Reserved Online Stochastic Convex Optimization for Commitment} (ROSCOC), which, like {\RAEC}, allocates a predetermined portion of the experiment phase to learning the optimal commitment decision.
Our analysis shows that ROSCOC achieves a similar regret upper bound as in \Cref{table:opt regret}.
Although this generalized problem appears to be much more complicated, we demonstrate that the effectiveness of the predetermined exploration effort allocation for the commitment phase still applies, and the problem is not fundamentally harder than the baseline problem.

\subsection{Related Work}

Multi-armed bandits (MAB) is a classical framework to study how to balance exploitation and exploration in sequential decision-making problems.
Two archetypal objectives, regret minimization and best-arm identification, have been extensively studied in the MAB literature. 
Seminal works on regret minimization include the class of Upper Confidence Bound (UCB) type algorithms \citep{ACF-02}, and the class of Thompson sampling (TS) algorithms \citep{AG-12,RV-14,AG-17}. 
\cite{LR-85} characterize asymptotic lower bounds on the expected regret for any consistent online algorithm in the MAB problem.

Algorithms designed for regret minimization typically explore different options without committing to a single arm. However, in practice, many applications may prefer a commitment to action instead of continuous exploration. 
In light of this motivation, many works have been devoted to studying best arm identification. 
In this line of work, two settings are considered: 
(i) a fixed-budget setting: given a time budget $\timeHorizon$, the decision-maker aims to maximize the probability of finding the optimal arm in at most $\timeHorizon$ steps \citep{AB-10,BWV-13,KKS-13}; 
(ii) a fixed-confidence setting: given a confidence level $\delta > 0$, the decision-maker aims to find the optimal item with a probability of at least $1-\delta$ in the smallest number of time steps \citep{KK-13}.
\cite{KCG-16} examine the lower bound complexity of both settings.
Beyond identifying a single best arm, \citet{ZCL-14,CCZZ-17} study the identification of multiple good arms under PAC criteria and characterize the corresponding sample complexity.
Unlike the algorithms for regret minimization, the best-arm identification algorithms do not account for the regret incurred during exploration.
Related to the problem of best-arm identification, \citet{LFH-25} study large-scale ranking and selection and establish the sample optimality of simple greedy procedures.
Recently, pure exploration problems are also studied under operational contexts, e.g., \citet{SS-26} study pure exploration for multi-product pricing with approximate demand models.

To mitigate the drawbacks of regret minimization and best-arm identification, more recently, many works (see, e.g., \citealp{BJM-11,DNCP-19,KIZ-23,ZY-23,ZCT-23,SW-23,SZZ-23,QR-24,YTJ-24}) have started to explore designing algorithms that can achieve best-of-both-worlds guarantees in various unified frameworks. 
However, these works generally either treat the commitment time as a decision variable or evaluate the experiment and commitment phases using the same underlying arm rewards.
Among these works, a particularly relevant work is by \cite{BJM-11}, where they also consider that the experimenter needs to stop the experimentation within a budget, and the reward earned after the commitment is scaled up by a constant factor.
In other words, the optimal arm remains the same before and after commitment.
Our work considers a more general reward shift structure to change the ranking of the arms, so the optimal arm in the experiment phase may differ from the optimal arm in the commitment phase.
\cite{QR-24} allow the experimenter to adaptively choose when to stop and commit and consider different within- and post-experiment costs. Their post-experiment target, however, remains the treatment with the highest underlying mean outcome, whereas our reward shift may change the optimal treatment after commitment.
Conceptually, our work also relates to recent research on cross-learning (e.g., \citealp{FLS-26}), but we abstract away from the operational details and focus on investigating the fundamental trade-offs between short-term exploration and long-term commitment.
Relatedly, another line of work studies sequential learning under restrictions on adaptivity, including bandits with switching constraints \citep{SX-23} and Bayesian sequential experimentation under limited flexibility \citep{ZFN-25}. 
These works restrict how flexibly the decision-maker can adapt during experimentation, whereas our setting permits adaptive experimentation before an exogenously given commitment date but requires a fixed long-run decision after the reward shift.

Our work also relates to the recent growing literature that uses regret minimization to study adaptive experimentation in online platforms (to name a few, see, e.g., \citealp{BDKCV-21,AW-21,FMCPZ-22, YZZZCX-23,HZBW-24}). Our work differs from these works as we consider an adaptive experimentation problem with commitment and potential reward shift after the commitment.
Our model is also related to nonstationary stochastic optimization, where the underlying reward or objective function changes over time and performance is measured through dynamic regret (see, e.g., \citealp{CWW-19,CWL-25,MR-23,W-25}). Unlike this literature, we consider an anticipated reward shift at an exogenously given time, after which the decision-maker must commit and can no longer adapt her decisions.
Notably, some earlier work that studies firms' exploration effort under changing environments also derives insights under the bandit framework; e.g., \cite{PL-2012} uses simulation methods to obtain qualitative observations on this issue.

\section{Preliminaries}
\newcommand{\outcome}{o}
\newcommand{\outcomeDist}{\nu}
\newcommand{\outcomeSpace}{O}
\newcommand{\exoRandomVar}{U}

We consider a sequential experiment with a commitment problem.
There are $\timeHorizon$ different experimental units participating in the experiment sequentially, where $\timeHorizon$ is fixed and known to the experimenter. 
Upon the arrival of the $t$-th subject, the experimenter assigns a treatment (arm), among $\armNum$ independent options.
The whole process consists of two phases. 

The first phase is the \textit{experiment phase}, which includes the first 
$\switchRound$ ($\switchRound < \timeHorizon$) periods, where $\switchRound$ is fixed and known to the experimenter.
In each period $t\in[\switchRound]$, the experimenter can pick one treatment $\pickArm_t\in[\armNum]$ ($\armNum\ll\switchRound, \timeHorizon-\switchRound$) to assign it to the subject. 
The environment generates a random outcome $\outcome_{t, \pickArm_t} \in \outcomeSpace$ which is independently and identically drawn from a fixed unknown outcome distribution $\outcomeDist_{\pickArm_t}$ over a common outcome space $\outcomeSpace$.
The outcome space could be a vector space, in which case a random outcome could be interpreted as the random feature vector of a treatment.
Each $\outcomeDist_{i}$ belongs to a general distribution space $\mathcal{V}$. We do not make any assumptions on the distribution of the observed outcomes. As will soon become clear, what matters is the corresponding reward distribution.
The experimenter will observe the realized outcome $\outcome_{t, \pickArm_t}$ and also collect a reward $\rewFn(\outcome_{t, \pickArm_t})$ where function $\rewFn(\cdot):\outcomeSpace \rightarrow [0,1]$ is a known reward function in the experiment phase.\footnote{In some applications, if the decision maker has more refined prior knowledge about certain treatments' rewards, we can capture such knowledge using smaller constant intervals $[a,b]\subseteq[0,1]$ for the rewards. This will not change the asymptotic results in our paper.}


The second phase is the \textit{commitment phase} which includes the remaining periods, namely, from $\switchRound+1$ to $\timeHorizon$. 
In each period $t > \switchRound$, the experimenter commits to assign the same treatment, denoted as $\commitArm_{\switchRound}\in[\armNum]$ and collects a realized reward $\rewFnCommit(\outcome_{t, \commitArm_{\switchRound}})$ where $\rewFnCommit(\cdot): \outcomeSpace\rightarrow[0,1]$ is the known reward function in the commitment phase.
This reward function is not necessarily the same as the reward function $\rewFn(\cdot)$ and is known to the experimenter.
Here, the outcome $\outcome_{t, \commitArm_{\switchRound}}$ is also independently and identically distributed according to the same distribution $\outcomeDist_{\commitArm_{\switchRound}}$ as in the experiment phase.  
The experimenter initially knows the experimentation budget $\switchRound$ and both reward functions $\rewFn$, $\rewFnCommit$. 
Her goal is to design a policy $\policy$, which executes over the two decision phases, i.e., $\policy=(\expPolicy,\comPolicy)$, a sequential experiment policy $\expPolicy$ and a commitment policy $\comPolicy$, both possibly randomized, to maximize the total expected rewards.
More formally, the sequential experiment policy $\expPolicy$ is a function that maps the total time horizon $\timeHorizon$, the experiment length $\switchRound$, the historical observations $(\pickArm_s, \outcome_{s, \pickArm_s})_{s\in[t-1]}$, and also a random variable $\exoRandomVar$, which encodes any additional sources of randomization, to an assignment $\pickArm_t = \expPolicy(\switchRound, \timeHorizon, (\pickArm_s, \outcome_{s, \pickArm_s})_{s\in[t-1]}, \exoRandomVar)\in[\armNum]$.
Similarly, the commitment policy $\comPolicy$ is a function that maps all historical observations up to the end of $\switchRound$-th period, together with the knowledge of $\timeHorizon, \switchRound$ and an additional random variable $\exoRandomVar$ to an assignment $\commitArm_\switchRound = \comPolicy(\switchRound, \timeHorizon, (\pickArm_s, \outcome_{s, \pickArm_s})_{s\in[\switchRound]}, \exoRandomVar)\in[\armNum]$.
Specifically, the experimenter aims to maximize the following cumulative rewards:
\begin{align}
    \label{eq:cumu rew}
    \Rew[\policy]{\switchRound, \timeHorizon}
    = \sum\nolimits_{t\in[\switchRound]} \expect{\rewFn(\outcome_{t, \pickArm_t})}
    + 
    \sum\nolimits_{t\in[\switchRound+1:\timeHorizon]} \expect{\rewFnCommit(\outcome_{t, \commitArm_{\switchRound}})}~,
\end{align}
where the expectations are over the randomness of the policy $\policy$, and the randomness from the outcome distributions $(\outcomeDist_i)_{i\in[\armNum]}$.
Directly optimizing \aiedit{the total rewards in Eqn.~\eqref{eq:cumu rew}} is not tractable given that the parameters $(\outcomeDist_i)_{i\in[\armNum]}$ are not known to the experimenter a priori.
Thus, it is also more convenient to focus on the regret. The objective then is to design the policy $\policy$ that minimizes the regret defined as:
\begin{align}
    \label{eq:cumu reg}
    \Reg[\policy]{\switchRound, \timeHorizon}
    = \sum\nolimits_{t\in[\switchRound]} \left(
    \expect{\rewFn(\outcome_{t, \optArm_\rewFn})}
    - 
    \expect{\rewFn(\outcome_{t, \pickArm_t})}
    \right)
    + 
    \sum\nolimits_{t\in[\switchRound+1:\timeHorizon]} \left(
    \expect{\rewFnCommit(\outcome_{t, \optArm_{\rewFnCommit}})}-
    \expect{\rewFnCommit(\outcome_{t, \commitArm_{\switchRound}})}
    \right)~,
\end{align}
where $\optArm_\rewFn \triangleq \argmax_{i\in[\armNum]} \expect[\outcome\sim \outcomeDist_{i}]{\rewFn(\outcome)}$ denotes the optimal treatment w.r.t.\ the reward function $\rewFn$, and 
$\optArm_\rewFnCommit \triangleq \argmax_{i\in[\armNum]} \expect[\outcome\sim \outcomeDist_{i}]{\rewFnCommit(\outcome)}$ denotes the optimal treatment w.r.t.\ the reward function $\rewFnCommit$.
Throughout the paper, we assume that there exist fixed small constants $a,c>0$ such that $\switchRound\geq c\timeHorizon^a$ for all sufficiently large $\timeHorizon$. This assumption is used to rule out the less interesting scenario when the length of the experiment phase, $\switchRound$, is extremely small compared to the whole time horizon, $\timeHorizon$.

We now instantiate our setting using the previously mentioned GDPR example.
We have an analogous discussion of the manufacturing example in \Cref{exam:sc} in Appendix \ref{apx:example}.
\begin{example}[The GDPR example (continued)]

One treatment $\pickArm_t\in[\armNum]$ could represent one interface design.
The length of the experiment phase $\switchRound$ is the length of the grace period before the enactment of GDPR.
The length of the commitment phase $\timeHorizon-\switchRound$ is the length of the platform-wide deployment of the committed treatment before a new design rolls out.
The commitment date is marked by the enactment of GDPR.
Picking treatment $\pickArm_t$ generates a treatment-associated random outcome $\outcome_{t,\pickArm_t}$ which consists of two attributes, $(\textsc{revenue}_{t,\pickArm_t}, \textsc{privacy\text{-}risk}_{t,\pickArm_t})$.
$\textsc{revenue}_{t,\pickArm_t}$ is a random variable, while $\textsc{privacy\text{-}risk}_{t,\pickArm_t}=\textsc{privacy\text{-}risk}_{\pickArm_t}$ is likely a constant for a given treatment $\pickArm_t$, and the larger the number is, the less privacy customers have.
Before the enactment of GDPR, the reward function has the form $\rewFn(\outcome_{t,\pickArm_t})=\textsc{revenue}_{t,\pickArm_t}$.
After the enactment of GDPR, the reward function has the form $\rewFnCommit(\outcome_{t,\pickArm_t})=\textsc{revenue}_{t,\pickArm_t}-\textsc{privacy\text{-}risk}_{\pickArm_t}$, which captures the legal penalty on privacy.
\end{example}

\xhdr{More Notations}
We define the following notations that will be useful for our analysis. 
For each treatment $i\in[\armNum]$, let $V_{\rewFn,i}$ and $V_{\rewFnCommit,i}$ denote the distributions of $\rewFn(\outcome)$ and $\rewFnCommit(\outcome)$, respectively, when $\outcome\sim\outcomeDist_i$, $\outcomeDist_i\in\mathcal{V}$.
The instance space is $\mathcal{E}=\mathcal{V}^{\armNum}$, and an instance is $\instance=(\outcomeDist_i)_{i\in[\armNum]}\in\mathcal{E}$.
We would like to note that although the outcome distributions $(\outcomeDist_i)_{i\in[\armNum]}$ may be general, the induced reward distributions have bounded support between $[0, 1]$ by the mapping of the reward functions $\rewFn, \rewFnCommit$.
And, because the distributions $(\outcomeDist_i)_{i\in[\armNum]}$ are unknown, the induced distributions $(V_{\rewFn,i})_{i\in[\armNum]}$ and $(V_{\rewFnCommit,i})_{i\in[\armNum]}$ are also unknown.
We use $\mathcal{V}_{\rewFn, i}$, $\mathcal{V}_{\rewFnCommit, i}$ to denote the space of the reward distributions for function $\rewFn, \rewFnCommit$ for the treatment $i$, respectively.
That is, $V_{\rewFn,i}\in\mathcal{V}_{\rewFn,i}$ and $V_{\rewFnCommit,i}\in\mathcal{V}_{\rewFnCommit,i}$.
Let $\mu(P)$ denote the mean of a distribution $P$. For each treatment $i\in[\armNum]$, we write $\mu_{\rewFn,i}=\mu(V_{\rewFn,i})$ and $\mu_{\rewFnCommit,i}=\mu(V_{\rewFnCommit,i})$.

\subsection{Additional Discussions}
We conclude this section with some additional discussions. 

\xhdr{The exogenously given experimentation budget $\switchRound$}
In our framework, the experimentation budget \(\switchRound\) is exogenously given by the environment. 
This modeling choice is motivated by applications in which the commitment date is determined by an external event rather than by the decision-maker. For example,  the enforcement date of an announced regulation, the launch date of mass production, or an internally predetermined deployment deadline fixes the amount of time available for experimentation. 
In such settings, the decision-maker cannot choose when the environment changes or when the long-run transition must occur; instead, the key decision is how to allocate the given experimentation budget before that date.
This feature differentiates our setting from models in which the experimenter can choose when to stop experimentation and commit (see, e.g., \citealp{BJM-11,DNCP-19,ZY-23,QR-24}).
When the reward functions satisfy $\rewFn = \rewFnCommit$, the standard explore-then-commit strategies (see, e.g., \citealp{R-52,RI-10,GL-16}) in the multi‐armed bandits can capture this soft‐commitment scenario as well.

\xhdr{Known reward functions $\rewFn$ and $\rewFnCommit$}
In the aforementioned model, the reward mappings $\rewFn$ and $\rewFnCommit$ are known, while each arm's outcome distribution is unknown. Thus, when arm $i$ is pulled, the learner observes an outcome $\outcome_{t,i}$ that can be evaluated under both the experiment- and commitment-phase reward functions. This captures settings where the environment shift is announced in advance and its effect on evaluating outcomes is largely predictable, and it is what allows the decision maker to estimate both $\expect[\outcome\sim \outcomeDist_i]{\rewFn(\outcome)}$ and $\expect[\outcome\sim \outcomeDist_i]{\rewFnCommit(\outcome)}$ during the experiment phase.

There is an alternative model that does not require the knowledge of the $\rewFn$ and $\rewFnCommit$ while most of our results still hold.
In this model, pulling arm $i$ reveals a pair of independent reward signals drawn from $V_{\rewFn,i}\times V_{\rewFnCommit,i}$, and these signal pairs are independently and identically distributed across pulls of arm $i$. 
Under this interpretation, the learner need not know the explicit functional forms of $\rewFn$ and $\rewFnCommit$; what is essential is that each pull provides information about both experiment- and commitment-phase rewards. By contrast, if pulling an arm during experimentation only reveals its experiment-phase reward, with no outcome observation or structural link to its commitment-phase reward, then post-shift rewards are statistically unidentified. In that case, sublinear regret is impossible in general: two indistinguishable experiment-phase instances may have different optimal commitment arms, leading to linear commitment regret in one of them unless additional structure, such as a parametric relation between $\mu_{\rewFn,i}$ and $\mu_{\rewFnCommit,i}$, is imposed.

\xhdr{Reducing to regret minimization and best-arm identification}
We observe that 
when $\rewFnCommit(\outcome) \equiv 0$ for all $\outcome\in\outcomeSpace$, our sequential experiment with the commitment problem reduces to regret minimization in the classic multi-armed bandit (MAB) problem with time horizon $\switchRound$ \citep{ACF-02,LR-85}.
Another variant of the MAB problem is the ``best-arm identification problem'' (often referred to as the pure exploration problem), where the goal is to identify the best treatment at the end of the experiment. 
In particular, when 
$\rewFn(\outcome) \equiv 0$ for all $\outcome\in\outcomeSpace$, our sequential experiment with the commitment problem reduces to the fixed-budget pure exploration problem \citep{AB-10,BMS-11,KCG-16,CL-16}. 
Note that when $\rewFn(\outcome) \equiv 0$,
our regret definition in \aiedit{Eqn.~\eqref{eq:cumu reg}} is essentially minimizing the \emph{simple regret} $ (\timeHorizon-\switchRound)\cdot
\left(\expect[\outcome\sim \outcomeDist_{\optArm_{\rewFnCommit}}]{\rewFnCommit(\outcome)}-
\expect[\outcome\sim\outcomeDist_{\commitArm_{\switchRound}}]{\rewFnCommit(\outcome)}\right)$.



%

\section{Our Algorithm and Results}

\renewcommand{\armf}[1]{\mathcal{A}_{f,#1}}
\renewcommand{\armg}[1]{\mathcal{A}_{g,#1}}
\renewcommand{\roundlen}[2]{m_{#1,#2}}

\newcommand{\elimiCriteria}{\widetilde{\Delta}}
\renewcommand{\empiricalmeanf}[1]{\hat{\mu}_{f, #1}}
\renewcommand{\empiricalmeang}[1]{\hat{\mu}_{g, #1}}

\newcommand{\epochNum}{L}
\newcommand{\epoch}{\ell}
\newcommand{\stopCriteria}{\varepsilon}

\newcommand{\totalEpochs}{\numround}
\newcommand{\algcomment}[1]{\texttt{/* #1 */}}

\newcommand{\tildeO}{\widetilde{O}}

In this section, we introduce the Reserved Arm Eliminations for Commitment ({\RAEC}) algorithm for our sequential experiment and commitment problem and discuss its regret performance.

\subsection{Our Algorithm}
We begin by detailing our Algorithm {\RAEC}.
Our algorithm operates as an elimination-type algorithm (see, e.g., \citealp{EMMM-06,AO-10}) that acts in phases (epochs) and eliminates arms using increasingly sensitive hypothesis tests. 

At a high level, there are two stages of the algorithm in the experiment phase, followed by a straightforward committing stage.
In Stage I, the algorithm focuses on collecting arm information for the commitment-phase reward function $\rewFnCommit$ until the experimenter gathers sufficient information for the reward function $\rewFnCommit$ to make the commitment decision.
The algorithm moves into Stage II, in which the algorithm uses the remaining rounds in the experiment phase to conduct the regret minimization over arms w.r.t.\ reward function $\rewFn$.
Throughout the experiment phase, the algorithm will maintain two active arm sets $\armf{\epoch}, \armg{\epoch}$ for each epoch $\epoch$ and for each reward function $\rewFn, \rewFnCommit$, respectively.
The active arm sets are both initialized as the whole arm space $[\armNum]$. 

\xhdr{Stage I: Reserved arm elimination for reward function $\rewFnCommit$}
In Stage I, which includes the first $\epochNum$ epochs, the algorithm explores arms in the active arm set $\armg{\epoch}$ and eliminates arms that have bad performance based on their empirical average rewards. 
Designing the number $\epochNum$ is one key of our algorithm.
The choice of this value represents the amount of effort the decision-maker should exert during the experiment phase to learn the optimal arm for the commitment phase.
In particular, in each epoch $\epoch\in[\epochNum]$, if more than one arm remains in the active arm set $\armg{\epoch}$, the algorithm will select each arm in the arm set $\armg{\epoch}$ until the total number of times it has been selected achieves $\roundlen{\rewFnCommit}{\epoch}$. 
At the end of each epoch $\epoch$, the algorithm computes the empirical average rewards $\empiricalmeang{j, \epoch}$ for each arm $j\in \armg{\epoch}$, then updates the arm set $\armg{\epoch+1}$ by eliminating arms that deviate too much from the arm that has the empirically best reward $\max_{j\in\armg{\epoch}}\empiricalmeang{j, \epoch}$.
Specifically, arms are eliminated from $\armg{\epoch}$ if the gap between their empirical reward and the best empirical reward exceeds a threshold $\elimiCriteria_\epoch$, which is initialized to $\sfrac{1}{2}$.
The threshold $\elimiCriteria_\epoch$ will then be halved for the next epoch.
On the other hand, when there is only one arm remaining in the arm set $\armg{\epoch}$, the algorithm would stop the arm elimination for the reward function $\rewFnCommit$.
For notational simplicity, let $\armg{\epochNum+1}$ in Algorithm {\RAEC} denote the active set available when Stage I ends.

\begin{algorithm}[h]
    \caption{Reserved Arm Eliminations for Commitment ({\RAEC})}
    \label{main algo}
    \begin{algorithmic}[1]
        \State \textbf{Input:} A set of arms $\{1,2,\ldots,\armNum\}$, $\switchRound$, $\timeHorizon$, 
        parameter $\stopCriteria$;
        \State \textbf{Initialization:} Set $\elimiCriteria_1=1/2$, $\armf{1}=\armg{1}=[\armNum]$, 
        $\epochNum= \max\left\{0,\left\lceil \log_2\frac{1}{\stopCriteria} \right\rceil\right\}$. 
        \State 
        {\em Whenever $\switchRound$ rounds are exhausted in Stage I or II, the algorithm enters the Commitment Stage. If $\switchRound$ is reached before an epoch is completed, the algorithm does not compute empirical means or eliminate arms using that incomplete epoch; it keeps the active set at the beginning of the epoch.}\\
        \Comment{Below $\roundlen{\rewFn}{\epoch}, \roundlen{\rewFnCommit}{\epoch}$ are defined as in Eqn.~\eqref{eq:rounds defn}.}
        \For{$\epoch = 1,\ldots,\epochNum$} 
        \hfill  
        \Comment{Stage I: Reserved arm elimination for reward function $\rewFnCommit$}
        \If{$\vert \armg{\epoch}\vert>1$}
        \State 
        \parbox[t]{\dimexpr\linewidth-\algorithmicindent}{%
        Sample each arm in $\armg{\epoch}$ one at a time until the total number of times it has been chosen is $\roundlen{\rewFnCommit}{\epoch}$ times or reaches the experiment cap $\switchRound$ and stops.}
        \State 
        \parbox[t]{\dimexpr\linewidth-\algorithmicindent}{%
        At the end of epoch $\epoch$, compute the empirical average reward
        $\empiricalmeang{i, \epoch}$ for reward function $\rewFnCommit$ for each $i\in\armg{\epoch}$.}
        \State Update 
        $\armg{\epoch+1}\gets\left\{i\in\armg{\epoch}:\max_{j\in\armg{\epoch}}\empiricalmeang{j, \epoch} -\empiricalmeang{i, \epoch}\leq \elimiCriteria_\epoch\right\}$.
        \Else
        \State $\armg{\epoch+1}\gets\armg{\epoch}$.
        \EndIf
        \State Set $\elimiCriteria_{\epoch+1}\gets\elimiCriteria_\epoch/2$.
        \EndFor
        \State 
        \Comment{Note that below $\epoch$ restarts from $1$.}
        \For{$\epoch = 1,2,\ldots$}
        \hfill  
        \Comment{{ Stage II: Arm eliminations for reward function $\rewFn$}} 
        \State 
        \parbox[t]{\dimexpr\linewidth-\algorithmicindent}{%
        Sample each arm in $\armf{\epoch}$ one at a time until the total number of times it has been chosen is $\roundlen{\rewFn}{\epoch}$ times or reaches the experiment cap $\switchRound$ and stops.}
        \State 
        \parbox[t]{\dimexpr\linewidth-\algorithmicindent}{%
        At the end of epoch $\epoch$, compute the empirical average reward $\empiricalmeanf{i,\epoch}$ for reward function $\rewFn$ for each $i\in\armf{\epoch}$.}
        \State 
        Update 
        $\armf{\epoch+1}\gets\left\{i\in\armf{\epoch}:\max_{j\in\armf{\epoch}}\empiricalmeanf{j, \epoch} -\empiricalmeanf{i, \epoch}\leq \elimiCriteria_\epoch\right\}$.
        \State Set $\elimiCriteria_{\epoch+1}\gets\elimiCriteria_\epoch/2$.
        \EndFor
        \State 
        \Comment{Denote by $\armg{\epochNum+1}$ the current active arm set for reward function $\rewFnCommit$.}
        \State Uniformly at random commit to an arm in $\armg{\epochNum+1}$. 
        \hfill 
        \Comment{{ Commitment Stage}}
    \end{algorithmic}
\end{algorithm}

\xhdr{Stage II: Continuing arm elimination for reward function $\rewFn$}
Stage II includes the remaining rounds in the experiment phase. 
When entering Stage II, no matter how many arms remain in the active arm set $\armg{\epochNum+1}$, the algorithm will not sample arms in $\armg{\epochNum+1}$. 
Instead, the algorithm will keep eliminating the arms $\armf{1}$ using a similar procedure in Stage I.
Especially, the elimination criteria also remain the same as before.

We carefully design the total number of selections $\roundlen{\rewFn}{\epoch}, \roundlen{\rewFnCommit}{\epoch}$ of the arms in each epoch as follows,
\begin{align}
    \label{eq:rounds defn}
    \roundlen{\rewFn}{\epoch}=\frac{4\log(\switchRound)}{\elimiCriteria_\epoch^2}, \quad 
    \roundlen{\rewFnCommit}{\epoch}=\frac{4\log(\timeHorizon-\switchRound)}{\elimiCriteria_\epoch^2}~.
\end{align}
Here, notice that the number of pulls for arms in $\armf{\epoch}$ and $\armg{\epoch}$ differ in the numerator, where $\roundlen{\rewFn}{\epoch}$ adjusts for the length of the experiment phase, $\switchRound$, while $\roundlen{\rewFnCommit}{\epoch}$ adjusts for the length of the commitment phase, $\timeHorizon-\switchRound$.

\xhdr{Commitment Stage}
Note that at the beginning of the $(\switchRound+1)$-th round, at least one arm remains in the active arm set $\armg{\epochNum+1}$. 
The algorithm then uniformly at random commits to an arm in this set for the commitment phase.

We name our algorithm \textit{reserved arm eliminations for commitment} because we predetermine $\epochNum$ epochs reserved for exploring the optimal arm in the commitment phase.
Instead of choosing adaptively while learning the problem structure, a predetermined $\epochNum$ is not only easy to implement in industrial environments but also explicitly indicates how much effort the decision-maker should exert for long-term benefits.


We highlight that during the algorithm execution Stage II, we sample each arm in $\armf{\epoch}$ until it has been pulled $\roundlen{\rewFn}{\epoch}$ times, but the $\roundlen{\rewFn}{\epoch}$ number of pulls may not all occur during Stage II. 
This is because the arm is likely to have been pulled many times during Stage I.
Because we observe the outcome realization $\outcome_{t,\pickArm_t}$ upon pulling each arm $\pickArm_t$ and function forms $\rewFn$ and $\rewFnCommit$ are known, we can calculate both $\rewFn(\outcome_{t,\pickArm_t})$ and $\rewFnCommit(\outcome_{t,\pickArm_t})$.
When executing the algorithm, we actually record the history of outcome realization for each arm pull, and that is used to calculate the empirical means $\empiricalmeanf{i,\epoch}$ and $\empiricalmeang{i,\epoch}$.

\subsection{The Regret Upper Bound of Algorithm {\RAEC}}
In this section, we provide the regret upper bound of Algorithm {\RAEC}.
The key parameter to choose is the target accuracy parameter $\stopCriteria$.
The main result of this section is summarized as follows:
\begin{restatable}[Regret upper bound]{theorem}{thmub}
\label{thm:upper bound}
For $\switchRound>1$ and $\timeHorizon-\switchRound>1$, choose the parameter
\begin{equation*}
  \stopCriteria=\max\left\{4\sqrt{\frac{\armNum\log(\timeHorizon-\switchRound)}{\switchRound}},\;\left(\frac{\armNum\log(\timeHorizon-\switchRound)}{\timeHorizon-\switchRound}\right)^{1/3}\right\}~,
\end{equation*}
we have the following regret upper bound for Algorithm {\RAEC}:
\begin{equation*}
  \Reg{\switchRound, \timeHorizon}
  \mathrel{=} 
  \tildeO\left(\sqrt{\frac{\armNum(\timeHorizon - \switchRound)^{\sfrac{2}{3}}\cdot\max\left\{(\timeHorizon - \switchRound)^{\sfrac{4}{3}},\armNum^{\sfrac{1}{3}}\switchRound\right\}}{\min\left\{\switchRound, \armNum^{\sfrac{1}{3}}(\timeHorizon - \switchRound)^{\sfrac{2}{3}}\right\}}}\right)~.
\end{equation*}
\end{restatable}

The above regret upper bound exhibits a three-regime pattern. Suppressing logarithmic factors and constant factors, these three regimes are:
\begin{enumerate}
\item Short experiment (Region I): when $\switchRound\lesssim \armNum^{1/3}(\timeHorizon-\switchRound)^{2/3}$, $\stopCriteria$ is chosen to be $4\sqrt{\frac{\armNum\log(\timeHorizon-\switchRound)}{\switchRound}}$, and the regret upper bound is $\tildeO(\timeHorizon\sqrt{\armNum/\switchRound})$.
\item Balanced (Region II): when $\switchRound\gtrsim \armNum^{1/3}(\timeHorizon-\switchRound)^{2/3}$ and $\timeHorizon-\switchRound\gtrsim \armNum^{1/4}\switchRound^{3/4}$, $\stopCriteria$ is chosen to be $\left(\frac{\armNum\log(\timeHorizon-\switchRound)}{\timeHorizon-\switchRound}\right)^{1/3}$, and the regret upper bound is $\tildeO(\armNum^{1/3}(\timeHorizon-\switchRound)^{2/3})$.
\item Short commitment (Region III): when $\timeHorizon-\switchRound\lesssim \armNum^{1/4}\switchRound^{3/4}$, $\stopCriteria$ is still $\left(\frac{\armNum\log(\timeHorizon-\switchRound)}{\timeHorizon-\switchRound}\right)^{1/3}$, while the regret upper bound is $\tildeO(\sqrt{\armNum\timeHorizon})$.
\end{enumerate}
Notice that Stage I exploration in {\RAEC} is skipped when $\stopCriteria\geq1$.
As a sanity check, consider the case $\rewFnCommit\equiv0$. In this scenario, our problem reduces to the classic $\armNum$-armed bandit setting. 
Choose $\stopCriteria=1$ to skip Stage I. Since every commitment gap is zero, the regret is entirely incurred in the experiment phase.
Consequently, our regret upper bound reduces to $\tildeO(\sqrt{\armNum\timeHorizon})$ which matches the standard result (see, e.g., \citealp{ACF-02}).
On the other hand, when $\rewFn=0$, the problem becomes a pure exploration setting focused on minimizing simple regret. 
In this case, the total regret is just the simple regret at the commitment period $\switchRound+1$, multiplied by the commitment phase length $\timeHorizon-\switchRound$. 
When $\rewFn=0$ is known, we set $\roundlen{\rewFnCommit}{\epoch}=\lceil4\log\switchRound/\elimiCriteria_\epoch^2\rceil$ and choose $\stopCriteria=8\sqrt{\armNum\log\switchRound/\switchRound}$. If $\stopCriteria<1$, then $\armNum\roundlen{\rewFnCommit}{\epochNum}\leq16\armNum\log\switchRound/\stopCriteria^2+\armNum\leq\switchRound/4+\armNum<\switchRound$, so all reserved epochs can be completed. The concentration argument in \Cref{lem:regret3} now gives failure probability at most $2\armNum\epochNum/\switchRound^2$, and hence regret $O((\timeHorizon-\switchRound)\sqrt{\armNum\log\switchRound/\switchRound})$. If $\stopCriteria\geq1$, Stage I is skipped and the trivial commitment bound gives the same order. In particular, $\timeHorizon=\switchRound+1$ gives $\tildeO(\sqrt{\armNum/\timeHorizon})$ simple regret, consistent with standard bounds (see, e.g., Chapter 33 in \citealp{LS-20}). 

Finally, we observe that over a wide range of regimes (i.e., in Regions I and II), our regret upper bound is significantly larger than the typical $\tildeO(\sqrt{\armNum\timeHorizon})$ bound in standard $\armNum$-armed bandit problems.
However, as becomes clearer in \Cref{subsec:prior knowledge}, when there is no reward shift (i.e., $\rewFn=\rewFnCommit$), the regret can be substantially improved. 

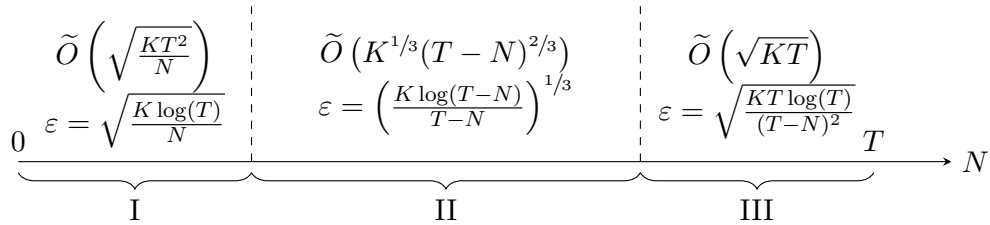
\begin{figure}[h]
    \centering
    \resizebox{0.8\textwidth}{!}{\begin{tikzpicture}[>=stealth, scale=1]

  \draw[->] (0,0) -- (12,0) node[right] {$\switchRound$};

  \node[above] at (0,0) {$0$};
  \node[above] at (11,0) {$\timeHorizon$};

  \draw[dashed] (3,0) -- (3,2);
  \draw[dashed] (8,0) -- (8,2);

  \node at (1.5,1) {
    \begin{tabular}{c}
      $\tildeO\left(\sqrt{\frac{\armNum\timeHorizon^2}{\switchRound}}\right)$\\
      $\stopCriteria=4\sqrt{\frac{\armNum\log(\timeHorizon - \switchRound)}{\switchRound}}$
    \end{tabular}
  };
  \node at (5.5,1) {
    \begin{tabular}{c}
      $\tildeO\left(\armNum^{\sfrac{1}{3}}(\timeHorizon-\switchRound)^{\sfrac{2}{3}}\right)$\\
      $\stopCriteria=\left(\frac{\armNum\log(\timeHorizon-\switchRound)}{\timeHorizon-\switchRound}\right)^{\sfrac{1}{3}}$
    \end{tabular}
  };
  \node at (9.5,1) {
    \begin{tabular}{c}
      $\tildeO\left(\sqrt{\armNum\timeHorizon}\right)$\\
      $\stopCriteria=\left(\frac{\armNum\log(\timeHorizon-\switchRound)}{\timeHorizon-\switchRound}\right)^{\sfrac{1}{3}}$
    \end{tabular}
  };
  
  \draw[
    decorate,
    decoration={brace, amplitude=6pt, mirror},
    yshift=-4pt
  ]
  (0,0) -- node[below=6pt] {I} (3,0);

  \draw[
    decorate,
    decoration={brace, amplitude=6pt, mirror},
    yshift=-4pt
  ]
  (3,0) -- node[below=6pt] {II} (8,0);

  \draw[
    decorate,
    decoration={brace, amplitude=6pt, mirror},
    yshift=-4pt
  ]
  (8,0) -- node[below=6pt] {III} (11,0);

\end{tikzpicture}}
    \caption{Phase Transition of The Regret Upper Bound of Algorithm {\RAEC}}
    \label{fig:upper_bnd_independent}
\end{figure}
We summarize the phase transition of the regret upper bound with respect to the different parameter regimes of $\switchRound$ in \Cref{fig:upper_bnd_independent}.
We highlight that in our algorithm {\RAEC}, designing $\stopCriteria$ is equivalent to determining the effort allocation between exploring the optimal arms in the experiment phase and the commitment phase.
Our algorithm {\RAEC} determines $\stopCriteria$ in advance instead of adaptively via learning the instance structure.
This simple yet effective design makes {\RAEC} easy to implement in practice while also providing transparency in how much effort is devoted to exploring good arms for the commitment phase.
However, a key question is whether our simple design can achieve the best possible performance for any online learning algorithm.
We answer this affirmatively by establishing a matching lower bound (see \Cref{thm:instance independent lower bound}).



\section{The Regret Upper Bounds Analysis}
\label{sec:upper bound}

\renewcommand{\Deltaf}[1]{\Delta_{\rewFn, #1}}
\renewcommand{\Deltag}[1]{\Delta_{\rewFnCommit,#1}}

\renewcommand{\optf}{\optArm_\rewFn}
\renewcommand{\optg}{\optArm_\rewFnCommit}

\newcommand{\eventa}[1]{\text{Event}_1\left(#1\right)}
\newcommand{\eventb}[1]{\text{Event}_2\left(#1\right)}
\newcommand{\eventc}[1]{\text{Event}_3\left(#1\right)}

\newcommand{\goodeventi}{\text{Good event $i$}}
\newcommand{\badeventi}{\text{Bad event $i$}}

\newcommand{\groupone}{\text{Group 1}}
\newcommand{\grouptwo}{\text{Group 2}}
\newcommand{\grouponea}{\text{Group 1a}}
\newcommand{\grouponeb}{\text{Group 1b}}

In this section, we provide the analysis for the regret upper bound of Algorithm {\RAEC}, i.e., we prove \Cref{thm:upper bound}.
Instead of directly proving the instance-independent upper bound presented in \Cref{thm:upper bound}, we first prove the following stronger result (see \Cref{prop:instance-dependent ub}) -- an instance-dependent regret bound for Algorithm {\RAEC}. 
We then discuss how to obtain the upper bound in \Cref{thm:upper bound} from \Cref{prop:instance-dependent ub}.

Before discussing Algorithm {\RAEC}'s instance-dependent regret upper bound, recall that $V_{\rewFn,i}$ and $V_{\rewFnCommit,i}$ denote arm $i$'s experiment-phase and commitment-phase reward distributions, respectively.
And $\mu_{\rewFn,i}=\mu(V_{\rewFn,i})$ and $\mu_{\rewFnCommit,i}=\mu(V_{\rewFnCommit,i})$. Let $\mu_{\rewFn}^*=\max_{i\in[\armNum]}\mu_{\rewFn,i}$ and $\mu_{\rewFnCommit}^*=\max_{i\in[\armNum]}\mu_{\rewFnCommit,i}$ be the mean of the optimal arm w.r.t.\ reward function $\rewFn, \rewFnCommit$ respectively. Then, let $\Deltaf{i}=\mu_{\rewFn}^*-\mu_{\rewFn,i}$ and $\Deltag{i}=\mu_{\rewFnCommit}^*-\mu_{\rewFnCommit,i}$ be the suboptimality gap. We denote the experiment- and commitment-suboptimal arm sets by $\mathcal{S}_{\rewFn}:=\{i\in[\armNum]:\Deltaf{i}>0\}$ and $\mathcal{S}_{\rewFnCommit}:=\{i\in[\armNum]:\Deltag{i}>0\}$, respectively.
We assume the optimal arm to commit, $\optg$, is unique, and we define
\[
\gamma_{\rewFnCommit,i}=\begin{cases}\Deltag{i},&i\ne\optg,\\ \min_{j\ne\optg}\Deltag{j},&i=\optg.\end{cases}
\]
With these definitions, we are ready to present the instance-dependent regret upper bound of Algorithm {\RAEC}.
\begin{restatable}{proposition}{propinsdub}
\label[proposition]{prop:instance-dependent ub}
For large $\switchRound$ and $\timeHorizon-\switchRound$, and $\stopCriteria\in(0,1)$, let $\epochNum=\lceil\log_2(1/\stopCriteria)\rceil$, and
$B_{\stopCriteria}=\min\{\switchRound,\armNum\roundlen{\rewFnCommit}{\epochNum}\}$.
If $\armNum\roundlen{\rewFnCommit}{\epochNum}\leq\switchRound$, the regret of Algorithm {\RAEC} is bounded as follows:

\begin{align}
\Reg[]{\switchRound,\timeHorizon}
= O\left(
\sum\nolimits_{i\in\mathcal{S}_{\rewFn}}\Deltaf{i}\max\left\{\frac{\log (\timeHorizon-\switchRound)}{\max\{\gamma_{\rewFnCommit,i},\stopCriteria\}^2},\frac{\log\switchRound}{\Deltaf{i}^2}\right\}
+(\timeHorizon-\switchRound)\max_{i:\Deltag{i}\leq2\stopCriteria}\Deltag{i}
+\frac{\armNum\epochNum}{\timeHorizon-\switchRound}
+\frac{\armNum B_{\stopCriteria}\epochNum}{(\timeHorizon-\switchRound)^2}
\right)~.
\label{eqn:regret}
\end{align}

\end{restatable}
\begin{proof}[Proof Sketch of \Cref{prop:instance-dependent ub}]
We now describe a proof sketch of \Cref{prop:instance-dependent ub}, which consists of three main steps.

\noindent 
\textbf{Step 1.}
We first upper bound the standard part, the expected regret incurred in Stage II, denoted by $\Rega{\switchRound}$, of the arm elimination process for the reward function $\rewFn$ (namely, eliminating arms from the active arm sets $\armf{\epoch}$) from epoch $1$ to $\numround \triangleq  \min\{\epoch: \roundlen{\rewFn}{\epoch} \ge \switchRound\}$.
See below \Cref{lem:regret1}.
\begin{restatable}{lemma}{lemregone}
\label[lemma]{lem:regret1}
$\Rega{\switchRound}\mathrel{=} O\left(\sum\nolimits_{i\in\mathcal{S}_{\rewFn}}\frac{\log(\switchRound)}{\Deltaf{i}}\right)$.
\end{restatable}
\textbf{Step 2.}
We next upper bound the expected regret, $\Regb{\switchRound}$, of the arm elimination process for the reward function $\rewFnCommit$ (namely, eliminating arms from the active arm sets $\armg{\epoch}$).
We note that to bound the regret incurred by pulling arms in the active arm set $\armg{\epoch}$, it essentially reduces to upper-bounding the expected total number of rounds used to pull the arms in the set $\armg{\epoch}$.
See below \Cref{lem:regret2}.
\begin{restatable}{lemma}{lemregtwo}
\label[lemma]{lem:regret2}

$\Regb{\switchRound}= O\left(\sum\nolimits_{i\in\mathcal{S}_{\rewFn}}\frac{\Deltaf{i}\log (\timeHorizon-\switchRound)}{\max\{\gamma_{\rewFnCommit,i},\stopCriteria\}^2}+\frac{\armNum B_{\stopCriteria}\epochNum}{(\timeHorizon-\switchRound)^2}\right)$.

\end{restatable}
\textbf{Step 3.}
Finally, we upper bound the expected regret, denoted by $\Regc{\switchRound}$, incurred in the commitment phase.
See \Cref{lem:regret3}.
\begin{restatable}{lemma}{lemregthree}
\label[lemma]{lem:regret3}

$\Regc{\switchRound,\timeHorizon}= O\left((\timeHorizon-\switchRound)\max_{i:\Deltag{i}\leq2\stopCriteria}\Deltag{i}+\frac{\armNum\epochNum}{\timeHorizon-\switchRound}\right)$.

\end{restatable}
Summing $\Rega{\switchRound}$, $\Regb{\switchRound}$, and $\Regc{\switchRound,\timeHorizon}$ completes the proof.
\end{proof}

For $\stopCriteria<1$ chosen as in \Cref{thm:upper bound}, we have
$\armNum\roundlen{\rewFnCommit}{\epochNum}\leq\frac{16\armNum\log (\timeHorizon-\switchRound)}{\stopCriteria^2}\leq\switchRound$.
Thus, the reserved epochs for Stage I exploration can be completed. If $\stopCriteria\geq1$, Stage I is skipped.
In the regime that we consider where both $\switchRound$ and $\timeHorizon-\switchRound$ are large and much bigger than $\armNum$, we actually always have $\stopCriteria<1$.
For a fixed instance, suppose $\timeHorizon-\switchRound\to\infty$ and $\log (\timeHorizon-\switchRound)/\switchRound\to0$ (in particular, the latter holds when $\switchRound=\Omega(\cc{poly}(\timeHorizon))$). 
Then, $\stopCriteria$ will vanish to zero. 
Eventually, $2\stopCriteria$ is below every positive commitment gap, so the term $(\timeHorizon-\switchRound)\max_{i:\Deltag{i}\leq2\stopCriteria}\Deltag{i}$ in \aiedit{Eqn.~\eqref{eqn:regret}} is zero.
Moreover, we notice that $\epochNum=O(\log (\timeHorizon-\switchRound))$ and $B_{\stopCriteria}\leq16(\armNum\log (\timeHorizon-\switchRound))^{1/3}(\timeHorizon-\switchRound)^{2/3}$.
Therefore, the terms $\frac{\armNum\epochNum}{\timeHorizon-\switchRound}$ and $\frac{\armNum B_{\stopCriteria}\epochNum}{(\timeHorizon-\switchRound)^2}$ in \aiedit{Eqn.~\eqref{eqn:regret}} also vanish.
With a unique $\rewFnCommit$-optimal arm $\optg$, for a sufficiently long horizon $\timeHorizon$, \aiedit{Eqn.~\eqref{eqn:regret}} becomes
\begin{equation}
\Reg[]{\switchRound,\timeHorizon\mid\instance}
= O\left(\sum\nolimits_{i\in\mathcal{S}_{\rewFn}}\Deltaf{i}\max\left\{\frac{\log (\timeHorizon-\switchRound)}{\gamma_{\rewFnCommit,i}^2},\frac{\log\switchRound}{\Deltaf{i}^2}\right\}\right)+o(1)~.
\label{eqn:regret_instance_dependent}
\end{equation}

\section{The Regret Lower Bound}
\label{sec:lower bound}

\newcommand{\textexp}{\cc{exp}}
\newcommand{\textcom}{\cc{com}}

In this section, we present an instance-dependent information lower bound and a minimax lower bound that matches the upper bound of Algorithm {\RAEC}.

We first introduce some definitions that will be useful for our analysis.
Recall that $\instance=(\outcomeDist_j)_{j=1}^{\armNum}\in\mathcal{E}$. For $i\in\mathcal{S}_{\rewFn}$ and $j\in\mathcal{S}_{\rewFnCommit}$, respectively, define
\[
d_{\rewFn,i}=\inf\nolimits_{\outcomeDist_i'\in\mathcal{V}}\left\{D(\outcomeDist_i,\outcomeDist_i'): \expect[\outcome\sim\outcomeDist_i']{\rewFn(\outcome)}>\mu_{\rewFn}^*\right\}~,~d_{\rewFnCommit,j}=\inf\nolimits_{\outcomeDist_j'\in\mathcal{V}}\left\{D(\outcomeDist_j,\outcomeDist_j'): \expect[\outcome\sim\outcomeDist_j']{\rewFnCommit(\outcome)}>\mu_{\rewFnCommit}^*\right\}~.
\]
Here, $D(\cdot,\cdot)$ represents KL-divergence, $d_{\rewFn,j}$ ($d_{\rewFnCommit,j}$) characterizes how much arm $j$'s outcome distribution $\outcomeDist_j$ needs to be modified so that arm $j$'s expected $\rewFn$ ($\rewFnCommit$) reward exceeds that of the $\rewFn$-optimal ($\rewFnCommit$-optimal) arm.
For the commitment-optimal arm $\optg$, we also define
\[
d_{\rewFnCommit,\optg}=\inf\nolimits_{\outcomeDist_{\optg}'\in\mathcal{V}}\left\{D(\outcomeDist_{\optg},\outcomeDist_{\optg}'): \expect[\outcome\sim\outcomeDist_{\optg}']{\rewFnCommit(\outcome)}<\max_{j\ne\optg}\mu_{\rewFnCommit,j}\right\}~.
\] 
Then, to derive a nontrivial instance-dependent regret lower bound, we need to specify the family of policies that we investigate.
Conventionally, we define the family of \textit{consistent policies}, which rules out policies like guessing policies that may perform extremely well in certain instances.
Following the spirit of \textit{consistent policies} defined in the standard $\armNum$-armed bandit setting, we extend the definition to our problem with modifications.
\begin{definition}[Consistent policy]
   A policy $\policy$ is called \textit{consistent} over a class of bandits $\mathcal{E}$ if for all $\instance\in\mathcal{E}$, arm permutation $\sigma$, and $p>0$, it holds that
   \begin{align*}
        \lim_{\timeHorizon-\switchRound\to\infty}\frac{\Reg[\textexp]{\switchRound\mid\instance}}{\switchRound^p}
        =0~,
        \quad
        \lim_{\timeHorizon-\switchRound\to\infty}\frac{\Reg[\textcom]{\timeHorizon-\switchRound\mid\instance}}{(\timeHorizon-\switchRound)^p}=0~,
        \quad
        \Reg[\textcom]{\timeHorizon-\switchRound\mid\instance}
        =\Reg[\textcom]{\timeHorizon-\switchRound\mid\sigma(\instance)}~,
   \end{align*}
   where $\Reg[\textexp]{\switchRound\mid \instance}$ and $\Reg[\textcom]{\timeHorizon-\switchRound\mid \instance}$ represent the expected cumulative regret of the experiment phase and the commitment phase, respectively.
\end{definition}
Recall that we have assumed $\switchRound\mathrel{=}\Omega(\cc{poly}(\timeHorizon))$, so the regime where $\timeHorizon-\switchRound$ goes to infinity implies both $\switchRound$ and $\timeHorizon$ grow to infinity. 
The last condition $\Reg[\textcom]{\timeHorizon-\switchRound\mid\instance}=\Reg[\textcom]{\timeHorizon-\switchRound\mid\sigma(\instance)}$ says that the algorithm should result in the same expected regret for the commitment phase if we simply permute the arms' identities,
i.e., the policy is symmetric under arm permutation.

\begin{restatable}[Instance-dependent lower bound]{theorem}{thminsd}
\label[theorem]{thm:instance dependent lower bound}
Given an instance $\instance\in\mathcal{E}$, the regret of every policy consistent over $\mathcal{E}$ satisfies the following asymptotic lower bound as $\timeHorizon-\switchRound\to\infty$:
\begin{align*}
    &\Reg[]{\switchRound, \timeHorizon\mid\instance}
    \mathrel{=} \Omega\left(
    \sum\nolimits_{i\in\mathcal{S}_{\rewFn}}
    \max\left\{\frac{\log(\switchRound)}{d_{\rewFn,i}},\frac{\log(\timeHorizon-\switchRound)}{d_{\rewFnCommit,i}}\right\}\Deltaf{i}
    \right)~.
\end{align*}
\end{restatable}

For a unique commitment optimum, this lower bound matches the gap and horizon dependence of \aiedit{Eqn.~\eqref{eqn:regret_instance_dependent}} up to constants whenever $d_{\rewFn,i}=\Theta(\Deltaf{i}^2)$ and $d_{\rewFnCommit,i}=\Theta(\gamma_{\rewFnCommit,i}^2)$ for $i\in\mathcal{S}_{\rewFn}$, with constants independent of the gaps. This holds, for example, for independent Bernoulli reward coordinates with means bounded away from zero and one.
Our analysis can be extended to sub-Gaussian distributed rewards, in which case the matching results can be extended to more general distributions, including Gaussian.

Furthermore, using a different instance construction and analysis, we derive the following instance-independent regret lower bound (i.e., minimax lower bound). Let $\mathcal{E}_{\mathrm{coord}}\subseteq\mathcal{E}$ be the fixed subclass in which $\outcomeSpace=[0,1]^2$, $\rewFn(x,y)=x$, $\rewFnCommit(x,y)=y$, and the two coordinates of every arm are independent Bernoulli random variables with means in a fixed interval contained in $(0,1)$.
\begin{restatable}[Instance-independent lower bound]{theorem}{thminsind}
\label[theorem]{thm:instance independent lower bound}
For $\armNum\geq4$, we have
\begin{align*}
    \inf\nolimits_{\policy}\sup\nolimits_{\instance\in\mathcal{E}_{\mathrm{coord}}}
    \Reg[\policy]{\switchRound,\timeHorizon\mid\instance} \mathrel{=} 
    \Omega\left(\sqrt{\frac{\armNum(\timeHorizon-\switchRound)^{\sfrac{2}{3}}\cdot\max\left\{(\timeHorizon-\switchRound)^{\sfrac{4}{3}},\armNum^{\sfrac{1}{3}}\switchRound\right\}}{\min\left\{\switchRound, \armNum^{\sfrac{1}{3}}(\timeHorizon-\switchRound)^{\sfrac{2}{3}}\right\}}}\right)~.
\end{align*}
\end{restatable}
Notice that since $\mathcal{E}_{\mathrm{coord}}\subseteq\mathcal{E}$, the above minimax lower bound also holds for $\inf\nolimits_{\policy}\sup\nolimits_{\instance\in\mathcal{E}}\Reg[\policy]{\switchRound,\timeHorizon\mid\instance}$.

\begin{figure}[htbp]
\centering
\begin{tikzpicture}[>=stealth, scale=1]

  \draw[->] (0,0) -- (10,0) node[right] {$\switchRound$};

  \node[above] at (0,0) {$0$};
  \node[above] at (9,0) {$\timeHorizon$};

  \draw[dashed] (2.5,0) -- (2.5,1.5);
  \draw[dashed] (6.5,0) -- (6.5,1.5);

  \node at (1.25,0.75) {
    \begin{tabular}{c}
      $\Omega\left(\sqrt{\frac{\armNum\timeHorizon^2}{\switchRound}}\right)$ 
    \end{tabular}
  };
  \node at (4.5,0.75) {
    \begin{tabular}{c}
      $\Omega\left(\armNum^{\sfrac{1}{3}}(\timeHorizon-\switchRound)^{\sfrac{2}{3}}\right)$ 
    \end{tabular}
  };
  \node at (7.75,0.75) {
    \begin{tabular}{c}
      $\Omega\left(\sqrt{\armNum\timeHorizon}\right)$ 
    \end{tabular}
  };
  
  \draw[
    decorate,
    decoration={brace, amplitude=6pt, mirror},
    yshift=-4pt
  ]
  (0,0) -- node[below=6pt] {I} (2.5,0);

  \draw[
    decorate,
    decoration={brace, amplitude=6pt, mirror},
    yshift=-4pt
  ]
  (2.5,0) -- node[below=6pt] {II} (6.5,0);

  \draw[
    decorate,
    decoration={brace, amplitude=6pt, mirror},
    yshift=-4pt
  ]
  (6.5,0) -- node[below=6pt] {III} (9,0);
\end{tikzpicture}
\caption{Phase Transition of Minimax Lower Bound}
\label{fig:lower_bnd_independent}
\end{figure}
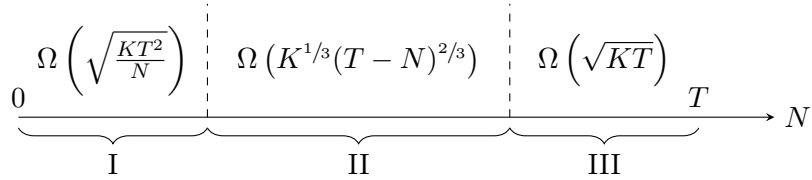
Figure \ref{fig:lower_bnd_independent} depicts the phase transition of the instance-independent regret lower bound with respect to the scale of $\switchRound$.
Regions I, II, and III are defined as follows:
$(i)$. Region I: $\switchRound<\armNum^{\sfrac{1}{3}}(\timeHorizon - \switchRound)^{\sfrac{2}{3}}$.
$(ii)$. Region II: $\switchRound\geq\armNum^{\sfrac{1}{3}}(\timeHorizon - \switchRound)^{\sfrac{2}{3}}$ and $(\timeHorizon - \switchRound)\geq\armNum^{\sfrac{1}{4}}\timeHorizon^{\sfrac{3}{4}}$.
$(iii)$. Region III: $(\timeHorizon - \switchRound)<\armNum^{\sfrac{1}{4}}\timeHorizon^{\sfrac{3}{4}}$.

The instance-independent lower bound matches the upper bound of {\RAEC} in \Cref{thm:upper bound}.
This is especially noteworthy because our Algorithm {\RAEC} 
allocates exploration effort for the commitment phase in a simple, predetermined way rather than adaptively.
This leads to a useful managerial implication for practitioners:
Reserving effort in advance to explore good arms for the long-term commitment is both easy to implement and sufficient.
More sophisticated and adaptive methods of allocating effort with the help of learning the instance structure will not help much.
Meanwhile, we notice that in a wide range of regimes (Region I and II), our tight regret bound is significantly larger than a typical $\Omega\left(\sqrt{\armNum\timeHorizon}\right)$ result.
This implies that the presence of the commitment phase and the reward shift makes the problem fundamentally more challenging than standard bandit problems.
As clarified in \Cref{subsec:prior knowledge}, such difficulty mainly comes from the reward shift.

We provide a proof sketch below for \Cref{thm:instance independent lower bound}.
At a high level, to prove the minimax lower bound, given any policy, we construct two instances such that this policy will suffer a large regret in at least one of these two instances (see, e.g., \citealp{S-19,LS-20}).
We first construct the following instance $\instance$: for some positive constants $\alpha, \delta_{\rewFn}, \delta_{\rewFnCommit}$ (their values are specified in the proof),
\begin{equation*}
\begin{aligned}
    \mu_{\rewFn, i} = 
    \begin{cases}
        \alpha+\delta_{\rewFn}~, & \text{ if }  i =1~; \\
        \alpha~, & \text{ if }  i \in [2, \sfrac{\armNum}{2}]~; \\
        0~, & \text{ if }  i \in [ \sfrac{\armNum}{2} + 1, \armNum]~; 
    \end{cases} \quad 
    \mu_{\rewFnCommit, i} = 
    \begin{cases}
        \delta_{\rewFnCommit}~, & \text{ if }  i = \sfrac{\armNum}{2} + 1~; \\
        0~, & \text{ otherwise}~.
    \end{cases}
\end{aligned}
\end{equation*}
We group the first $\armNum/2$ arms as arm group $A$ and the other half of arms as arm group $B$.
Given a policy $\policy$, let $\expect[\instance]{T_A}$ and $\expect[\instance]{T_B}$ be the expected number of pulls distributed to group $A$ and $B$, respectively.
Clearly, $\expect[\instance]{T_A}+\expect[\instance]{T_B}=\switchRound$.
Let $\expect[\instance]{T_l}$ be policy $\policy$'s expected number of pulls of arm $l$.
Then, we define 
$i\primed \triangleq \argmin_{l\in A\setminus\{1\}}\expect[\instance]{T_l}$ and 
$j\primed =\argmin_{l\in B\setminus\{\sfrac{\armNum}{2} + 1\}}\expect[\instance]{T_l}$.
Clearly, $\expect[\instance]{T_{i\primed}}\leq\frac{\expect[\instance]{T_A}}{\sfrac{\armNum}{2}-1}$ and $\expect[\instance]{T_{j\primed}}\leq\frac{\expect[\instance]{T_B}}{\sfrac{\armNum}{2}-1}$.
With these definitions, we define the instance $\instance\primed$ as follows:
\begin{equation*}
\begin{aligned}
    \mu_{\rewFn, i}\primed = 
    \begin{cases}
        \alpha+\delta_{\rewFn}~, & \text{ if }  i = 1~; \\
        \alpha+2\delta_{\rewFn}~, & \text{ if }  i = i\primed~; \\
        \alpha~, & \text{ if }  i \in [2, \sfrac{\armNum}{2}] \setminus\{i\primed\}~; \\
        0~, & \text{ if }  i \in [ \sfrac{\armNum}{2} + 1, \armNum]~; 
    \end{cases} \quad 
    \mu_{\rewFnCommit, i}\primed = 
    \begin{cases}
        \delta_{\rewFnCommit}~, & \text{ if }  i = \sfrac{\armNum}{2} + 1~; \\
        2\delta_{\rewFnCommit}~, & \text{ if }  i = j\primed~; \\
        0~, & \text{ otherwise}~.
    \end{cases} 
\end{aligned}
\end{equation*}
We show that the policy $\policy$ would suffer a large regret either in instance $\instance$ or in instance $\instance\primed$.

\section{Extensions}
\label{sec:extensions}

\newcommand{\incresfun}{\phi}
\newcommand{\gradientUB}{M}
\newcommand{\noisefun}{\delta}
\newcommand{\noiseBound}{D}
\newcommand{\tax}{\textsc{tax}}

In this section, we discuss two extensions of our basic model:
(1) In the first extension, we consider the situation when the decision-maker has access to more nuanced prior knowledge about the reward shift structure (i.e., the functional relationship between reward functions $\rewFn$ and $\rewFnCommit$).
One main insight from the previous basic model is that, with the information from learning the instance structure, adaptively determining the effort allocation between exploring the optimal arms in the experiment phase and the commitment phase will not help much.
However, we argue that prior knowledge of the reward shift structure can sometimes help.
The key is to identify the reward shift's ranking-changing effect instead of focusing on its absolute value.
(2) In the second extension, we consider a generalized version of the basic model in which the decision-maker is allowed to commit to a portfolio of arms. The reward in the commitment phase is a concave function of the portfolio's combined outcome.
We argue that the effectiveness of the idea of predetermined effort allocation persists in this general setting.
Besides, a novel idea of the algorithm is that instead of calculating the optimal portfolio to commit to based on the estimations of each arm, our algorithm commits to the empirical distribution induced by the exploration execution path directly.

\subsection{Improved Regret with Prior Knowledge on Reward Shift}
\label{subsec:prior knowledge}
In the basic model, we do not pose any assumptions on the relationship between reward functions $f(\cdot)$ and $g(\cdot)$ except that both are bounded within $[0,1]$.
Our algorithm {\RAEC} predetermines the accuracy level $\stopCriteria$ that we aim to achieve for the commitment phase.
And, as we argued, adaptively determining the target accuracy level $\stopCriteria$ via learning the problem structure won't help improve the asymptotic regret performance.
However, in certain applications, some prior knowledge of the reward shift structure is available.
We use the following tariff application as an example.
\begin{example}[Tax policy change example]
   Let $\outcome_{t,\pickArm_t}$ be the profit of a product $\pickArm_t$, $\textsc{profit}_{t,\pickArm_t}$.
   Before the tax policy change, all productions are tax-free.
   The new tax policy puts a common $(1-\gradientUB)$ portion of the tax, with some product-specific adjustments, $\textsc{tax-adjustment}_{t,\pickArm_t}$. 
   Then, after the enactment of the new tax policy, the profit of product $\pickArm_t$ has the form $\gradientUB\cdot\textsc{profit}_{t,\pickArm_t}+\textsc{tax-adjustment}_{t,\pickArm_t}$.
\end{example}
As we demonstrate below, prior knowledge of the problem structure helps.
Motivated by the example, we consider the following functional relationship between $f(\cdot)$ and $g(\cdot)$.
\begin{definition}[Perturbed affine relationship]

Reward functions $\rewFn$ and $\rewFnCommit$
exhibit a perturbed affine relationship if there exists a constant $\gradientUB \ge 0$ and a noise function $\noisefun(\cdot): \outcomeSpace\rightarrow [\sfrac{-\noiseBound}{2},\sfrac{\noiseBound}{2}]$ for some $\noiseBound \ge 0$, namely,
$\rewFnCommit(\outcome)=\gradientUB\cdot\rewFn(\outcome)+\noisefun(\outcome)$ for $\outcome\in\outcomeSpace$.
\end{definition}
Intuitively, $\gradientUB\cdot\rewFn(\cdot)$ is a ranking-preserving mapping, i.e., the arms' mean reward ranking remains the same as under $\rewFn(\cdot)$ after multiplying by $\gradientUB$.
On the other hand, $\noisefun$ could be ranking-changing.
We can tune {\RAEC} based on the knowledge of $\gradientUB$ and $\noiseBound$ as well as $\armNum$, $\switchRound$, and $\timeHorizon$.

Both $\gradientUB$ and $\noisefun$ can significantly shift the value of a treatment, but we argue that the size of $D$ is significantly more important than the size of $M$.
In particular, for constant \(\gradientUB>0\) and $\noiseBound=0$, the regret of {\RAEC} is upper bounded by
\[
\Reg{\switchRound,\timeHorizon}
\mathrel{=}
\left\{
\begin{aligned}
&\widetilde O\left(
\sqrt{\armNum\switchRound}+(\timeHorizon-\switchRound)\sqrt{\frac{\armNum}{\switchRound}}
\right)~,~
&\text{for }
\switchRound\leq\gradientUB (\timeHorizon-\switchRound)~,
\\
&\widetilde O\left(
\sqrt{\armNum\switchRound}+\sqrt{\armNum(\timeHorizon-\switchRound)}
\right)~,~
&\text{for }
\switchRound>\gradientUB (\timeHorizon-\switchRound)~.
\end{aligned}
\right.
\]
For example, in the balanced regime for $\switchRound=\frac{\gradientUB}{1+\gradientUB}\cdot\timeHorizon$, the upper bound becomes $\widetilde O\left(\sqrt{\armNum\timeHorizon}\right)$, which becomes the familiar bound of the standard multi-armed bandit problem.
Specifically, if $\gradientUB=1$ and $\noiseBound=0$, there is no reward shift; the regret upper bound remains almost the same except that the threshold for the two cases becomes $\switchRound=\timeHorizon/2$.
On the other hand, if the constant $\noiseBound>0$, the problem is fundamentally no different from the basic model without the perturbation structural information, and the regret of {\RAEC} is upper-bounded by the same bound derived in the basic model.
As we can see, although both positive constants $\gradientUB$ and $\noiseBound$ can lead to constant reward shift for each option, their impacts on the problem structure are fundamentally different.
The managerial takeaway is that specifying the ranking-changing effect is much more important than focusing on the absolute value of the reward shifts (the detailed analysis is deferred to \Cref{subsec:perturb}).

\subsection{Generalized Concave Commitment Reward}
\label{subsec:concave}
Recall our motivation setting where a firm has a short-term payoff of the accumulated profit, while due to policy environment changes, the long-term goal is to find a business solution that balances profitability and other factors like average quality, aggregated environmental impact across all product lines, average supply chain resilience over time, etc. 
We provide a practical example for such a situation in \Cref{exam:concave} in Appendix \ref{apx:example}.

Unlike the basic model, due to the nonlinearity of the reward function for the commitment phase, the optimal commitment could be a combination of arms instead of a single arm.
In this extension, the firm is allowed to commit to a portfolio of candidate solutions for the commitment phase. A portfolio of candidate solutions means an action plan for deploying each candidate solution in the commitment phase.
This means the action for each period in the commitment phase could be different, but is determined at the beginning of the commitment phase.
We consider the case where the portfolio's commitment phase payoff is Lipschitz concave, defined in the outcome space.
For example, the long-run performance is linear in profit minus the convex increasing penalty of environmental impact (i.e., the marginal penalty increases as the environmental impact grows).
Formally, we formulate the decision-maker's problem as follows,
\begin{equation}
\label{eqn:concave commit}
    \max_{\pi} \;
    \expect{\sum\nolimits_{t=1}^{\switchRound}\rewFn(\outcome_{t, \pickArm_t})+ (\timeHorizon-\switchRound)\cdot \rewFnCommit\left(\frac{\sum\nolimits_{t=\switchRound+1}^\timeHorizon\outcome_{t, \pickArm_t}}{\timeHorizon-\switchRound}\right)}~,
\end{equation}
where $\pickArm_{\switchRound+1}$ to $\pickArm_{\timeHorizon}$ are committed at the beginning of the commitment phase.
We require $\outcomeSpace$ to be a compact space in this extension.
For simplicity, we assume $\outcomeSpace=[0,1]^\feaDimen$, where $\feaDimen$ is the dimension of the outcome space.
And, for any $\outcome_1,~\outcome_2\in\outcomeSpace$, $\lvert\rewFnCommit(\outcome_1)-\rewFnCommit(\outcome_2)\rvert\leq \lipschitzConst\cdot\lVert\outcome_1-\outcome_2\rVert$ for some Lipschitz constant $\lipschitzConst$.

First, we argue that our basic problem in \aiedit{Eqn.~\eqref{eq:cumu rew}} is a special case of problem in \aiedit{Eqn.~\eqref{eqn:concave commit}}.
To see this, let $\rewFnCommit$ be an affine function, then we return to the formulation of \aiedit{Eqn.~\eqref{eq:cumu rew}} except that in the basic problem in \aiedit{Eqn.~\eqref{eq:cumu rew}}, the decision-maker can only commit to one arm for the commitment phase.
However, notice that if we allow the decision-maker to commit to a portfolio of arms in \aiedit{Eqn.~\eqref{eq:cumu rew}}, the optimal decision of the benchmark is still committing to the single best arm in the commitment phase.
What might make the optimal decision to commit to a nontrivial portfolio is the nonlinearity of $\rewFnCommit$ in \aiedit{Eqn.~\eqref{eqn:concave commit}}.

Now, we have demonstrated the formulation in \aiedit{Eqn.~\eqref{eqn:concave commit}} is an extension of the basic model formulation, but the analysis later relies on the mild assumption that $\outcomeSpace$ is compact.
According to \cite{AD-2014}, the normalized per-period reward of the commitment phase can be upper bounded by $\OPT_{\rewFnCommit}=\max_{\boldsymbol{p} \in \Delta([\armNum])}\rewFnCommit\left(\sum_{i=1}^Kp_i\cdot\expect[\outcome\sim \outcomeDist_{i}]{\outcome}\right)$ due to concavity.
Therefore, we consider the following definition of regret,
\begin{align*}
    \Reg[\policy]{\switchRound, \timeHorizon}
    = \sum\nolimits_{t\in[\switchRound]} \left(
    \expect{\rewFn(\outcome_{t, \optArm_\rewFn})}
    - 
    \expect{\rewFn(\outcome_{t, \pickArm_t})}
    \right)
    +
    (\timeHorizon-\switchRound)\cdot
    \left(
    \OPT_{\rewFnCommit} 
    -
    \expect{\rewFnCommit\left(\frac{\sum_{t=\switchRound+1}^\timeHorizon\outcome_{t, \pickArm_t}}{\timeHorizon-\switchRound}\right)}
    \right)~.
\end{align*}

\xhdr{Challenges}
Unlike the basic model, where the optimal commitment strategy is to commit to the single optimal arm for the \textit{commitment phase} and the reserved exploration can be devoted to identifying the optimal arm to commit, with a concave commitment objective function, exploration effort should be distributed more evenly across all arms in order to identify the optimal portfolio.

\xhdr{Intuitions of {\ROSCOC}}
Following the spirit of {\RAEC} algorithm, we propose the \textit{reserved online stochastic convex optimization for commitment} algorithm ({\ROSCOC}), where we still reserve a certain amount of time for pure exploration.
However, instead of applying the arm elimination technique during the reserved periods, we apply the well-established online stochastic convex optimization algorithm.
Unlike the arm elimination method, which identifies the optimal arm for the commitment phase, online stochastic convex optimization only generates an execution path in the experiment phase.
The challenge is how to relate this execution path to the portfolio of arms to commit.

Instead of calculating the optimal arm portfolio for the commitment phase based on historical estimates of each arm's outcome vector, the novel idea of {\ROSCOC} is to use the execution path to approximate the ideal arm portfolio directly.
Following this idea, we generate the action plan for the commitment phase by uniformly sampling from the execution path.
In other words, the algorithm would commit to the empirical distribution induced by the execution path of the algorithm during the reserved exploration periods.
And, we show that this method is optimal in some sense.
For details, please refer to the pseudo-code of Algorithm {\ROSCOC} in \Cref{algo:roscoc} in the appendix. 
The main result is summarized as follows:
\begin{restatable}{theorem}{thmregretconcave}
\label{thm:regret concave}
The regret of Algorithm {\ROSCOC} has the form
$$\Reg{\switchRound, \timeHorizon}
\mathrel{=} 
\widetilde{O}\left(\timeHorizon\cdot\lipschitzConst\sqrt{\frac{\armNum\feaDimen}{\switchRound}}+\sqrt{\armNum\switchRound}+\lipschitzConst^{\sfrac{2}{3}}\armNum^{\sfrac{1}{3}}\feaDimen^{\sfrac{1}{3}}(\timeHorizon-\switchRound)^{\sfrac{2}{3}}\right)~,$$
where {\ROSCOC} reserves $\tau=\min\left\{\switchRound,\lipschitzConst^{\sfrac{2}{3}}\armNum^{\sfrac{1}{3}}\feaDimen^{\sfrac{1}{3}}(\timeHorizon-\switchRound)^{\sfrac{2}{3}}\log(\timeHorizon-\switchRound)^{\sfrac{1}{3}}\right\}$ periods to exploration for commitment.
\end{restatable}
When $\lipschitzConst$ and $\feaDimen$ are constants, the regret bound stated in \Cref{thm:regret concave} essentially reduces to the regret bound derived in \Cref{thm:upper bound}:
$\widetilde{O}\left(\sqrt{\frac{\armNum(\timeHorizon - \switchRound)^{\sfrac{2}{3}}\cdot\max\left\{(\timeHorizon - \switchRound)^{\sfrac{4}{3}},\armNum^{\sfrac{1}{3}}\switchRound\right\}}{\min\left\{\switchRound, \armNum^{\sfrac{1}{3}}(\timeHorizon - \switchRound)^{\sfrac{2}{3}}\right\}}}\right)$.
On the other hand, since the concave-reward model is a generalization of the basic model \aiedit{Eqn.~\eqref{eq:cumu rew}}, we know that the minimax lower bound in \Cref{thm:instance independent lower bound} also applies to the portfolio commitment class.

In summary, although the general model appears to be much more challenging, we find that the effectiveness of predetermined effort allocation of exploration for the commitment phase still applies in the more general setting.
The execution path of the online stochastic convex optimization algorithm is a good enough proxy for the optimal portfolio in the commitment phase.
With the help of these algorithmic design techniques, we show that the general problem is not fundamentally harder than the basic model.

\section{Experiments}

In this section, we complement our theory with three numerical experiments. 
The first two experiments focus on the baseline post-commitment reward-shift model: we first examine whether \RAEC\ exhibits the schedule-dependent regret scaling predicted by \Cref{thm:upper bound} (see \Cref{subsec:exp-scaling}), and then compare \RAEC\ with a UCB-style baseline (see \Cref{subsec:raec-vs-ucb}).
The third experiment, reported in \Cref{subsec:concave-numerics}, evaluates ROSCOC against a UCB-style portfolio baseline in the generalized concave commitment setting of \Cref{subsec:concave}. 
All reported regrets are Monte Carlo averages over independent replications.


\subsection{Hard-Instance Scaling for \RAEC}
\label{subsec:exp-scaling}

Our first experiment examines the horizon-scaling behavior of algorithm \RAEC\ on hard instances motivated by the minimax lower-bound construction. 
The purpose of this experiment is not to recover the exact constants hidden in the theoretical bounds, but rather to test whether the empirical growth rate of total regret changes with the commitment schedule in the way predicted by \Cref{thm:upper bound}. 
Since the theoretical rates are stated up to logarithmic factors and multiplicative constants, the relevant comparison in this experiment is the slope of the regret curve on a log-log scale.\footnote{Throughout this section, $\log$ denotes the natural logarithm, with base $e$.}

\xhdr{Experiment settings}
We fix the number of arms $\armNum=4$ and evaluate five commitment schedules: $\switchRound=\lfloor \timeHorizon^{1/2}\rfloor$, $\switchRound=\lfloor \timeHorizon^{2/3}\rfloor$, $\switchRound=\lfloor \timeHorizon/4\rfloor$, $\switchRound=\lfloor \timeHorizon-\timeHorizon^{3/4}\rfloor$, and $\switchRound=\lfloor \timeHorizon-\timeHorizon^{1/2}\rfloor$. 
The horizon grid is $\timeHorizon\in\{500,1000,2500,5000,10000,15000,20000\}$. 
These five schedules cover the three regimes in our regret characterization. 
The schedule $\switchRound=\lfloor \timeHorizon^{1/2}\rfloor$ corresponds to the short-experiment regime, where the theory predicts regret of order $\sqrt{\armNum\timeHorizon^2/\switchRound}$, which gives a $\timeHorizon^{3/4}$-type scaling when $\armNum$ is fixed and $\switchRound=\timeHorizon^{1/2}$. 
The schedules $\switchRound=\lfloor \timeHorizon^{2/3}\rfloor$ and $\switchRound=\lfloor \timeHorizon/4\rfloor$ lie in the boundary/balanced regimes, where the leading rate is $\armNum^{1/3}(\timeHorizon-\switchRound)^{2/3}$, corresponding to a $\timeHorizon^{2/3}$-type scaling when $\armNum$ is fixed. 
The schedules $\switchRound=\lfloor \timeHorizon-\timeHorizon^{3/4}\rfloor$ and $\switchRound=\lfloor \timeHorizon-\timeHorizon^{1/2}\rfloor$ represent late-commitment regimes. 
For the schedule $\switchRound=\lfloor \timeHorizon-\timeHorizon^{3/4}\rfloor$, the balanced-regime rate gives $(\timeHorizon-\switchRound)^{2/3}=\timeHorizon^{1/2}$. 
For the schedule $\switchRound=\lfloor \timeHorizon-\timeHorizon^{1/2}\rfloor$, the short-commitment rate gives the same $\timeHorizon^{1/2}$-type scaling.

For each pair $(\switchRound,\timeHorizon)$, we use schedule-specific Bernoulli instances that follow the spirit of the hard-instance construction in the proof of \Cref{thm:instance independent lower bound}. 
The experiment-phase mean vector is denoted by $\mu_{\rewFn}$, and the commitment-phase mean vector is denoted by $\mu_{\rewFnCommit}$. 
\begin{itemize}
    \item 
    For the commitment schedule $\switchRound=\lfloor \timeHorizon^{1/2}\rfloor$, the mean vectors are
    \begin{align*}
    \mu_{\rewFn}
    =
    (0.45+\Delta_{\rewFn},\,0.45,\,0,\,0)~,
    \quad
    \mu_{\rewFnCommit}
    =
    (0.45,\,0.45,\,0.45+\Delta_{\rewFnCommit},\,0.45)~,
    \end{align*}
    where the gap parameters are $\Delta_{\rewFn}=\Delta_{\rewFnCommit}=0.5\sqrt{\armNum/\switchRound}$. 
    \item 
    For the commitment schedules $\switchRound=\lfloor \timeHorizon^{2/3}\rfloor$ and $\switchRound=\lfloor \timeHorizon/4\rfloor$, the mean vectors are
    \begin{align*}
        \mu_{\rewFn}
        =
        (0.55,\,0.535,\,0.49,\,0.49)~,
        \quad
        \mu_{\rewFnCommit}
        =
        (0.02,\,0.02,\,0.55+\Delta_{\rewFnCommit},\,0.55)~,
    \end{align*}
    where the commitment-phase gap parameter is $\Delta_{\rewFnCommit}=0.3(\armNum/(\timeHorizon-\switchRound))^{1/3}$. 
    \item 
    For the commitment schedules $\switchRound=\lfloor \timeHorizon-\timeHorizon^{3/4}\rfloor$ and $\switchRound=\lfloor \timeHorizon-\timeHorizon^{1/2}\rfloor$, the mean vectors are
    \begin{align*}
    \mu_{\rewFn}
    =
    (0.45+\Delta_{\rewFn},\,0.45,\,0,\,0)~, 
    \quad
    \mu_{\rewFnCommit}
    =
    (0,\,0,\,2\Delta_{\rewFnCommit},\,\Delta_{\rewFnCommit})~,
    \end{align*}
    where the gap parameters are $\Delta_{\rewFn}=2\sqrt{\armNum/\switchRound}$ and $\Delta_{\rewFnCommit}=\sqrt{\armNum/(\timeHorizon-\switchRound)}$. 
\end{itemize}
Each pull generates phase-specific Bernoulli reward observations with the corresponding means, so that the learner can update empirical estimates for both the $\rewFn$- and $\rewFnCommit$-rewards, as in the baseline model.
The short-experiment schedule emphasizes the difficulty of identifying a good commitment arm, with the commitment-phase gap $\Delta_{\rewFnCommit}$ shrinking at the order of $\sqrt{\armNum/\switchRound}$. 
The balanced schedules use the commitment-phase gap $\Delta_{\rewFnCommit}$ shrinking at the order of $(\armNum/(\timeHorizon-\switchRound))^{1/3}$. 
The late-commitment schedules emphasize the standard bandit component in the experiment phase, with the experiment-phase gap $\Delta_{\rewFn}$ shrinking at the order of $\sqrt{\armNum/\switchRound}$.

The implementation of \RAEC\ is as follows. 
The Stage-I target pull count $m_{\rewFnCommit,\ell}$ denotes the target total number of pulls for each remaining arm after epoch $\ell$ during the arm-elimination process for the commitment-phase reward. 
For the commitment schedule $\switchRound=\lfloor \timeHorizon^{1/2}\rfloor$, we implement the exact \RAEC\ algorithm defined in \Cref{main algo}. 
For the other four commitment schedules, we use a constant adjustment in Stage I: in each Stage-I epoch, the target pull count is set to $(1/10)m_{\rewFnCommit,\ell}$ instead of $m_{\rewFnCommit,\ell}$. 
This adjustment only changes the constant factor in the reserved commitment exploration and avoids overly conservative finite-horizon exploration induced by the concentration constants. 
The Stage-II arm-elimination process for the experiment-phase reward is always implemented exactly as in \Cref{main algo}.
For commitment, the algorithm commits to the empirically optimal arm, which does not improve minimax performance but performs better empirically.

\xhdr{Experiment results}
For each schedule and horizon length, we run \RAEC\ independently for $1500$ replications and compute the mean total regret, denoted by $\bar{\rew}_{\timeHorizon}$. 
\Cref{fig:hard-instance-scaling} reports $\log \bar{\rew}_{\timeHorizon}$ against $\log \timeHorizon$. 
In each panel, the solid curve is the empirical log mean regret. 
The dashed curve, labeled ``Scaled theory,'' is the theoretical rate predicted by \Cref{table:opt regret} after the same vertical normalization. 
The dotted curve, labeled ``Slope fit,'' is an unconstrained empirical log-log regression line.

We now describe those two reference curves more explicitly. 
For each schedule $s$, the schedule-specific theoretical rate is denoted by $r_s(\timeHorizon)$, which is defined as
\begin{align*}
    r_s(\timeHorizon)
    =
    \begin{cases}
        \sqrt{\armNum\timeHorizon^2/\switchRound(\timeHorizon)}~,
        & \text{if } \switchRound(\timeHorizon)=\lfloor \timeHorizon^{1/2}\rfloor~, \\[4pt]
        \armNum^{1/3}(\timeHorizon-\switchRound(\timeHorizon))^{2/3}~,
        & \text{if } \switchRound(\timeHorizon)=\lfloor \timeHorizon^{2/3}\rfloor
        \text{ or } \lfloor \timeHorizon/4\rfloor~, \\[4pt]
        \sqrt{\armNum\timeHorizon}~,
        & \text{if } \switchRound(\timeHorizon)=\lfloor \timeHorizon-\timeHorizon^{3/4}\rfloor
        \text{ or } \lfloor \timeHorizon-\timeHorizon^{1/2}\rfloor~.
    \end{cases}
\end{align*}
Because the theoretical result is stated in $\widetilde{\Theta}(\cdot)$ form, the multiplicative constant in front of $r_s(\timeHorizon)$ is not pinned down. 
Therefore, for each schedule $s$, we choose a scaling constant, denoted by $c_s$, so that the scaled theory curve passes through the last empirical point at $\timeHorizon_{\max}=20000$. 
That is, 
$\log \left(c_s r_s(\timeHorizon_{\max})\right)
=
\log \bar{\rew}_{\timeHorizon_{\max}}$.
Equivalently, the scaling constant is $c_s=\bar{\rew}_{\timeHorizon_{\max}}/r_s(\timeHorizon_{\max})$. 
The dashed curve in \Cref{fig:hard-instance-scaling} is then $\log(c_s r_s(\timeHorizon))$. 
This construction removes irrelevant level differences caused by constants and focuses the comparison on the slope of the theoretical rate.

For example, consider the panel with the commitment schedule $\switchRound=\lfloor \timeHorizon^{2/3}\rfloor$. 
Since the number of arms $\armNum$ is fixed, \Cref{table:opt regret} predicts a regret rate of order $\widetilde{\Theta}(\timeHorizon^{2/3})$. 
Thus, the corresponding scaled theory curve is a straight line with slope $2/3$ in the log-log plot. 
A visual calculation using the two plotted points around $\timeHorizon=1000$ and $\timeHorizon=20000$ gives an empirical slope close to
\begin{align*}
    \frac{5-3}{\log(20000/1000)}
    \approx 0.667
    \approx \frac{2}{3}~.
\end{align*}
The scaled theory curve is obtained by translating this slope-$2/3$ line vertically so that it crosses the empirical point at $\timeHorizon=20000$. 
The same construction is applied to the other schedules. 
The theoretical slopes are $3/4$, $2/3$, $1/2$, and $1/2$ for the schedules $\switchRound=\lfloor \timeHorizon^{1/2}\rfloor$, $\switchRound=\lfloor \timeHorizon/4\rfloor$, $\switchRound=\lfloor \timeHorizon-\timeHorizon^{3/4}\rfloor$, and $\switchRound=\lfloor \timeHorizon-\timeHorizon^{1/2}\rfloor$, respectively.

\begin{figure}[ht]
    \centering

\begin{tikzpicture}
  \newlength{\HardPanelWidth}
  \newlength{\HardPanelHeight}
  \newlength{\HardPanelSep}
  \setlength{\HardPanelWidth}{0.36\textwidth}
  \setlength{\HardPanelHeight}{0.24\textwidth}
  \setlength{\HardPanelSep}{1.2cm}
  \pgfplotsset{
    hard instance panel/.style={
      width=\HardPanelWidth,
      height=\HardPanelHeight,
      xlabel={$T$ (log scale)},
      ylabel={$\log \bar{R}_T$},
      xlabel style={font=\tiny, yshift=10pt},
      ylabel style={font=\tiny, yshift=-5pt},
      xtick={6.214608,6.907755,7.824046,8.517193,9.210340,9.615805,9.903488},
      xticklabels={500,1000,2500,5000,10000,15000,20000},
      xticklabel style={rotate=40, anchor=east,font=\fontsize{5pt}{6pt}\selectfont},
      yticklabel style={font=\tiny, xshift=5pt},
      ymajorgrids=true,
      xmajorgrids=true,
      grid style={gray!25},
      tick align=outside,
      title style={font=\tiny},
    },
    empirical plot/.style={
      color={rgb,255:red,31;green,119;blue,180},
      line width=1.1pt,
      draw opacity=0.35,
      mark options={solid, scale=0.9, fill={rgb,255:red,31;green,119;blue,180}, draw={rgb,255:red,31;green,119;blue,180}, fill opacity=1, draw opacity=1},
    },
    scaled theory plot/.style={
      color={rgb,255:red,255;green,127;blue,14},
      line width=0.9pt,
      dashed,
      mark=none,
    },
    slope fit plot/.style={
      color=black,
      line width=0.9pt,
      dotted,
      mark=none,
    },
  }
  \begin{groupplot}[
    hard instance panel,
    group style={group size=3 by 1, horizontal sep=\HardPanelSep},
  ]
    \nextgroupplot[title={$N=\lfloor T^{1/2}\rfloor$}, ymin=3.678345, ymax=6.894284, legend style={draw=none, fill=none, font=\tiny, at={(0.02,0.98)}, anchor=north west}, legend cell align=left]
    \addplot+[
      empirical plot,
      mark=*,
    ] coordinates {
    (6.214608, 3.999777)
    (6.907755, 4.467164)
    (7.824046, 5.189290)
    (8.517193, 5.656984)
    (9.210340, 6.184973)
    (9.615805, 6.460825)
    (9.903488, 6.660154)
    };
    \addlegendentry{Empirical}
    \addplot+[
      scaled theory plot,
    ] coordinates {
    (6.214608, 3.900134)
    (6.907755, 4.421808)
    (7.824046, 5.099081)
    (8.517193, 5.623992)
    (9.210340, 6.138802)
    (9.615805, 6.444842)
    (9.903488, 6.660154)
    };
    \addlegendentry{Scaled theory}
    \addplot+[
      slope fit plot,
    ] coordinates {
    (6.214608, 3.992028)
    (6.907755, 4.495692)
    (7.824046, 5.161501)
    (8.517193, 5.665166)
    (9.210340, 6.168831)
    (9.615805, 6.463456)
    (9.903488, 6.672495)
    };
    \addlegendentry{Slope fit}
    \nextgroupplot[title={$N=\lfloor T^{2/3}\rfloor$}, ymin=2.285372, ymax=5.234154]
    \addplot+[
      empirical plot,
      mark=*,
    ] coordinates {
    (6.214608, 2.497113)
    (6.907755, 2.984185)
    (7.824046, 3.606355)
    (8.517193, 4.117440)
    (9.210340, 4.570670)
    (9.615805, 4.821973)
    (9.903488, 5.011252)
    };
    \addplot+[
      scaled theory plot,
    ] coordinates {
    (6.214608, 2.488736)
    (6.907755, 2.968853)
    (7.824046, 3.598988)
    (8.517193, 4.071936)
    (9.210340, 4.542476)
    (9.615805, 4.816875)
    (9.903488, 5.011252)
    };
    \addplot+[
      slope fit plot,
    ] coordinates {
    (6.214608, 2.509977)
    (6.907755, 2.983642)
    (7.824046, 3.609794)
    (8.517193, 4.083459)
    (9.210340, 4.557125)
    (9.615805, 4.834201)
    (9.903488, 5.030790)
    };
    \nextgroupplot[title={$N=\lfloor T/4\rfloor$}, ymin=2.149584, ymax=5.640122]
    \addplot+[
      empirical plot,
      mark=*,
    ] coordinates {
    (6.214608, 2.417529)
    (6.907755, 2.963514)
    (7.824046, 3.628769)
    (8.517193, 4.221166)
    (9.210340, 4.784586)
    (9.615805, 5.119885)
    (9.903488, 5.399396)
    };
    \addplot+[
      scaled theory plot,
    ] coordinates {
    (6.214608, 2.940143)
    (6.907755, 3.402241)
    (7.824046, 4.013101)
    (8.517193, 4.475199)
    (9.210340, 4.937297)
    (9.615805, 5.207607)
    (9.903488, 5.399396)
    };
    \addplot+[
      slope fit plot,
    ] coordinates {
    (6.214608, 2.390311)
    (6.907755, 2.947183)
    (7.824046, 3.683326)
    (8.517193, 4.240198)
    (9.210340, 4.797069)
    (9.615805, 5.122818)
    (9.903488, 5.353940)
    };
  \end{groupplot}
  \begin{scope}[
    xshift=\dimexpr\HardPanelWidth/2+\HardPanelSep/2-1cm\relax,
    yshift=\dimexpr-\HardPanelHeight\relax,
  ]
  \begin{groupplot}[
    hard instance panel,
    group style={group size=2 by 1, horizontal sep=\HardPanelSep},
  ]
    \nextgroupplot[title={$N=\lfloor T-T^{3/4}\rfloor$}, ymin=4.568073, ymax=6.737124]
    \addplot+[
      empirical plot,
      mark=*,
    ] coordinates {
    (6.214608, 4.979815)
    (6.907755, 5.154020)
    (7.824046, 5.641233)
    (8.517193, 5.997495)
    (9.210340, 6.322210)
    (9.615805, 6.442953)
    (9.903488, 6.562102)
    };
    \addplot+[
      scaled theory plot,
    ] coordinates {
    (6.214608, 4.717663)
    (6.907755, 5.064236)
    (7.824046, 5.522382)
    (8.517193, 5.868955)
    (9.210340, 6.215529)
    (9.615805, 6.418261)
    (9.903488, 6.562102)
    };
    \addplot+[
      slope fit plot,
    ] coordinates {
    (6.214608, 4.926252)
    (6.907755, 5.238410)
    (7.824046, 5.651060)
    (8.517193, 5.963218)
    (9.210340, 6.275376)
    (9.615805, 6.457977)
    (9.903488, 6.587534)
    };
    \nextgroupplot[title={$N=\lfloor T-T^{1/2}\rfloor$}, ymin=4.562842, ymax=6.767883]
    \addplot+[
      empirical plot,
      mark=*,
    ] coordinates {
    (6.214608, 5.016283)
    (6.907755, 5.207799)
    (7.824046, 5.704867)
    (8.517193, 6.060860)
    (9.210340, 6.357317)
    (9.615805, 6.480762)
    (9.903488, 6.559354)
    };
    \addplot+[
      scaled theory plot,
    ] coordinates {
    (6.214608, 4.714914)
    (6.907755, 5.061487)
    (7.824046, 5.519633)
    (8.517193, 5.866206)
    (9.210340, 6.212780)
    (9.615805, 6.415513)
    (9.903488, 6.559354)
    };
    \addplot+[
      slope fit plot,
    ] coordinates {
    (6.214608, 4.984181)
    (6.907755, 5.290768)
    (7.824046, 5.696052)
    (8.517193, 6.002639)
    (9.210340, 6.309225)
    (9.615805, 6.488566)
    (9.903488, 6.615811)
    };
  \end{groupplot}
  \end{scope}
\end{tikzpicture}
    \caption{Log-log regret curves for \RAEC\ under five commitment schedules, each evaluated on its schedule-specific retained hard case. Each panel shows the empirical log mean regret, the scaled theory curve, and the fitted empirical slope.}
    \label{fig:hard-instance-scaling}
\end{figure}
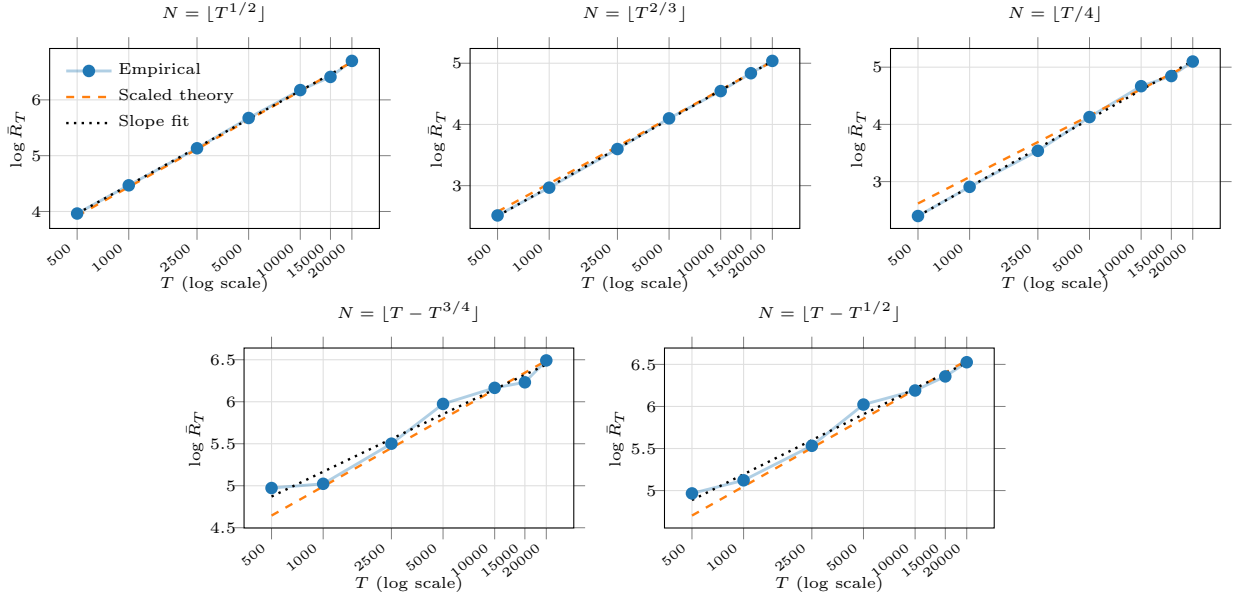

The dotted ``Slope fit'' curve is constructed independently from the theoretical prediction. 
For each schedule $s$, we regress the empirical log mean regret on the log horizon as follows:
$\log \bar{\rew}_{\timeHorizon}
=
a_s+b_s\log \timeHorizon+\xi_{\timeHorizon}$ where $\xi_{\timeHorizon}$ is the regression residual.
This regression uses the seven plotted horizons. 
The fitted line $a_s+b_s\log \timeHorizon$ is plotted as the dotted curve, and the fitted coefficient $b_s$ is the empirical scaling exponent. 
Thus, the comparison between the dotted curve and the dashed curve provides a direct visual check of whether the empirical scaling exponent is close to the theoretical exponent.

\Cref{fig:hard-instance-scaling} verifies the analysis validity of \Cref{thm:upper bound}. 
Within each panel, the empirical curve is close to the scaled theory curve over most of the plotted horizons, especially for longer horizons. 
Since the scaled theory curve is allowed only a vertical shift, this closeness indicates that the empirical growth rate is close to the theoretical growth rate. 
Moreover, the fitted empirical slope is nearly parallel to the scaled theory curve in most panels, which further suggests that the empirical exponent is close to the predicted exponent. 
Across panels, the slopes flatten in the direction predicted by the theory: regret grows fastest under the schedule $\switchRound=\lfloor \timeHorizon^{1/2}\rfloor$, consistent with the $\timeHorizon^{3/4}$-type short-experiment scaling; the balanced schedules exhibit flatter $\timeHorizon^{2/3}$-type scaling; and the late-commitment schedules flatten further toward $\timeHorizon^{1/2}$-type scaling.
The finite-horizon deviations at the smallest values of $\timeHorizon$ are expected, because the theory is asymptotic and suppresses logarithmic factors and constants. 

\subsection{{\RAEC} versus a {\UCB}-Style Baseline}
\label{subsec:raec-vs-ucb}

Our second experiment compares {\RAEC} with a natural {\UCB}-style baseline in a representative balanced-regime instance (i.e., Region II in the basic model). 

\xhdr{Experiment settings}
The length of the experiment phase is $\switchRound=\lfloor \timeHorizon/4\rfloor$, and the time horizon varies over $\timeHorizon\in\{500,1000,2000,4000,10000,15000,20000\}$. 
The number of arms is fixed at $\armNum=4$. 
For each time horizon $\timeHorizon$, the experiment-phase mean vector $\mu_{\rewFn}$ and the commitment-phase mean vector $\mu_{\rewFnCommit}$ are
\begin{align*}
    \mu_{\rewFn}
    =
    (0.8,\,0.77,\,0.2,\,0.2)~,
    \quad
    \mu_{\rewFnCommit}
    =
    (0.2,\,0.2,\,0.55+\Delta_{\rewFnCommit}(\timeHorizon),\,0.55)~,
\end{align*}
where the commitment-phase gap is
$\Delta_{\rewFnCommit}(\timeHorizon)
=
\frac{0.4}{\sqrt{\log \switchRound}}
=
\frac{0.4}{\sqrt{\log(\lfloor \timeHorizon/4\rfloor)}}$.
This instance is designed to separate the experiment-phase and commitment-phase objectives. 
Arms $1$ and $2$ are attractive under the experiment-phase reward $\rewFn$, while arms $3$ and $4$ are much less attractive during experimentation but are the two best arms under the commitment-phase reward $\rewFnCommit$. 
Among these two better arms to commit, arm $3$ is the $\rewFnCommit$-optimal arm, and arm $4$ is the second-best arm to commit with gap $\Delta_{\rewFnCommit}(\timeHorizon)$. 
Thus, a policy that focuses primarily on experiment-phase regret minimization may undersample the two arms that are most relevant to the commitment decision.

The {\UCB}-style baseline operates as follows. 
During the experiment phase, it applies the standard {\UCB} rule using the experiment-phase reward. 
After each arm has been pulled once, the {\UCB} index $\UCB^{\rewFn}_{i,t}$ of arm $i\in[\armNum]$ in period $t\in[\switchRound]$ is
\begin{align*}
    \UCB^{\rewFn}_{i,t}
    =
    \empiricalmeanf{i,t}
    +
    \sqrt{\frac{2\log(t)}{n_{i,t}}}~,
\end{align*}
where the experiment-phase empirical mean $\empiricalmeanf{i,t}$ is arm $i$'s empirical mean reward for the experiment phase up to period $t$, and the pull count $n_{i,t}$ is the number of times arm $i$ has been pulled up to period $t$. 
At the same time, because every pulled outcome can also be evaluated under the commitment-phase reward function, the baseline keeps track of the commitment-phase empirical mean $\empiricalmeang{i,t}$ for every arm $i$. 
At the commitment date, the baseline commits to the arm with the largest commitment-phase empirical mean, i.e., $\argmax_{i\in[\armNum]}\empiricalmeang{i,\switchRound}$.

For {\RAEC}, we use the same two-stage structure as in \Cref{main algo}, but with a constant speedup adjustment to avoid overly conservative exploration at the plotted horizons. 
In particular, during Stage I, the sampling length of each $\rewFnCommit$-elimination epoch is set to $(1/25)\roundlen{\rewFnCommit}{\ell}$, and during Stage II, the sampling length of each $\rewFn$-elimination epoch is set to $(1/10)\roundlen{\rewFn}{\ell}$. 
These constant factors affect only the finite-sample implementation and preserve the intended structure of {\RAEC}: Stage I reserves exploration for the commitment-phase reward, while Stage II focuses on experiment-phase regret minimization. 
The final comparison uses $2500$ independent replications for each time horizon $\timeHorizon$.

\begin{figure}[h]
    \centering
    \resizebox{0.6\textwidth}{!}{%
\begin{tikzpicture}
  \begin{axis}[
    width=\linewidth,
    height=0.60\linewidth,
    xlabel={T},
    ylabel={Mean total regret},
    xmin=0,
    xmax=20600,
    ymin=0,
    ymax=367.693194,
    xtick={500,1000,2000,4000,10000,15000,20000},
    xticklabels={500,1000,2000,4000,10000,15000,20000},
    xticklabel style={rotate=30, anchor=east},
    ymajorgrids=true,
    xmajorgrids=true,
    grid style={gray!25},
    tick align=outside,
    legend style={draw=none, fill=none, font=\small, at={(0.03,0.97)}, anchor=north west},
    scaled x ticks=false,
  ]
    \addplot+[
      draw=none,
      fill={rgb,255:red,22;green,50;blue,79},
      fill opacity=0.14,
      forget plot,
      mark=none,
    ] coordinates {
    (500, 29.384025)
    (1000, 41.680764)
    (2000, 63.582115)
    (4000, 95.727317)
    (10000, 140.899307)
    (15000, 144.814585)
    (20000, 179.403924)
    (20000, 202.385753)
    (15000, 162.197238)
    (10000, 155.525035)
    (4000, 102.571787)
    (2000, 70.405691)
    (1000, 45.258218)
    (500, 31.44131)
    } \closedcycle;

    \addplot+[
      draw=none,
      fill={rgb,255:red,45;green,127;blue,191},
      fill opacity=0.14,
      forget plot,
      mark=none,
    ] coordinates {
    (500, 26.480517)
    (1000, 38.965403)
    (2000, 63.953081)
    (4000, 102.143764)
    (10000, 182.6995)
    (15000, 248.476659)
    (20000, 294.738699)
    (20000, 340.456661)
    (15000, 284.65264)
    (10000, 207.657347)
    (4000, 114.819492)
    (2000, 70.861932)
    (1000, 42.667484)
    (500, 28.656362)
    } \closedcycle;

    \addplot+[
      color={rgb,255:red,22;green,50;blue,79},
      line width=1.3pt,
      mark=*,
      mark options={solid, scale=0.9},
    ] coordinates {
    (500, 30.412667)
    (1000, 43.469491)
    (2000, 66.993903)
    (4000, 99.149552)
    (10000, 148.212171)
    (15000, 153.505912)
    (20000, 190.894838)
    };
    \addlegendentry{RAEC}
    \addplot+[
      color={rgb,255:red,45;green,127;blue,191},
      line width=1.3pt,
      mark=square*,
      mark options={solid, scale=0.9},
    ] coordinates {
    (500, 27.56844)
    (1000, 40.816443)
    (2000, 67.407507)
    (4000, 108.481628)
    (10000, 195.178423)
    (15000, 266.564649)
    (20000, 317.59768)
    };
    \addlegendentry{UCB}
    \addlegendimage{empty legend}
    \addlegendentry{Shaded bands: 95\% CI}
  \end{axis}
\end{tikzpicture}}
    \caption{Mean total regret of {\RAEC} and the {\UCB} baseline. {\RAEC} exhibits a long-horizon mean-regret advantage.}
    \label{fig:raec-vs-ucb}
\end{figure}
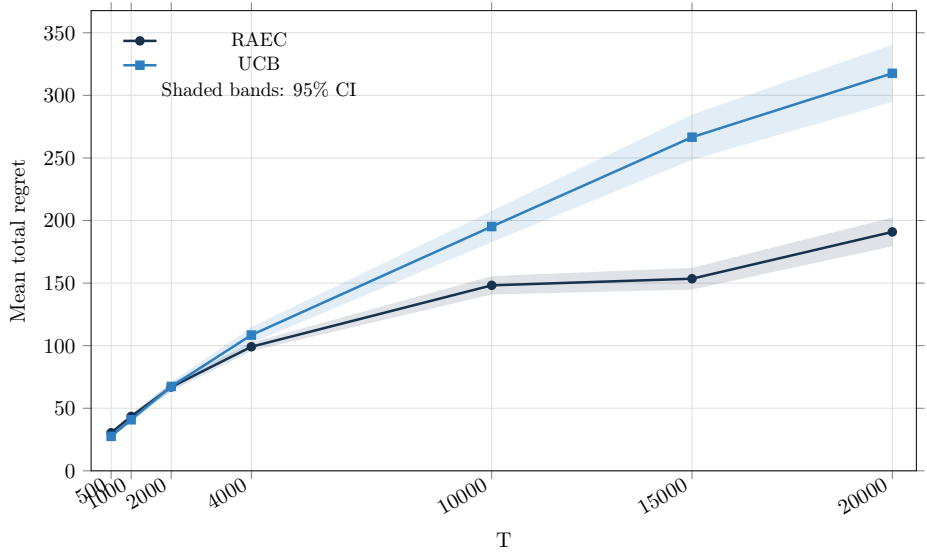

\xhdr{Experiment results}
\Cref{fig:raec-vs-ucb} reports the mean total regret of {\RAEC} and the {\UCB} baseline. 
The shaded area represents the $95\%$ confidence intervals. 
The figure shows a clear long-horizon advantage for {\RAEC}. 
At short horizons, {\RAEC} pays an additional exploration cost because it deliberately reserves part of the experiment phase to learn the commitment-phase reward. 
As the horizon grows, however, the commitment phase becomes more important, and the benefit of making a more accurate post-commitment decision dominates this early exploration cost. 
As a result, the mean total regret of {\RAEC} grows more slowly than that of the {\UCB} baseline.

The observed pattern is also consistent with our theoretical analysis. 
In this instance, arms $3$ and $4$ are substantially suboptimal under the experiment-phase reward $\rewFn$. 
A standard {\UCB} policy therefore pulls each of these $\rewFn$-suboptimal arms for only $O(\log \switchRound)$ times. 
Consequently, the commitment-phase empirical comparison between arms $3$ and $4$ is based on only $O(\log \switchRound)$ samples per arm. 
The $\rewFnCommit$-reward gap between these two arms satisfies
$ \Delta_{\rewFnCommit}(\timeHorizon)
=
O(1/\sqrt{\log \switchRound})$.
This gap is of the same order as the statistical noise in the empirical $\rewFnCommit$-mean estimates after $O(\log \switchRound)$ samples. 
Thus, this amount of sampling is not enough to drive the probability of confusing arms $3$ and $4$ to zero. 
Therefore, the {\UCB} baseline can commit to the $\rewFnCommit$-suboptimal arm $4$ with a non-negligible probability.

When this miscommitment occurs, the per-period commitment loss is $\Delta_{\rewFnCommit}(\timeHorizon)$. 
The resulting commitment regret is therefore on the order of
$(\timeHorizon-\switchRound)\Delta_{\rewFnCommit}(\timeHorizon)
=
O((\timeHorizon-\switchRound)/\sqrt{\log \switchRound})$.
For the schedule $\switchRound=\lfloor \timeHorizon/4\rfloor$, this term is of order $O(\timeHorizon/\sqrt{\log \timeHorizon})$. 
This growth is substantially larger than the $\widetilde{O}(\timeHorizon^{2/3})$ regret rate that {\RAEC} can attain in the balanced regime. 
This explains the increasing performance gap in \Cref{fig:raec-vs-ucb}: the {\UCB} baseline performs well for experiment-phase reward collection, but it does not allocate enough samples to accurately distinguish the two best commitment arms. 
{\RAEC} incurs short-term exploration costs, but its reserved commitment-phase exploration reduces the probability of a costly long-run commitment error.

\subsection{Concave Commitment Reward}
\label{subsec:concave-numerics}
Our third experiment evaluates the generalized concave commitment model in \Cref{subsec:concave}.

\xhdr{Experiment settings} 
We compare {\ROSCOC} with a {\UCB}-style portfolio baseline, which we call {\UCB}-Portfolio. 
The length of the experiment phase is $\switchRound=\lfloor \timeHorizon/4\rfloor$, and the time horizon varies over $\timeHorizon\in\{500,1000,2000,4000,10000,15000,20000\}$. 
The number of arms is fixed at $\armNum=5$.

The experiment-phase rewards are Bernoulli random variables. 
The experiment-phase mean vector $\mu_{\rewFn}$ is
\begin{align*}
    \mu_{\rewFn}
    =
    \left(
        0.55+\frac{\Delta_{\rewFn}}{2},
        \,0.55-\frac{\Delta_{\rewFn}}{2},
        \,0.16,
        \,0.08,
        \,0.02
    \right)~,
\end{align*}
where the experiment-phase gap $\Delta_{\rewFn}$ is $\Delta_{\rewFn}=1/\sqrt{\switchRound}$. 
Thus, arms $1$ and $2$ are the most attractive arms in the experiment phase, while arms $3$, $4$, and $5$ are substantially suboptimal under the short-run reward $\rewFn$.

For the commitment phase, each arm $i\in[\armNum]$ is associated with an intrinsic three-dimensional outcome vector $\outcome_i=(\outcome_{i1},\outcome_{i2},\outcome_{i3})^\top$. 
When arm $i$ is pulled during the experiment phase, the learner observes a realized outcome vector whose coordinates are independent Bernoulli random variables with mean vector $\outcome_i$. 
The intrinsic outcome vectors are
\begin{align*}
    \outcome_1
    &=
    (0.90,\,0.18,\,0.34)^\top~,\quad
    \outcome_2
    =
    (0.08,\,0.90,\,0.32)^\top~, \quad
    \outcome_3
    =
    (0.08,\,0.64,\,0.18)^\top, \\
    \outcome_4
    &=
    (0.60,\,0.10,\,0.18)^\top~, \quad
    \outcome_5
    =
    (0.28,\,0.72,\,0.18)^\top ~.
\end{align*}
The normalized reward for the commitment phase has the quadratic form
$\rewFnCommit(\outcome)
=
-20\cdot(\outcome-\bs{c})^\top A(\outcome-\bs{c})$,
where the target vector $\bs{c}$ is $\bs{c}=(0.33,\,0.54,\,0.18)^\top$, and the matrix $A$ is
\begin{align*}
    A
    =
    \begin{bmatrix}
        2.8 & 2.3 & 0 \\
        2.3 & 2.6 & 0 \\
        0 & 0 & 0.75
    \end{bmatrix}~.
\end{align*}
By construction, the target vector $\bs{c}$ satisfies
$\bs{c}
\approx
0.18\outcome_3+0.27\outcome_4+0.55\outcome_5$.
Thus, the optimal commitment portfolio assigns weights approximately $0.18$, $0.27$, and $0.55$ to arms $3$, $4$, and $5$, respectively, and assigns zero weight to arms $1$ and $2$. 
This construction creates a conflict between the short-run and long-run objectives: the arms that are most useful for the commitment portfolio are the arms that appear unattractive during the experiment phase.

The {\UCB}-Portfolio baseline operates as follows. 
During the experiment phase, it applies the standard {\UCB} rule using the experiment-phase reward $\rewFn$, as described in \Cref{subsec:raec-vs-ucb}. 
At the same time, it records the realized outcome vectors from the arms it pulls and uses these observations to estimate each intrinsic outcome vector. 
The empirical mean outcome vector of arm $i$ at the end of the experiment phase is denoted by $\hat{\outcome}_{i,\switchRound}$. 
At the commitment date, {\UCB}-Portfolio computes the estimated optimal commitment portfolio $\widehat{\bs{x}}_{\switchRound}^{\mathrm{UCB}}$ by solving
\begin{align*}
    \widehat{\bs{x}}_{\switchRound}^{\mathrm{UCB}}
    \in
    \argmax\nolimits_{\bs{x}\in\Delta([\armNum])}
    -
    \left(
        \hat{\outcome}(\bs{x})-\bs{c}
    \right)^\top
    A
    \left(
        \hat{\outcome}(\bs{x})-\bs{c}
    \right)~,
\end{align*}
where the probability simplex $\Delta([\armNum])$ is $\Delta([\armNum])
=
\{
    \bs{x}\in\mathbb{R}_+^{\armNum}:
    \sum_{i=1}^{\armNum}x_i=1
\}$
and the estimated portfolio outcome $\hat{\outcome}(\bs{x})$ is
$\hat{\outcome}(\bs{x})
=
\sum\nolimits_{i\in[\armNum]}x_i\cdot \hat{\outcome}_{i,\switchRound}$.
The baseline then uses this estimated portfolio in the commitment phase. 
Thus, {\UCB}-Portfolio's limitation is that the experiment-phase sampling rule is driven by the short-run reward $\rewFn$, rather than by the need to estimate the commitment-relevant outcome vectors.

For {\ROSCOC}, we use the reserved exploration structure described in \Cref{subsec:concave}. 
To avoid excessive finite-sample exploration at the plotted horizons, we use $0.3\cdot\tau$ as the Stage-I length, where $\tau$ is the theoretical Stage-I budget in \Cref{thm:regret concave}. 
In Stage II, we use $(1/10)m_{\rewFn,\ell}$ as the sampling length of each arm-elimination epoch. 
These constant adjustments preserve the main structure of {\ROSCOC}: Stage I learns a good commitment portfolio through reserved exploration, while Stage II controls experiment-phase regret. 
The comparison uses $200$ independent replications for each time horizon $\timeHorizon$.

\xhdr{Experiment results}
\Cref{fig:concave-regret} reports the mean total regret of {\ROSCOC} and {\UCB}-Portfolio. 
The shaded areas represent the $95\%$ confidence intervals. 
{\ROSCOC} has lower mean total regret at every plotted horizon, and the gap increases as $\timeHorizon$ grows. 
In particular, the regret of {\UCB}-Portfolio grows nearly linearly over longer horizons, while the growth of {\ROSCOC} is substantially slower. 
This pattern is consistent with the theoretical message of \Cref{subsec:concave}: even in the generalized concave commitment setting, a predetermined reservation of exploration effort remains effective.

The reason for the performance difference is similar to that in the previous experiment, but now the commitment decision is on a portfolio rather than a single arm. 
In this instance, arms $3$, $4$, and $5$ are highly suboptimal under the experiment-phase reward $\rewFn$, but they are exactly the arms that form the optimal commitment portfolio. 
A {\UCB} policy that focuses on the experiment-phase reward pulls each of these $\rewFn$-suboptimal arms only on the order of $\log\switchRound$ times. 
As a result, the empirical outcome vectors $\hat{\outcome}_{3,\switchRound}$, $\hat{\outcome}_{4,\switchRound}$, and $\hat{\outcome}_{5,\switchRound}$ have coordinate-wise estimation errors on the order of
$O\left(1/\sqrt{\log\switchRound}\right)$.
These errors then enter the quadratic portfolio optimization problem used at the commitment date. 
Since the commitment payoff depends on the distance between the realized portfolio outcome and the target vector $\bs{c}$, inaccurate estimates of the commitment-relevant arms can lead to a poorly chosen portfolio. 
The worst-case commitment regret is on the order of
$O\left((\timeHorizon-\switchRound)/\sqrt{\log\switchRound}\right)$.
For the schedule $\switchRound=\lfloor \timeHorizon/4\rfloor$, this scale becomes $O\left(\timeHorizon/\sqrt{\log\timeHorizon}\right)$. 
This growth is substantially worse than the $\widetilde{O}\left(\timeHorizon^{2/3}\right)$ regret rate that {\ROSCOC} can attain in the balanced regime.

This comparison explains the increasing gap in \Cref{fig:concave-regret}. 
{\UCB}-Portfolio spends most of its experiment-phase samples on arms that are good for short-run reward collection, but these samples are not the most useful for estimating the long-run optimal portfolio. 
Similar to {\RAEC}, {\ROSCOC} pays a predetermined exploration cost to learn the commitment-phase reward structure, and this investment leads to a more accurate portfolio choice and lower long-horizon regret.

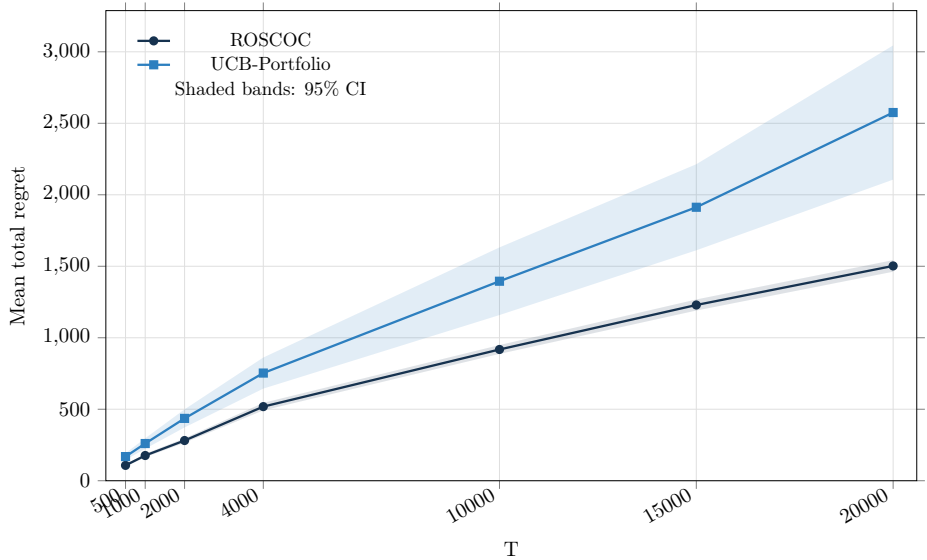
\begin{figure}[h]
    \centering
    \resizebox{0.6\textwidth}{!}{%
\begin{tikzpicture}
  \begin{axis}[
    width=\linewidth,
    height=0.62\linewidth,
    xlabel={T},
    xmin=0,
    xmax=20600,
    ylabel={Mean total regret},
    ymin=0,
    ymax=3288.134174,
    xtick={500,1000,2000,4000,10000,15000,20000},
    xticklabels={500,1000,2000,4000,10000,15000,20000},
    xticklabel style={rotate=30, anchor=east},
    ymajorgrids=true,
    xmajorgrids=true,
    grid style={gray!25},
    tick align=outside,
    legend style={draw=none, fill=none, font=\small, at={(0.03,0.97)}, anchor=north west},
    scaled x ticks=false,
  ]
    \addplot+[
      draw=none,
      fill={rgb,255:red,22;green,50;blue,79},
      fill opacity=0.14,
      forget plot,
      mark=none,
    ] coordinates {
    (500, 99.167299)
    (1000, 164.117938)
    (2000, 265.779494)
    (4000, 491.918821)
    (10000, 887.002782)
    (15000, 1190.225645)
    (20000, 1461.703794)
    (20000, 1543.004095)
    (15000, 1267.915814)
    (10000, 948.899679)
    (4000, 544.766775)
    (2000, 296.705464)
    (1000, 188.74116)
    (500, 115.670443)
    } \closedcycle;

    \addplot+[
      draw=none,
      fill={rgb,255:red,45;green,127;blue,191},
      fill opacity=0.14,
      forget plot,
      mark=none,
    ] coordinates {
    (500, 145.594468)
    (1000, 223.446613)
    (2000, 372.701056)
    (4000, 644.382394)
    (10000, 1158.756819)
    (15000, 1611.686927)
    (20000, 2106.002061)
    (20000, 3044.568679)
    (15000, 2213.756308)
    (10000, 1631.915651)
    (4000, 861.77587)
    (2000, 499.283401)
    (1000, 296.895541)
    (500, 192.137786)
    } \closedcycle;

    \addplot+[
      color={rgb,255:red,22;green,50;blue,79},
      line width=1.2pt,
      mark=*,
      mark options={solid, scale=0.9},
    ] coordinates {
    (500, 107.418871)
    (1000, 176.429549)
    (2000, 281.242479)
    (4000, 518.342798)
    (10000, 917.95123)
    (15000, 1229.07073)
    (20000, 1502.353944)
    };
    \addlegendentry{ROSCOC}
    \addplot+[
      color={rgb,255:red,45;green,127;blue,191},
      line width=1.2pt,
      mark=square*,
      mark options={solid, scale=0.9},
    ] coordinates {
    (500, 168.866127)
    (1000, 260.171077)
    (2000, 435.992229)
    (4000, 753.079132)
    (10000, 1395.336235)
    (15000, 1912.721617)
    (20000, 2575.28537)
    };
    \addlegendentry{UCB-Portfolio}
    \addlegendimage{empty legend}
    \addlegendentry{Shaded bands: 95\% CI}
  \end{axis}
\end{tikzpicture}
    }
    \caption{Mean total regret of {\ROSCOC} and the {\UCB}-Portfolio baseline. {\ROSCOC} exhibits a long-horizon mean-regret advantage in the concave commitment setting.}
    \label{fig:concave-regret}
\end{figure}

\section{Conclusion}

We study adaptive experimentation with post-commitment reward shifts, motivated by operational settings in which actions that perform well before an anticipated environmental change may be poorly suited for the subsequent long-run commitment. For the baseline model, we develop {\RAEC}, a simple algorithm that reserves a predetermined portion of the experiment phase for identifying a good commitment arm while using the remaining rounds to control short-run regret. We derive regret guarantees across all parameter regimes and establish matching information-theoretic lower bounds, yielding a tight characterization of the cost of balancing short-term performance and long-term commitment. A central implication is that predetermined effort allocation is sufficient at the level of minimax regret, while more sophisticated adaptive allocation based on the learned instance structure cannot improve the leading regret rate.
We also study two extensions. With structural knowledge of reward shifts, we show that ranking changes matter more than shift magnitudes. For concave commitment rewards, ROSCOC converts its reserved exploration path directly into a commitment portfolio and achieves the same regret order as {\RAEC} under mild assumptions. Numerical experiments support the predicted scaling and demonstrate long-horizon gains over UCB-style baselines.

Interesting future research questions include algorithm design when the commitment time is exogenously given but remains unknown (i.e., $\switchRound$ is unknown) to the decision-maker or when reward functions $\rewFn$ and $\rewFnCommit$ admit more fine-grained structural relationships, etc. 

\setlength{\bibsep}{0.0pt}
\bibliographystyle{plainnat}
\OneAndAHalfSpacedXI
{\footnotesize
\bibliography{ref}}

\ECSwitch
\ECDisclaimer

\crefalias{section}{appendix}
\crefalias{subsection}{appendix}
\crefalias{subsubsection}{appendix}

\makeatletter

\renewcommand{\theHsection}{appendix.section.\thesection}
\providecommand*{\theHsubsection}{\thesubsection}
\renewcommand{\theHsubsection}{appendix.subsection.\thesubsection}

\providecommand*{\theHequation}{\theequation}
\renewcommand{\theHequation}{appendix.equation.\theequation}

\providecommand*{\theHfigure}{\thefigure}
\renewcommand{\theHfigure}{appendix.figure.\thefigure}

\providecommand*{\theHtable}{\thetable}
\renewcommand{\theHtable}{appendix.table.\thetable}

\providecommand*{\theHtheorem}{\thetheorem}
\renewcommand{\theHtheorem}{appendix.theorem.\thetheorem}

\providecommand*{\theHlemma}{\thelemma}
\renewcommand{\theHlemma}{appendix.lemma.\thelemma}

\providecommand*{\theHproposition}{\theproposition}
\renewcommand{\theHproposition}{appendix.proposition.\theproposition}

\providecommand*{\theHcorollary}{\thecorollary}
\renewcommand{\theHcorollary}{appendix.corollary.\thecorollary}

\providecommand*{\theHexample}{\theexample}
\renewcommand{\theHexample}{appendix.example.\theexample}

\makeatother

\section{Additional Real-world Examples}
\label{apx:example}
\begin{example}[Manufacturing supply chain examples]\label{exam:manufacture}
In manufacturing contexts, state-level regulations limiting gas vehicle development, scheduled enactment of carbon tariffs, and other environmental regulations put automobile manufacturers in a similar position (see \citealp{reuters_eu_automakers_breathing_space_2025}).
The profitability of different car models can be impacted quite differently depending on their carbon emission volumes.
In some industrial contexts, such as the semiconductor industry, export restrictions led by geopolitical tensions force firms in key supply chains to explore substitute solutions before they lose access to critical components or technology (see \citealp{economist_china_reducing_foreign_chip_reliance_2024}).
Different supply chains can be impacted differently depending on their resilience to those export restrictions.
As another example, consider manufacturing firms that would experience per-unit cost reduction after mass production launches.
The mass production facilities often start construction long before the experiments and trial production end.
Some production plans may enjoy more cost reduction than others during mass production. For example, products with cost structures weighted more towards manufacturing expenses can benefit more from automation, learning curve effects, and other factors.
Importantly, firms in manufacturing supply chains typically reserve a portion of their operating capacity for experiments during the grace period, after which mass production of substitutes or transitions of business structures across all capacities must roll out.
\end{example}

\begin{example}[The manufacturing supply chain examples (continued)]\label{exam:sc}
One treatment $\pickArm_t\in[\armNum]$ could represent one new production plan.
The length of the experiment phase $\switchRound$ is the length of the grace period before the enactment of carbon tariffs, export restrictions, cost reduction due to mass production, or some other policy environment shifts.
The length of the commitment phase $\timeHorizon-\switchRound$ is the market life of one generation of product.
The commitment date is marked by the enactment of the carbon tariffs, export restrictions, full production capacity transition, and so on.
The random outcome $\outcome_{t,\pickArm_t}$ consists of two attributes, e.g., $(\textsc{revenue}_{t,\pickArm_t}, \textsc{carbon\text{-}emission}_{t,\pickArm_t})$ or $(\textsc{revenue}_{t,\pickArm_t}, \textsc{risky\text{-}supplier}_{t,\pickArm_t})$.
$\textsc{revenue}_{t,\pickArm_t}$ is a random variable, while $\textsc{carbon\text{-}emission}_{t,\pickArm_t}=\textsc{carbon\text{-}emission}_{\pickArm_t}$ records the carbon emission volume and $\textsc{risky\text{-}supplier}_{t,\pickArm_t}=\textsc{risky\text{-}supplier}_{\pickArm_t} \in \{0, 1\}$ takes binary values indicating whether the production involves risky suppliers, and are constants for a given treatment $\pickArm_t$.
Before the commitment, the reward function is again simply $\rewFn(\outcome_{t,\pickArm_t})=\textsc{revenue}_{t,\pickArm_t}$.
After the enactment of the carbon tariffs, the reward function can have the form $\rewFnCommit(\outcome_{t,\pickArm_t})=\textsc{revenue}_{t,\pickArm_t}-\textsc{carbon\text{-}emission}_{\pickArm_t}$.
Or, after the enactment of export restrictions, the reward function can have the form $\rewFnCommit(\outcome_{t,\pickArm_t})=\left(1-\textsc{risky\text{-}supplier}_{\pickArm_t}\right)\cdot \textsc{revenue}_{t,\pickArm_t}$.
\end{example}

\begin{example}[Adaptive clinical trial examples]\label{exam:clinical}
The FDA's guidance highlights the application of adaptive experiments in clinical trials \citep[Sections II.A and V.E]{fda_adaptive_designs_2019}.
Typically, the experiment recruits $\switchRound$ patients in advance and treats them sequentially.
The researchers can adaptively choose a treatment for each patient based on the outcomes observed from previously treated patients.
When the trial ends, the researchers select the treatment judged to be the best for future production.
However, the researchers also care about the aggregated welfare of all the patients participating in the trial.
Therefore, the objective can be a Lagrangian summation of the treated patients' welfare and the goodness of the final selection.
Those two parts of the objective may not always be aligned, since, from a commercial perspective, the ``goodness'' of the final selection may involve factors not related to patients' welfare, such as production cost.

\end{example}

\begin{example}[Carbon emission example]\label{exam:concave}
Continuing the carbon emission example introduced in \Cref{exam:sc}.
The random outcome $\outcome_{t,\pickArm_t}$ still consists of two attributes, $(\textsc{revenue}_{t,\pickArm_t}, \textsc{carbon\text{-}emission}_{t,\pickArm_t})$.
Now, after the enactment of certain carbon tax policies, the carbon tax is a convex increasing function of the total emission over a long monitored period, the commitment phase.
For example, the total reward for the commitment phase has the form $\sum_{t=\switchRound+1}^{\timeHorizon}\textsc{revenue}_{t,\pickArm_t}-\tax \left( \sum_{t=\switchRound+1}^{\timeHorizon}\textsc{carbon\text{-}emission}_{t, \pickArm_t} \right) $.
Here, $\tax \left( \cdot \right) $ is a convex increasing function, which makes the total reward for the commitment phase a concave function.
And due to reasons like long-term strategic planning requirements, or the inflexibility of mass-production line construction and operations, the firm needs to commit to a production plan that may involve a portfolio of products, with each product taking a specific, committed portion of total capacity.
\end{example}

%

\section{Missing Proofs in \texorpdfstring{\Cref{sec:upper bound}}{}}



\begin{proof}[Proof for \Cref{lem:regret1}]
Let $\numround=\min\{\epoch:\roundlen{\rewFn}{\epoch}\geq \switchRound\}$ be the largest possible epoch index involved before the experiment cap $\switchRound$ is reached.
We recall that $\optf\in[\armNum]$ denotes the true optimal arm of the experiment phase.
We first consider the event that all empirical means used by the arm-elimination process for the reward function $\rewFn$ are accurate. In particular, let
\begin{align*}
\mathcal{G}_{\rewFn}=\left\{\left|\empiricalmeanf{i,\epoch}-\mu_{\rewFn,i}\right|\leq \frac{\myDelta{\epoch}}{2},~\text{for every empirical mean computed before the experiment cap $\switchRound$}\right\}~.
\end{align*}
Here, $\empiricalmeanf{i,\epoch}$ is the average of the first $\roundlen{\rewFn}{\epoch}$ rewards observed from arm $i$. We start from setting $\roundlen{\rewFn}{\epoch}=\frac{\epochfactor\log(\switchRound)}{\elimiCriteria_\epoch^2}$, where $\epochfactor$ is a tuning parameter that we will pick later. 
Then, Azuma's inequality gives
\begin{align*}
\prob{\left|\empiricalmeanf{i,\epoch}-\mu_{\rewFn,i}\right|>\frac{\myDelta{\epoch}}{2}}
&\leq 2\myexp{-\frac{\roundlen{\rewFn}{\epoch}\myDelta{\epoch}^2}{2}}
=\frac{2}{\switchRound^{\epochfactor/2}}~.
\end{align*}
Therefore, by a union bound,
we have $$\prob{\mathcal{G}_{\rewFn}^c}
\leq \frac{2\armNum\numround}{\switchRound^{\epochfactor/2}}~.$$
Conditional on event $\mathcal{G}_{\rewFn}$, arm $\optf$ is never eliminated. This is due to the fact that for any active arm $i$ in epoch $\epoch$,
$\empiricalmeanf{i,\epoch}-\empiricalmeanf{\optf,\epoch}
\leq \mu_{\rewFn,i}-\mu_{\rewFn}^*+\myDelta{\epoch}
\leq \myDelta{\epoch}$.
For every $i\in\mathcal{S}_{\rewFn}$, define
$\marginroundf{i}=\min\left\{\epoch:\myDelta{\epoch}\leq\Deltaf{i}/4\right\}$.
If epoch $\marginroundf{i}$ is completed, then on event $\mathcal{G}_{\rewFn}$,
we have $\empiricalmeanf{\optf,\marginroundf{i}}-\empiricalmeanf{i,\marginroundf{i}}
\geq \Deltaf{i}-\myDelta{\marginroundf{i}}
>\myDelta{\marginroundf{i}}$,
and hence arm $i$ is eliminated after this epoch. If the experiment cap $\switchRound$ is reached before this epoch is completed, the number of pulls charged to the arm-elimination process for arm $i$ is no larger than $\roundlen{\rewFn}{\marginroundf{i}}$. Thus, the following argument also covers the cases in which arm $i$ remains active until the experiment cap $\switchRound$. Since $\myDelta{\marginroundf{i}}>\Deltaf{i}/8$, we have
$\roundlen{\rewFn}{\marginroundf{i}}
\leq \frac{64\epochfactor\log(\switchRound)}{\Deltaf{i}^2}$.
On the complement event $\mathcal{G}_{\rewFn}^c$, the regret is trivially bounded by $\switchRound$. 

Combining the above bounds, we have
\begin{align*}
    \Rega{\switchRound}
    \leq \sum\nolimits_{i\in\mathcal{S}_{\rewFn}}\Deltaf{i}\roundlen{\rewFn}{\marginroundf{i}}
    +\switchRound\prob{\mathcal{G}_{\rewFn}^c}
    \mathrel{=} O\left(\sum\nolimits_{i\in\mathcal{S}_{\rewFn}}\frac{\log(\switchRound)}{\Deltaf{i}}
    +\frac{\armNum\numround}{\switchRound^{\epochfactor/2-1}}\right)
    \mathrel{=} O\left(\sum\nolimits_{i\in\mathcal{S}_{\rewFn}}\frac{\log(\switchRound)}{\Deltaf{i}}\right)~,
\end{align*}
where the last equality holds for $\epochfactor=4$, $\numround=O(\log(\switchRound))$, and $\armNum/\switchRound\mathrel{=} O(1)$. 
%
\end{proof}


\begin{proof}[Proof for \Cref{lem:regret2}]
For each arm and each $\epoch\leq\epochNum$, consider the average of its first $\roundlen{\rewFnCommit}{\epoch}=\frac{\epochfactorg\log(\timeHorizon-\switchRound)}{\elimiCriteria_\epoch^2}$ observations. 
Analogous to the proof for \Cref{lem:regret1}, let $\mathcal{G}_{\rewFnCommit}$ be the event that all these averages are within $\myDelta{\epoch}/2$ of their means. Since the rewards lie in $[0,1]$, the concentration argument of \Cref{lem:regret1}, with $\epochfactorg=4$, bounds the failure probability for each arm and epoch by $2\exp(-\roundlen{\rewFnCommit}{\epoch}\myDelta{\epoch}^2/2)=2/(\timeHorizon-\switchRound)^2$. A union bound gives
\begin{equation}
\prob{\mathcal{G}_{\rewFnCommit}^c}\leq\frac{2\armNum\epochNum}{(\timeHorizon-\switchRound)^2}~.
\label{eqn:good-event-fg}
\end{equation}
On every sample path, Stage I has length at most $B_{\stopCriteria}=\min\{\switchRound,\armNum\roundlen{\rewFnCommit}{\epochNum}\}$, therefore, the algorithm's regret under event $\mathcal{G}_{\rewFnCommit}^c$ is at most $B_{\stopCriteria}$.

On the other hand, for $\mathcal{G}_{\rewFnCommit}$, we notice that the $\rewFnCommit-$optimal arm is retained. 
This is because for any active arm $j$ at a completed epoch $\epoch$,
$\empiricalmeang{j,\epoch}-\empiricalmeang{\optg,\epoch}
\leq\mu_{\rewFnCommit,j}-\mu_{\rewFnCommit}^*+\myDelta{\epoch}
\leq\myDelta{\epoch}$.
Use $\gamma_{\rewFnCommit,i}$ as defined before \Cref{prop:instance-dependent ub}, and set
$\marginroundg{i}=\min\{\epoch\geq1:\myDelta{\epoch}\leq\gamma_{\rewFnCommit,i}/4\}$.
For a $\rewFnCommit$-suboptimal arm $i$, we have $\gamma_{\rewFnCommit,i}=\Deltag{i}$. If it is still active when epoch $\marginroundg{i}\leq\epochNum$ is completed, then due to the definition of event $\mathcal{G}_{\rewFnCommit}$, we have
\[
\empiricalmeang{\optg,\marginroundg{i}}-\empiricalmeang{i,\marginroundg{i}}
\geq\Deltag{i}-\myDelta{\marginroundg{i}}
>\myDelta{\marginroundg{i}}~,
\]
so it is eliminated at the end of this epoch.

Since $\optg$ is unique, $\gamma_{\rewFnCommit,\optg}>0$. At epoch $\marginroundg{\optg}$, every competitor $j\ne\optg$ has $\Deltag{j}\geq\gamma_{\rewFnCommit,\optg}\geq4\myDelta{\marginroundg{\optg}}$. Hence, if this epoch is reached and completed, only $\optg$ remains.
According to {\RAEC}, in this case, the algorithm takes no further Stage-I exploration. And the number of samples of $\optg$ in this last elimination epoch is $\roundlen{\rewFnCommit}{\marginroundg{\optg}}$. 

Therefore, on event $\mathcal{G}_{\rewFnCommit}$, the number $n_i^{\mathrm I}$ of Stage-I pulls of arm $i$ satisfies, 
\[
n_i^{\mathrm I}\leq\roundlen{\rewFnCommit}{\min\{\marginroundg{i},\epochNum\}}
\leq\frac{256\log(\timeHorizon-\switchRound)}{\max\{\gamma_{\rewFnCommit,i},\stopCriteria\}^2}~.
\]
As a sanity check, notice that $\myDelta{\epochNum}\geq\stopCriteria/2$ gives $\roundlen{\rewFnCommit}{\epochNum}\leq16\log(\timeHorizon-\switchRound)/\stopCriteria^2$. And the definition of $\marginroundg{i}$ gives $\myDelta{\marginroundg{i}}>\gamma_{\rewFnCommit,i}/8$, and hence $\roundlen{\rewFnCommit}{\marginroundg{i}}\leq256\log(\timeHorizon-\switchRound)/\gamma_{\rewFnCommit,i}^2$. Taking the smaller of these two bounds proves the claim. 

Finally, splitting the expectation over $\mathcal{G}_{\rewFnCommit}$ and $\mathcal{G}_{\rewFnCommit}^c$ gives
\begin{align*}
\Regb{\switchRound}
&=\sum_{i\in\mathcal{S}_{\rewFn}}\Deltaf{i}\expect{n_i^{\mathrm I}}
\leq256\sum_{i\in\mathcal{S}_{\rewFn}}\frac{\Deltaf{i}\log(\timeHorizon-\switchRound)}{\max\{\gamma_{\rewFnCommit,i},\stopCriteria\}^2}
+B_{\stopCriteria}\prob{\mathcal{G}_{\rewFnCommit}^c}
=O\left(\sum_{i\in\mathcal{S}_{\rewFn}}\frac{\Deltaf{i}\log(\timeHorizon-\switchRound)}{\max\{\gamma_{\rewFnCommit,i},\stopCriteria\}^2}
+\frac{\armNum B_{\stopCriteria}\epochNum}{(\timeHorizon-\switchRound)^2}\right)~,
\end{align*}
which proves the lemma. 
\end{proof}

\begin{proof}[Proof for \Cref{lem:regret3}]
Use the event $\mathcal{G}_{\rewFnCommit}$ from \aiedit{Eqn.~\eqref{eqn:good-event-fg}}. Feasibility guarantees that every reserved epoch can be completed. If Stage I reaches a singleton on this event, that arm is commitment-optimal. Otherwise, an optimal commitment arm remains active through epoch $\epochNum$. Every arm $i$ surviving the final elimination then satisfies
$\Deltag{i}\leq
\bigl(\empiricalmeang{\optg,\epochNum}-\empiricalmeang{i,\epochNum}\bigr)+\myDelta{\epochNum}
\leq2\myDelta{\epochNum}\leq2\stopCriteria$.
Thus regret on $\mathcal{G}_{\rewFnCommit}$ is bounded by $(\timeHorizon-\switchRound)\max_{i:\Deltag{i}\leq2\stopCriteria}\Deltag{i}$. On its complement, commitment regret is at most $\timeHorizon-\switchRound$. Applying \aiedit{Eqn.~\eqref{eqn:good-event-fg}} gives
\[
\Regc{\switchRound,\timeHorizon}\leq (\timeHorizon-\switchRound)\max_{i:\Deltag{i}\leq2\stopCriteria}\Deltag{i}+\frac{2\armNum\epochNum}{\timeHorizon-\switchRound}~,
\]
as stated in the lemma.
\end{proof}

\begin{proof}[Proof of \Cref{prop:instance-dependent ub}]
Decompose regret into Stage II, Stage I, and commitment:
$\Reg{\switchRound,\timeHorizon}=\Rega{\switchRound}+\Regb{\switchRound}+\Regc{\switchRound,\timeHorizon}$.
Apply \Cref{lem:regret1,lem:regret2,lem:regret3}. Combining the two experiment-phase sums with $a+b\leq2\max\{a,b\}$ yields \aiedit{Eqn.~\eqref{eqn:regret}}. The Stage-II failure term is absorbed in \Cref{lem:regret1}. The two remaining failure terms are $O(\armNum B_{\stopCriteria}\epochNum/(\timeHorizon-\switchRound)^2)$ and $O(\armNum\epochNum/(\timeHorizon-\switchRound))$.

For a fixed instance, the chosen $\stopCriteria$ tends to zero under the horizon conditions stated before \aiedit{Eqn.~\eqref{eqn:regret_instance_dependent}}. Eventually $2\stopCriteria<\gamma_{\rewFnCommit,\optg}$, so the sampling denominators in \aiedit{Eqn.~\eqref{eqn:regret}} equal $\gamma_{\rewFnCommit,i}^2$. Moreover, $\epochNum=O(\log(\timeHorizon-\switchRound))$ and $B_{\stopCriteria}\leq16(\armNum\log(\timeHorizon-\switchRound))^{1/3}(\timeHorizon-\switchRound)^{2/3}$ imply that $\frac{\armNum\epochNum}{\timeHorizon-\switchRound}=O_{\instance}\left(\frac{\log(\timeHorizon-\switchRound)}{\timeHorizon-\switchRound}\right)=o(1)$ and $\frac{\armNum B_{\stopCriteria}\epochNum}{(\timeHorizon-\switchRound)^2}=O_{\instance}\left(\frac{\log^{4/3}(\timeHorizon-\switchRound)}{(\timeHorizon-\switchRound)^{4/3}}\right)=o(1)$.
This proves \aiedit{Eqn.~\eqref{eqn:regret_instance_dependent}}.
If $\stopCriteria\geq1$, Stage I is skipped. Splitting the Stage-II sampling counts at gap $\sqrt{\armNum\log\switchRound/\switchRound}$, as detailed in the proof of \Cref{thm:upper bound} below, gives $O(\sqrt{\armNum\switchRound\log\switchRound})$ experiment-phase regret. Bounding each commitment-period loss by $\max_{i\in[\armNum]}\Deltag{i}$ proves the separate bound in \Cref{prop:instance-dependent ub}.
\end{proof}

\thmub*

\begin{proof}[Proof for \Cref{thm:upper bound}]
First suppose $\stopCriteria<1$. Since $\myDelta{\epochNum}\geq\stopCriteria/2$, the choice in \Cref{thm:upper bound} implies that
$\armNum\roundlen{\rewFnCommit}{\epochNum}
\leq\frac{16\armNum\log(\timeHorizon-\switchRound)}{\stopCriteria^2}
\leq\switchRound$.
Thus every reserved epoch can be completed. Moreover, Stage I incurs at most
$B_{\stopCriteria}\leq16(\armNum\log(\timeHorizon-\switchRound))^{1/3}(\timeHorizon-\switchRound)^{2/3}$ regret on every sample path.

For Stage II, split arms at gap $d=\sqrt{\armNum\log\switchRound/\switchRound}$. On $\mathcal{G}_{\rewFn}$, pulls of arms with gap at most $d$ incur at most $\switchRound d$ regret, while the other arms contribute at most $O(\armNum\log\switchRound/d)$ by the sampling-count argument in \Cref{lem:regret1}. The failure contribution is $O(\armNum\log\switchRound/\switchRound)$ and is absorbed into $O(\sqrt{\armNum\switchRound\log\switchRound})$ when $\armNum\leq\switchRound$. If $d>1$ or $\armNum>\switchRound$, the trivial bound $\switchRound$ gives the same conclusion. This argument also applies when Stage I is skipped.

By \Cref{lem:regret3}, commitment regret is at most
$2(\timeHorizon-\switchRound)\stopCriteria+2\armNum\epochNum/(\timeHorizon-\switchRound)$.
Because $\stopCriteria\geq(\armNum\log(\timeHorizon-\switchRound)/(\timeHorizon-\switchRound))^{1/3}$ and $\stopCriteria<1$, we have $\epochNum=O(\log(\timeHorizon-\switchRound))$ and $\armNum\epochNum/(\timeHorizon-\switchRound)=O(1)$. Combining the three stages gives
\begin{equation}
\Reg{\switchRound,\timeHorizon}=O\Bigl(\sqrt{\armNum\switchRound\log\switchRound}
+(\armNum\log(\timeHorizon-\switchRound))^{1/3}(\timeHorizon-\switchRound)^{2/3}
+(\timeHorizon-\switchRound)\sqrt{\armNum\log(\timeHorizon-\switchRound)/\switchRound}+1\Bigr)~.
\label{eqn:regret_up_bnd_independent}
\end{equation}
If $\stopCriteria\geq1$, then $\epochNum=0$ and Stage I is skipped. Regret is at most the Stage-II bound plus $\timeHorizon-\switchRound$. The definition of $\stopCriteria$ implies
\[
\timeHorizon-\switchRound
\leq4(\timeHorizon-\switchRound)\sqrt{\frac{\armNum\log(\timeHorizon-\switchRound)}{\switchRound}}
+(\armNum\log(\timeHorizon-\switchRound))^{1/3}(\timeHorizon-\switchRound)^{2/3}~.
\]
Hence \aiedit{Eqn.~\eqref{eqn:regret_up_bnd_independent}} also holds in this case. For $\timeHorizon-\switchRound\leq1$, choosing $\stopCriteria=1$ gives the classical bandit bound plus at most one.

Suppressing logarithmic factors, in the regime $\armNum<\min\{\switchRound,\timeHorizon-\switchRound\}$, \aiedit{Eqn.~\eqref{eqn:regret_up_bnd_independent}} yields
\begin{align}
\Reg{\switchRound,\timeHorizon}&=\tildeO(\timeHorizon\sqrt{\armNum/\switchRound}),&&\switchRound\lesssim\armNum^{1/3}(\timeHorizon-\switchRound)^{2/3},
\notag\\
\Reg{\switchRound,\timeHorizon}&=\tildeO(\armNum^{1/3}(\timeHorizon-\switchRound)^{2/3}),&&\switchRound\gtrsim\armNum^{1/3}(\timeHorizon-\switchRound)^{2/3},\quad\timeHorizon-\switchRound\gtrsim\armNum^{1/4}\switchRound^{3/4},
\notag\\
\Reg{\switchRound,\timeHorizon}&=\tildeO(\sqrt{\armNum\timeHorizon}),&&\timeHorizon-\switchRound\lesssim\armNum^{1/4}\switchRound^{3/4}.
\notag
\end{align}
It is straightforward to verify that the compact bound in the theorem statement is equivalent to the above. 
\end{proof}

\section{Missing Proofs in \texorpdfstring{\Cref{sec:lower bound}}{}}

\thminsd*
\begin{proof}[Proof for \Cref{thm:instance dependent lower bound}]

Consider an instance $\instance=(\outcomeDist_j)_{j=1}^{\armNum}\in\mathcal{E}$, where pulling arm $j$ reveals an outcome realization drawn from $\outcomeDist_j$. Let $T_i(\switchRound)$ denote the number of times arm $i$ is pulled during the experiment phase.
Fix an arm $i\in\mathcal{S}_{\rewFn}$ and an arbitrary $\varepsilon>0$. By the definition of $d_{\rewFn,i}$, there exists $\outcomeDist_i'\in\mathcal{V}$ such that
\begin{align*}
    &D(\outcomeDist_i,\outcomeDist_i')\leq d_{\rewFn,i}+\varepsilon
    \quad\text{and}\quad
    \mu'_{\rewFn,i}\equiv\expect[\outcome\sim\outcomeDist_i']{\rewFn(\outcome)}>\mu_{\rewFn}^*~.
\end{align*}
Define the alternative instance $\instance\prime=(\outcomeDist_j')_{j=1}^{\armNum}$ by setting $\outcomeDist_j'=\outcomeDist_j$ for $j\neq i$. Notice that changing $\outcomeDist_i$ may change both induced $\rewFn$ and $\rewFnCommit$ reward distributions of arm $i$. This does not affect the argument because the divergence is calculated using the outcome samples observed by the policy. By the divergence decomposition lemma (see Section 15.1 in \citealp{LS-20}), we have
\begin{align}
    D(\PP_{\instance},\PP_{\instance\prime})
    =\expect[\instance]{T_i(\switchRound)}D(\outcomeDist_i,\outcomeDist_i')
    \leq\expect[\instance]{T_i(\switchRound)}(d_{\rewFn,i}+\varepsilon)~,
    \label{eq:change-measure-f}
\end{align}
where $\PP_{\instance}$ and $\PP_{\instance\prime}$ are the likelihoods of the observed outcomes and actions through the experiment phase.
Take the event $A_i=\{T_i(\switchRound)\geq\switchRound/2\}$. Under $\instance$, every pull of arm $i$ incurs experiment-phase regret $\Deltaf{i}$. Under $\instance\prime$, arm $i$ is optimal with respect to $\rewFn$, and every pull of another arm incurs regret at least $\mu'_{\rewFn,i}-\mu_{\rewFn}^*$. Therefore,
\begin{align*}
    \Reg[\textexp]{\switchRound\mid\instance}
    +\Reg[\textexp]{\switchRound\mid\instance\prime}
    & \geq
    \frac{\switchRound}{2}
    \min\left\{\Deltaf{i},\mu'_{\rewFn,i}-\mu_{\rewFn}^*\right\}
    \left(\prob[\instance]{A_i}+\prob[\instance\prime]{A_i^c}\right)\\
    &\geq
    \frac{\switchRound}{4}
    \min\left\{\Deltaf{i},\mu'_{\rewFn,i}-\mu_{\rewFn}^*\right\}
    \exp\left(-\expect[\instance]{T_i(\switchRound)}(d_{\rewFn,i}+\varepsilon)\right)~,
\end{align*}
where the last inequality follows from the Bretagnolle-Huber inequality and \aiedit{Eqn.~\eqref{eq:change-measure-f}}. Taking logarithms and using consistency over $\mathcal{E}$ gives
\begin{align*}
    &\liminf_{\timeHorizon-\switchRound\to\infty}
    \frac{\expect[\instance]{T_i(\switchRound)}}{\log(\switchRound)}
    \geq\frac{1}{d_{\rewFn,i}+\varepsilon}~.
\end{align*}
Since $\varepsilon>0$ is arbitrary, for every $i\in\mathcal{S}_{\rewFn}$, we have
\begin{align}
    \liminf_{\timeHorizon-\switchRound\to\infty}
    \frac{\expect[\instance]{T_i(\switchRound)}}{\log(\switchRound)}
    \geq\frac{1}{d_{\rewFn,i}}~.
    \notag
\end{align}

Then, fix an arm $i\in\mathcal{S}_{\rewFnCommit}$ and an arbitrary $\varepsilon>0$. By the definition of $d_{\rewFnCommit,i}$, choose a new alternative $\outcomeDist_i'\in\mathcal{V}$ such that
$D(\outcomeDist_i,\outcomeDist_i')\leq d_{\rewFnCommit,i}+\varepsilon$
and $\mu'_{\rewFnCommit,i}\equiv\expect[\outcome\sim\outcomeDist_i']{\rewFnCommit(\outcome)}>\mu_{\rewFnCommit}^*$.
Again define $\instance\prime$ by changing only the outcome distribution of arm $i$. The divergence decomposition lemma gives
\begin{align*}
    D(\PP_{\instance},\PP_{\instance\prime})
    \leq\expect[\instance]{T_i(\switchRound)}(d_{\rewFnCommit,i}+\varepsilon)~.
\end{align*}
Let $B_i=\{\commitArm_{\switchRound}=i\}$. Under $\instance$, event $B_i$ incurs commitment regret $(\timeHorizon-\switchRound)\Deltag{i}$. Under $\instance\prime$, event $B_i^c$ incurs commitment regret at least $(\timeHorizon-\switchRound)(\mu'_{\rewFnCommit,i}-\mu_{\rewFnCommit}^*)$. Thus,
\begin{align*}
    \Reg[\textcom]{\timeHorizon-\switchRound\mid\instance}
    +\Reg[\textcom]{\timeHorizon-\switchRound\mid\instance\prime}\geq
    \frac{\timeHorizon-\switchRound}{2}
    \min\left\{\Deltag{i},\mu'_{\rewFnCommit,i}-\mu_{\rewFnCommit}^*\right\}
    \exp\left(-\expect[\instance]{T_i(\switchRound)}(d_{\rewFnCommit,i}+\varepsilon)\right)~.
\end{align*}
Taking logarithms, using consistency over $\mathcal{E}$, and then letting $\varepsilon$ go to $0$ gives, for every $i\in\mathcal{S}_{\rewFnCommit}$,
\begin{align}
    &\liminf_{\timeHorizon-\switchRound\to\infty}
    \frac{\expect[\instance]{T_i(\switchRound)}}{\log(\timeHorizon-\switchRound)}
    \geq\frac{1}{d_{\rewFnCommit,i}}~.
    \label{eq:sample-lb-g}
\end{align}

It remains to consider $i=\optg$. Choose any admissible alternative $\outcomeDist_{\optg}'$ such that $\mu'_{\rewFnCommit,\optg}:=\expect[\outcome\sim\outcomeDist_{\optg}']{\rewFnCommit(\outcome)}<\max_{j\ne\optg}\mu_{\rewFnCommit,j}$, and let $\instance\prime$ differ only in this arm. For the same event $B_{\optg}=\{\commitArm_{\switchRound}=\optg\}$, the errors are now $B_{\optg}^c$ under $\instance$ and $B_{\optg}$ under $\instance\prime$. Their gaps are at least $\min_{j\ne\optg}\Deltag{j}$ and $\max_{j\ne\optg}\mu_{\rewFnCommit,j}-\mu'_{\rewFnCommit,\optg}$, respectively. Applying the same divergence decomposition and Bretagnolle--Huber inequality gives
\begin{align*}
&\Reg[\textcom]{\timeHorizon-\switchRound\mid\instance}
+\Reg[\textcom]{\timeHorizon-\switchRound\mid\instance\prime}\\
&\quad\geq\frac{\timeHorizon-\switchRound}{2}
\min\left\{\min_{j\ne\optg}\Deltag{j},\max_{j\ne\optg}\mu_{\rewFnCommit,j}-\mu'_{\rewFnCommit,\optg}\right\}
\myexp{-\expect[\instance]{T_{\optg}(\switchRound)}D(\outcomeDist_{\optg},\outcomeDist_{\optg}')}~.
\end{align*}
Taking logarithms and using the consistency definition therefore gives
\[
\liminf_{\timeHorizon-\switchRound\to\infty}
\frac{\expect[\instance]{T_{\optg}(\switchRound)}}{\log(\timeHorizon-\switchRound)}
\geq\frac{1}{D(\outcomeDist_{\optg},\outcomeDist_{\optg}')}~.
\]
Taking alternatives approaching the infimum extends \aiedit{Eqn.~\eqref{eq:sample-lb-g}} to $i=\optg$ with the $d_{\rewFnCommit,\optg}$.

Therefore, we prove that for every fixed $\eta>0$ and sufficiently large $\timeHorizon-\switchRound$,
\begin{align*}
    &\expect[\instance]{T_i(\switchRound)}
    \geq(1-\eta)\max\left\{
    \frac{\log(\switchRound)}{d_{\rewFn,i}},
    \frac{\log(\timeHorizon-\switchRound)}{d_{\rewFnCommit,i}}
    \right\},
    \quad i\in\mathcal{S}_{\rewFn}~.
\end{align*}
Finally, the total regret is at least the experiment-phase regret. Therefore,
\begin{align*}
    \Reg[]{\switchRound,\timeHorizon\mid\instance}
    \geq\Reg[\textexp]{\switchRound\mid\instance}
    & =\sum\nolimits_{i\in\mathcal{S}_{\rewFn}}
    \expect[\instance]{T_i(\switchRound)}\Deltaf{i}\geq
    (1-o(1))
    \sum\nolimits_{i\in\mathcal{S}_{\rewFn}}
    \max\left\{
    \frac{\log(\switchRound)}{d_{\rewFn,i}},
    \frac{\log(\timeHorizon-\switchRound)}{d_{\rewFnCommit,i}}
    \right\}\Deltaf{i}~.
\end{align*}
This proves \Cref{thm:instance dependent lower bound}.
\end{proof}

\thminsind*
\begin{proof}[Proof for \Cref{thm:instance independent lower bound}]
We prove the lower bound on the fixed class $\mathcal{E}_{\mathrm{coord}}$ in \Cref{thm:instance independent lower bound}. Thus, every outcome takes the form $o=(x,y)$, the known reward functions are fixed as $\rewFn(o)=x$ and $\rewFnCommit(o)=y$, and the two coordinates are independent Bernoulli random variables. The mean values below are written after subtracting a common interior baseline (the constant baseline is used to guarantee that all the mean values are bounded away from the $0$ and $1$ boundary). Adding the same baseline to every arm leaves all gaps and regrets unchanged and ensures that every Bernoulli mean remains in the fixed interval defining $\mathcal{E}_{\mathrm{coord}}$.
We consider an instance where in the \textit{experiment phase}, the mean reward vector has form $\mu_{\rewFn}=\left(\alpha+\delta_{\rewFn},\alpha,\ldots,\alpha,0,\ldots,0\right)$ where the last $\armNum/2$ arms have zero mean, and in the \textit{commitment phase}, the mean reward vector has form $\mu_{\rewFnCommit}=\left(0,\ldots,0,\delta_{\rewFnCommit},0,\ldots,0\right)$ where $\delta_{\rewFnCommit}$ is on the $k$th arm's position and $k\equiv \armNum/2+1$, $\alpha$ is some positive constant.
We group the first $\armNum/2$ arms as arm group $A$ and the other half of the arms as arm group $B$.
Given policy $\pi$, let $\expect[\instance]{T_A}$ and $\expect[\instance]{T_B}$ be the expected number of pulls distributed to group $A$ and $B$, respectively.
Clearly, $\expect[\instance]{T_A}+\expect[\instance]{T_B}=\switchRound$.
Let $\expect[\instance]{T_l}$ be policy $\pi$'s expected number of pulls of arm $l$.
Then, we define $i^\dagger=\argmin_{l\in A\setminus\{1\}}\expect[\instance]{T_l}$ and $j^\dagger=\argmin_{l\in B\setminus\{k\}}\expect[\instance]{T_l}$.
Clearly, $\expect[\instance]{T_{i^\dagger}}\leq\frac{\expect[\instance]{T_A}}{\frac{\armNum}{2}-1}$ and $\expect[\instance]{T_{j^\dagger}}\leq\frac{\expect[\instance]{T_B}}{\frac{\armNum}{2}-1}$.
Now, we construct an alternative instance with $\mu_{\rewFn}^{\dagger}=\left(\alpha+\delta_{\rewFn},\alpha,\ldots,\alpha+2\delta_{\rewFn},\ldots,\alpha,0,\ldots,0\right)$ where arm $i^\dagger$'s position has value $\alpha+2\delta_{\rewFn}$, and $\mu_{\rewFnCommit}^{\dagger}=\left(0,\ldots,0,\delta_{\rewFnCommit},0,\ldots,2\delta_{\rewFnCommit},\ldots,0\right)$ where arm $j^\dagger$'s position has value $2\delta_{\rewFnCommit}$.
Parameters $\alpha, \delta_{\rewFn}$ and $\delta_{\rewFnCommit}$ will be specified shortly.
{\allowdisplaybreaks
\begin{align}
    \notag
    & \Reg[\pi]{\switchRound, \timeHorizon\mid \instance}+ \Reg[\pi]{\switchRound, \timeHorizon\mid \instance^\dagger}\\
    \notag
    \geq~ & \expect[\instance]{T_B}\cdot\alpha+\expect[\instance^\dagger]{T_B}\cdot\alpha+\prob[\instance]{T_1\leq\frac{\switchRound}{2}}\cdot\frac{\switchRound\cdot\delta_{\rewFn}}{2}+\prob[\instance^\dagger]{T_1>\frac{\switchRound}{2}}\cdot\frac{\switchRound\cdot\delta_{\rewFn}}{2}\\
    \notag
    & +\prob[\instance]{\commitArm_{\switchRound}\neq k}\cdot(\timeHorizon - \switchRound)\cdot\delta_{\rewFnCommit}+\prob[\instance^\dagger]{\commitArm_{\switchRound}=k}\cdot(\timeHorizon - \switchRound)\cdot\delta_{\rewFnCommit}\\
\notag
    \geq ~&
    \expect[\instance]{T_B}\cdot\alpha+\frac{\switchRound\cdot\delta_{\rewFn}}{2}\cdot\frac12\exp\left(-D\left(\PP_{\instance},\PP_{\instance^\dagger}\right)\right)+(\timeHorizon - \switchRound)\cdot\delta_{\rewFnCommit}\cdot\frac12\exp\left(-D\left(\PP_{\instance},\PP_{\instance^\dagger}\right)\right)\\
\notag
    = ~ &\expect[\instance]{T_B}\cdot\alpha+\left(\frac{\switchRound\cdot\delta_{\rewFn}}{2}+(\timeHorizon - \switchRound)\cdot\delta_{\rewFnCommit}\right)\cdot\frac12\exp\left(-\expect[\instance]{T_{i^\dagger}}\cdot D\left(\alpha,\alpha+2\delta_{\rewFn}\right)-\expect[\instance]{T_{j^\dagger}}\cdot D\left(0,2\delta_{\rewFnCommit}\right)\right)\\
    \mathrel{=} ~ & \Omega\left(\expect[\instance]{T_B}+\left(\switchRound\cdot\delta_{\rewFn}+(\timeHorizon - \switchRound)\cdot\delta_{\rewFnCommit}\right)\cdot\exp\left(-C\cdot\frac{\switchRound-\expect[\instance]{T_B}}{\armNum}\cdot(\delta_{\rewFn})^2-\frac{\expect[\instance]{T_B}}{\armNum}\cdot(\delta_{\rewFnCommit})^2\right)\right)~,
    \label{eqn:lower_bnd_minimax}
\end{align}}
where the second inequality again uses the \textit{Bretagnolle-Huber inequality}, and the last asymptotic estimate uses the fact that for many common distributions like Gaussian, $D(a,b)\approx(a-b)^2$.
And in the last equality $C$ is some constant.
For some distributions like Bernoulli, we just need to raise $\mu_{\rewFn}$ and $\mu_{\rewFnCommit}$ up by a small constant such that they are bounded away from $0$ and $1$, and we will again get $D(a,b)\approx(a-b)^2$.
For \aiedit{Eqn.~\eqref{eqn:lower_bnd_minimax}}, we now keep $\expect[\instance]{T_B}\in[0,\switchRound]$ arbitrary and specify $\delta_{\rewFn}$ and $\delta_{\rewFnCommit}$ to derive lower bounds that hold uniformly for every policy.

First, let $\delta_{\rewFn}=0$ and $\delta_{\rewFnCommit}=\sqrt{\frac{\armNum}{\switchRound}}$. Since $\expect[\instance]{T_B}\leq\switchRound$, \aiedit{Eqn.~\eqref{eqn:lower_bnd_minimax}} gives
\begin{align*}
    &\Reg[\policy]{\switchRound,\timeHorizon\mid\instance}+\Reg[\policy]{\switchRound,\timeHorizon\mid\instance^\dagger}\mathrel{=}\Omega\left((\timeHorizon-\switchRound)\sqrt{\frac{\armNum}{\switchRound}}\cdot
    \exp\left(-\frac{\expect[\instance]{T_B}}{\switchRound}\right)\right)
    \mathrel{=}\Omega\left(\sqrt{\frac{\armNum\cdot(\timeHorizon-\switchRound)^2}{\switchRound}}\right)~.
\end{align*}
Next, let $\delta_{\rewFn}=0$ and $\delta_{\rewFnCommit}=\left(\frac{\armNum}{\timeHorizon-\switchRound}\right)^{1/3}$. Then \aiedit{Eqn.~\eqref{eqn:lower_bnd_minimax}} gives
\begin{align*}
    &\Reg[\policy]{\switchRound,\timeHorizon\mid\instance}+\Reg[\policy]{\switchRound,\timeHorizon\mid\instance^\dagger}\mathrel{=}\Omega\left(\expect[\instance]{T_B}
    +\armNum^{1/3}(\timeHorizon-\switchRound)^{2/3}\cdot\exp\left(-\frac{\expect[\instance]{T_B}}{\armNum^{1/3}(\timeHorizon-\switchRound)^{2/3}}\right)\right)~.
\end{align*}
If $\expect[\instance]{T_B}\geq\armNum^{1/3}(\timeHorizon-\switchRound)^{2/3}$, the first term gives a lower bound of $\Omega\left(\armNum^{1/3}(\timeHorizon-\switchRound)^{2/3}\right)$. Otherwise, the exponential term is at least a positive constant and gives the same lower bound. Therefore, uniformly for all $\expect[\instance]{T_B}$, the sum of regrets in \aiedit{Eqn.~\eqref{eqn:lower_bnd_minimax}} is $\Omega\left(\armNum^{1/3}(\timeHorizon-\switchRound)^{2/3}\right)$.

Finally, let $\delta_{\rewFn}=\sqrt{\frac{\armNum}{\switchRound}}$ and $\delta_{\rewFnCommit}=0$. Since $\switchRound-\expect[\instance]{T_B}\leq\switchRound$, we have
\begin{align*}
    &\Reg[\policy]{\switchRound,\timeHorizon\mid\instance}+\Reg[\policy]{\switchRound,\timeHorizon\mid\instance^\dagger}\mathrel{=}\Omega\left(\sqrt{\armNum\switchRound}\cdot
    \exp\left(-\frac{\switchRound-\expect[\instance]{T_B}}{\switchRound}\right)\right)\mathrel{=}\Omega\left(\sqrt{\armNum\switchRound}\right)~.
\end{align*}

Therefore, we find that when $\switchRound<\armNum^{1/3}(\timeHorizon-\switchRound)^{2/3}$, the instance-independent lower bound is $\Omega\left(\sqrt{\frac{\armNum\cdot(\timeHorizon-\switchRound)^2}{\switchRound}}\right)$.
When $\switchRound\geq\armNum^{1/3}(\timeHorizon-\switchRound)^{2/3}$, taking the worse of the last two hard instances gives the instance-independent lower bound 
$\Omega\left(\max\left\{\armNum^{1/2}\switchRound^{1/2},
    \armNum^{1/3}(\timeHorizon-\switchRound)^{2/3}\right\}\right)$.
When $(\timeHorizon-\switchRound)\geq\armNum^{1/4}\switchRound^{3/4}$, the second term in the maximum dominates; otherwise, the first term dominates.

In summary, the instance-independent lower bound has form $\Omega\left(\sqrt{\frac{\armNum(\timeHorizon - \switchRound)^{2/3}\cdot\max\left\{(\timeHorizon - \switchRound)^{4/3},\armNum^{1/3}\switchRound\right\}}{\min\left\{\switchRound, \armNum^{1/3}(\timeHorizon - \switchRound)^{2/3}\right\}}}\right)$.
\end{proof}

\section{Missing Proofs in \texorpdfstring{\Cref{sec:extensions}}{}}

\subsection{Regret under a Perturbed Affine Relationship}
\label{subsec:perturb}

For $\gradientUB>0$, define the following accuracy parameters: $\stopCriteria_{\switchRound}=8\sqrt{\frac{\armNum\log(\timeHorizon)}{\switchRound}}$, $\stopCriteria_B=\left(\frac{\armNum\log(\timeHorizon)}{\timeHorizon-\switchRound}\right)^{1/3}$, $\stopCriteria_{\gradientUB}=\sqrt{\frac{\armNum\log(\timeHorizon)}{\gradientUB(\timeHorizon-\switchRound)}}$, $\stopCriteria_{\noiseBound}=\left(\frac{\armNum\noiseBound\log(\timeHorizon)}{\gradientUB(\timeHorizon-\switchRound)}\right)^{1/3}$ and $\stopCriteria_S=\max\{\stopCriteria_{\gradientUB},\stopCriteria_{\noiseBound}\}$.
Here, $\stopCriteria_{\switchRound}$ ensures completion of the stage I exploration of {\RAEC}; $\stopCriteria_B$ balances the commitment regret against the exploration cost;
$\stopCriteria_{\gradientUB}$ and $\stopCriteria_{\noiseBound}$ balance the commitment regret against the cost induced by the ranking-preserving component $\gradientUB\rewFn$, and the cost induced by the ranking-changing perturbation, respectively.

\begin{theorem}[Regret under a perturbed affine relationship]
\label{thm:full-spectrum-perturbed}

Suppose $\switchRound>1$ and $\timeHorizon-\switchRound>1$. For $\gradientUB>0$, let $\stopCriteria_{\mathrm{base}}$ be the target accuracy in \Cref{thm:upper bound}, and set the target accuracy based on prior knowledge (denoted by the subscript $\mathrm{pk}$) to
$\stopCriteria_{\mathrm{pk}}=\max\{\stopCriteria_{\switchRound},\min\{\stopCriteria_B,\stopCriteria_S\}\}$.
Then
\[
\Reg{\switchRound,\timeHorizon}=\tildeO\bigl(\min\{R_{\mathrm{base}},R_{\mathrm{cap}},R_{\mathrm{pk}}\}\bigr)~,
\]
where $R_{\mathrm{base}}=\sqrt{\frac{\armNum(\timeHorizon-\switchRound)^{2/3}\max\{(\timeHorizon-\switchRound)^{4/3},\armNum^{1/3}\switchRound\}}{\min\{\switchRound,\armNum^{1/3}(\timeHorizon-\switchRound)^{2/3}\}}}$, $R_{\mathrm{cap}}=\sqrt{\armNum\switchRound}+(\timeHorizon-\switchRound)\min\{\gradientUB+\noiseBound,1\}$, and $R_{\mathrm{pk}}=\sqrt{\armNum\switchRound}+(\timeHorizon-\switchRound)\stopCriteria_{\mathrm{pk}}$.
The three candidate bounds hold for $\stopCriteria=\stopCriteria_{\mathrm{base}}$, $1$, and $\stopCriteria_{\mathrm{pk}}$, respectively. Choose the candidate with the smallest expression before observing any outcomes. The baseline candidate uses \aiedit{Eqn.~\eqref{eq:rounds defn}}. For the candidate $R_{\mathrm{pk}}$, we set 
$\roundlen{\rewFnCommit}{\epoch}=\left\lceil\frac{4\log(\timeHorizon)}{\elimiCriteria_\epoch^2}\right\rceil$, for $\epoch\in[\epochNum]$,
with the Stage-II same as in \aiedit{Eqn.~\eqref{eq:rounds defn}}.

\end{theorem}

\begin{proof}
We use the means and suboptimality gaps defined before \Cref{prop:instance-dependent ub}. 
For every arm $i$, the perturbed affine relationship implies that 
$\gradientUB\Deltaf{i}-\noiseBound\leq\Deltag{i}\leq\gradientUB\Deltaf{i}+\noiseBound$.
Consequently, every commitment gap is at most $\min\{\gradientUB+\noiseBound,1\}$. 
For $\stopCriteria=1$, there is no stage I exploration, the regret is upper bounded as follows,
\[
\Reg{\switchRound,\timeHorizon}=O\bigl(\sqrt{\armNum\switchRound\log\switchRound}+(\timeHorizon-\switchRound)\min\{\gradientUB+\noiseBound,1\}\bigr)~.
\]
The baseline choice $\stopCriteria=\stopCriteria_{\mathrm{base}}$ leads to the regret upper bound in \Cref{thm:upper bound}. These arguments also prove both cases with $\gradientUB=0$.

Now let $\gradientUB>0$. 
For $\stopCriteria_{\switchRound}\leq\stopCriteria<1$, the Stage-I sampling targets for $R_{\mathrm{pk}}$ give
$\armNum\roundlen{\rewFnCommit}{\epochNum}
\leq\frac{16\armNum\log(\timeHorizon)}{\stopCriteria^2}+\armNum
\leq\frac{\switchRound}{4}+\armNum<\switchRound$.
The last inequality follows from $\armNum<\switchRound/(64\log(\timeHorizon))$ and $\log(\timeHorizon)>1$. Thus stage I exploration can be completed and $B_{\stopCriteria}\leq\switchRound$.
For $\roundlen{\rewFnCommit}{\epoch}=\lceil4\log(\timeHorizon)/\elimiCriteria_\epoch^2\rceil$, the concentration argument in \Cref{lem:regret2} bounds the estimation failure probability for each arm and epoch by $2/\timeHorizon^2$. Taking a union bound over $\armNum$ arms and $\epochNum$ epochs gives
$\prob{\mathcal{G}_{\rewFnCommit}^c}\leq2\armNum\epochNum/\timeHorizon^2$.
The resulting Stage-I and commitment failure contributions are at most
$2\armNum\epochNum B_{\stopCriteria}/\timeHorizon^2$ and $2\armNum\epochNum(\timeHorizon-\switchRound)/\timeHorizon^2$, respectively.
Here, to exploit the perturbed affine relationship, we bound Stage-I regret on $\mathcal{G}_{\rewFnCommit}$ and $\mathcal{G}_{\rewFnCommit}^c$ separately. 

For the Stage-I pull count $n_i^{\mathrm I}$ used in the proof of \Cref{lem:regret2}, every commitment-optimal arm satisfies
$n_i^{\mathrm I}\leq\roundlen{\rewFnCommit}{\epochNum}
\leq\frac{16\log(\timeHorizon)}{\stopCriteria^2}+1$,
because Stage I ends after at most $\epochNum$ epochs and $\myDelta{\epochNum}\geq\stopCriteria/2$.
On $\mathcal{G}_{\rewFnCommit}$, commitment-optimal arms are never eliminated. Every commitment-suboptimal arm $i$ is eliminated by the end of epoch $\marginroundg{i}$ from the proof of \Cref{lem:regret2}, provided that epoch is reached, since $\myDelta{\marginroundg{i}}\leq\Deltag{i}/4$. If $\marginroundg{i}>\epochNum$, Stage I ends first. Hence
$n_i^{\mathrm I}\leq\roundlen{\rewFnCommit}{\min\{\marginroundg{i},\epochNum\}}
\leq\frac{256\log(\timeHorizon)}{\max\{\Deltag{i},\stopCriteria\}^2}+1$.
Since $\Deltag{i}=0$ for the $\rewFnCommit$-optimal arm, both cases give
$n_i^{\mathrm I}=O(\log(\timeHorizon)/\max\{\Deltag{i},\stopCriteria\}^2)$ on $\mathcal{G}_{\rewFnCommit}$.
Moreover, every arm remaining in $\armg{\epochNum+1}$ has $\Deltag{i}\leq2\myDelta{\epochNum}\leq2\stopCriteria$ on this event, by the argument in \Cref{lem:regret3}.
Combining these bounds with the Stage-II bound in \Cref{lem:regret1}, we have 
\[
\begin{aligned}
\Reg{\switchRound,\timeHorizon}=O\Biggl(&\sum_{i\in\mathcal{S}_{\rewFn}}\Deltaf{i}\max\left\{\frac{\log(\timeHorizon)}{\max\{\Deltag{i},\stopCriteria\}^2},\frac{\log\switchRound}{\Deltaf{i}^2}\right\}
+(\timeHorizon-\switchRound)\max_{i:\Deltag{i}\leq2\stopCriteria}\Deltag{i}
+\frac{\armNum\epochNum(\timeHorizon-\switchRound+B_{\stopCriteria})}{\timeHorizon^2}\Biggr)~.
\end{aligned}
\]

For Stage-II, we split arms at gap $\sqrt{\armNum\log\switchRound/\switchRound}$ as in the proof of \Cref{thm:upper bound}, which leads to $O(\sqrt{\armNum\switchRound\log\switchRound})$ regret. On $\mathcal{G}_{\rewFnCommit}$, the Stage-I contribution is bounded using $\Deltaf{i}\leq1$ and
\[
\frac{\Deltaf{i}}{\max\{\Deltag{i},\stopCriteria\}^2}
\leq\frac{\Deltag{i}+\noiseBound}{\gradientUB\max\{\Deltag{i},\stopCriteria\}^2}
\leq\frac{\stopCriteria+\noiseBound}{\gradientUB\stopCriteria^2}~.
\]
Combining these bounds with the experiment cap and the commitment bound, we get 
\[
\begin{aligned}
\Reg{\switchRound,\timeHorizon}=O\Biggl(&\sqrt{\armNum\switchRound\log\switchRound}+(\timeHorizon-\switchRound)\stopCriteria
+\min\left\{\switchRound,\frac{\armNum\log(\timeHorizon)}{\stopCriteria^2},\frac{\armNum\log(\timeHorizon)(\stopCriteria+\noiseBound)}{\gradientUB\stopCriteria^2}\right\}
+\frac{\armNum\epochNum(\timeHorizon-\switchRound+B_{\stopCriteria})}{\timeHorizon^2}\Biggr)~.
\end{aligned}
\]
Notice that $\epochNum=O(\log\switchRound)$, $\armNum<\switchRound/(64\log(\timeHorizon))$ and $B_{\stopCriteria}\leq\switchRound$, we know
$\frac{\armNum\epochNum(\timeHorizon-\switchRound+B_{\stopCriteria})}{\timeHorizon^2}
\leq\frac{\armNum\epochNum}{\timeHorizon}
=O\left(\frac{\switchRound\log\switchRound}{\timeHorizon\log(\timeHorizon)}\right)
=O(1)$,
which is absorbed by the Stage-II bound.

Set $\stopCriteria=\stopCriteria_{\mathrm{pk}}<1$. If $\stopCriteria_B\leq\stopCriteria_S$, then
\[
\frac{\armNum\log(\timeHorizon)}{\stopCriteria_{\mathrm{pk}}^2}
\leq(\timeHorizon-\switchRound)\stopCriteria_B
\leq(\timeHorizon-\switchRound)\stopCriteria_{\mathrm{pk}}~.
\]
If $\stopCriteria_S<\stopCriteria_B$, then $\stopCriteria_{\mathrm{pk}}\geq\stopCriteria_S\geq\max\{\stopCriteria_{\gradientUB},\stopCriteria_{\noiseBound}\}$, and
\[
\frac{\armNum\log(\timeHorizon)}{\gradientUB\stopCriteria_{\mathrm{pk}}}
+\frac{\armNum\noiseBound\log(\timeHorizon)}{\gradientUB\stopCriteria_{\mathrm{pk}}^2}
\leq2(\timeHorizon-\switchRound)\stopCriteria_{\mathrm{pk}}~.
\]
Both cases give $O(\sqrt{\armNum\switchRound\log\switchRound}+(\timeHorizon-\switchRound)\stopCriteria_{\mathrm{pk}})$ regret.
If $\stopCriteria_{\mathrm{pk}}\geq1$, Stage I is skipped, and commitment regret is at most $\timeHorizon-\switchRound\leq(\timeHorizon-\switchRound)\stopCriteria_{\mathrm{pk}}$, giving the same bound.
Thus choosing $\stopCriteria$ for the smallest corresponding regret expression in advance proves the theorem.
\end{proof}

In the following implications, we suppress logarithmic and constant factors in the accuracy targets and regret bounds, and use the predetermined tuning selection in \Cref{thm:full-spectrum-perturbed}.
In particular, if \(\gradientUB>0\) and \(\noiseBound=0\), then the transformation is purely
ranking-preserving, \(\stopCriteria_{\noiseBound}=0\), and
$
\stopCriteria_S=\stopCriteria_{\gradientUB}
=
\sqrt{\frac{\armNum}{\gradientUB(\timeHorizon-\switchRound)}}.
$
In this case,
\begin{align*}
\begin{aligned}
&\Reg{\switchRound,\timeHorizon}\\
&\mathrel{=}
\tildeO\Bigg(
\min\Bigg\{
R_{\mathrm{base}},
\sqrt{\armNum\switchRound}+(\timeHorizon-\switchRound)\min\{\gradientUB,1\},
\\
& 
\quad\sqrt{\armNum\switchRound}
+
\max\left\{
(\timeHorizon-\switchRound)\sqrt{\frac{\armNum}{\switchRound}},
\min\left[
\armNum^{1/3}(\timeHorizon-\switchRound)^{2/3},
\sqrt{\frac{\armNum(\timeHorizon-\switchRound)}{\gradientUB}}
\right]
\right\}
\Bigg\}
\Bigg)\\
&\mathrel{=}
\tildeO\Bigg(
\min\Bigg\{
R_{\mathrm{base}},
\sqrt{\armNum\switchRound}+(\timeHorizon-\switchRound)\min\{\gradientUB,1\},
\sqrt{\armNum\switchRound}
+
\max\left\{
(\timeHorizon-\switchRound)\sqrt{\frac{\armNum}{\switchRound}},
\sqrt{\frac{\armNum(\timeHorizon-\switchRound)}{\gradientUB}}
\right\}
\Bigg\}
\Bigg)~.
\end{aligned}
\end{align*}
Therefore, for constant \(\gradientUB>0\) and \(\noiseBound=0\), we get
\begin{align*}
\Reg{\switchRound,\timeHorizon}
\mathrel{=}
\left\{
\begin{aligned}
&\tildeO\left(
\sqrt{\armNum\switchRound}+(\timeHorizon-\switchRound)\sqrt{\frac{\armNum}{\switchRound}}
\right)~,~
&\text{for }
\switchRound\leq\gradientUB (\timeHorizon-\switchRound)~,
\\
&\tildeO\left(
\sqrt{\armNum\switchRound}+\sqrt{\armNum(\timeHorizon-\switchRound)}
\right)~,~
&\text{for }
\switchRound>\gradientUB (\timeHorizon-\switchRound)~.
\end{aligned}
\right.
\end{align*}
For example, in the balanced regime with $\switchRound=\frac{\gradientUB}{1+\gradientUB}\cdot\timeHorizon$, the upper bound becomes $\tildeO\left(\sqrt{\armNum\timeHorizon}\right)$, which becomes the familiar bound in the standard multi-armed bandit problem.

When $\gradientUB=0<\noiseBound$, every commitment gap is at most $\min\{\noiseBound,1\}$. Choosing between $\stopCriteria=\stopCriteria_{\mathrm{base}}$ and $\stopCriteria=1$ leads regret $\tildeO(\min\{R_{\mathrm{base}},\sqrt{\armNum\switchRound}+(\timeHorizon-\switchRound)\min\{\noiseBound,1\}\})$. When $\gradientUB=\noiseBound=0$, $\rewFnCommit\equiv0$, so skipping Stage I gives the standard bandit bound.


\subsection{Missing Algorithms and Proofs in \Cref{subsec:concave}}

\begin{algorithm}[h]
\caption{Reserved Online Stochastic Convex Optimization for Commitment ({\ROSCOC})}
\label{algo:roscoc}
\begin{algorithmic}[1]
        \State \textbf{Input:} A set of arms $\{1,2,\ldots,\armNum\}$, $\switchRound$, $\timeHorizon$, 
        $\tau$;
        \State \textbf{Initialization:} Set $\history_0=\emptyset$, $\elimiCriteria_1=1/2$, $\armf{1}=[\armNum]$.
        \State 
        {\em Whenever $\switchRound$ rounds are exhausted in Stage I or II, the algorithm enters the Commitment Stage.}
\For{$t = 1,\ldots,\tau$}
        \hfill  
        \Comment{Stage I: Reserved online stochastic convex optimization for $\rewFnCommit$}
\State Implement \textit{Online Stochastic Convex Optimization} algorithm for function $\rewFnCommit$.
\State Record the execution path $\history_t=[\pickArm_1,\ldots,\pickArm_t]$.
\EndFor
        \For{$\epoch = 1,2,\ldots$}
        \hfill  
        \Comment{{ Stage II: Arm eliminations for reward function $\rewFn$}} 
        \State 
        \parbox[t]{\dimexpr\linewidth-\algorithmicindent}{%
        Sample each arm in $\armf{\epoch}$ until the total number of times it has been chosen is $\roundlen{\rewFn}{\epoch}$ times or reaches the experiment cap $\switchRound$ and stops.}
        \State 
        \parbox[t]{\dimexpr\linewidth-\algorithmicindent}{%
        At the end of epoch $\epoch$, for each arm $i\in[\armNum]$, compute the empirical average reward $\empiricalmeanf{i,\epoch}$ for reward function $\rewFn$.}
        \State 
        Update 
        $\armf{\epoch+1}\gets\left\{i\in\armf{\epoch}:\max_{j\in\armf{\epoch}}\empiricalmeanf{j, \epoch} -\empiricalmeanf{i, \epoch}\leq \elimiCriteria_\epoch\right\}$.
        \State Set $\elimiCriteria_{\epoch+1}\gets\elimiCriteria_\epoch/2$.
        \EndFor
\State Uniformly sample $\timeHorizon-\switchRound$ elements from $\history_{\tau}$ with replacement.
Denote the sample set as sequence $S=[\pickArm_1,\ldots,\pickArm_{\timeHorizon-\switchRound}]$.
        \hfill 
        \Comment{{ Commitment Stage: Execution-history-induced portfolio}}
\For{$t = \switchRound+1,\ldots,\timeHorizon$}
\State Pull the $(t-\switchRound)-$th arm in $S$.
\EndFor
\end{algorithmic}
\end{algorithm}
In Stage I of the above algorithm, we use the Fenchel-dual bandit algorithm in Algorithm 3 of \cite{AD-2014}, specialized to the Euclidean norm and a finite set of arms.

\begin{algorithm}[h]
\caption{Online Stochastic Convex Optimization}
\label{alg:oco}
\begin{algorithmic}[1]
\State \textbf{Initialization:} Choose $x_1\in\mathcal{X}:=\{x:\lVert x\rVert_2\leq \lipschitzConst\}$, e.g., $x_1=0$.
\State \algcomment{Below $\operatorname{sign}(x_t)$ is the vector of coordinate-wise signs. Following \cite{AD-2014}, $n_{t,i}$ is one plus the number of pulls of arm $i$ before round $t$.}
\State For every arm $i\in[\armNum]$, set $n_{1,i}=1$, $\bar{\outcome}_{1,i}=0$, and
$\tilde{\outcome}_{1,i}=\Pi_{[0,1]^{\feaDimen}}\left(\bar{\outcome}_{1,i}-\operatorname{sign}(x_1)\cdot\sqrt{\ln(\tau)/n_{1,i}}\right)$, where $\Pi_{[0,1]^{\feaDimen}}(\cdot)$ clips the vector to $[0,1]^{\feaDimen}$. 
\For{$t=1,\ldots,\tau$}
\State $\pickArm_t=\argmax_{i\in[\armNum]}\rewFnCommit^*(x_t)-x_t^T\tilde{\outcome}_{t,i}$. Here, $\rewFnCommit^*(x)\equiv\max_{y\in[0,1]^{\feaDimen}} \left\{x^Ty+\rewFnCommit(y)\right\}$, is the Fenchel dual of $\rewFnCommit$.
\State Pull arm $\pickArm_t$ and observe outcome $\outcome_{t,\pickArm_t}$.
\State Update $n_{t+1,i}=n_{t,i}+\mathbf{1}\{i=\pickArm_t\}$ and
$\bar{\outcome}_{t+1,i}=\frac{n_{t,i}\bar{\outcome}_{t,i}+\mathbf{1}\{i=\pickArm_t\}\outcome_{t,\pickArm_t}}{n_{t+1,i}}$ for all $i\in[\armNum]$.
\State Update $x_{t+1}$ by applying \textit{Online Gradient Descent} (OGD) to $h_t(x)=\rewFnCommit^*(x)-x^T\tilde{\outcome}_{t,\pickArm_t}$ for one step starting from $x_t$ and clip to $\mathcal{X}$.
\State Update $\tilde{\outcome}_{t+1, i}=\Pi_{[0,1]^{\feaDimen}}\left(\bar{\outcome}_{t+1,i}-\operatorname{sign}(x_{t+1})\cdot\sqrt{\frac{\ln(\tau)}{n_{t+1,i}}}\right)$ for all $i\in[\armNum]$.
\EndFor
\end{algorithmic}
\end{algorithm}

\thmregretconcave*

\begin{proof}[Proof for \Cref{thm:regret concave}]
We first state the result that we use from \cite{AD-2014}. The Stage I subroutine in \Cref{alg:oco} is their Algorithm 3 adapted to our problem. Applying their Theorem 5.1 with $m=\armNum$, $p=q=2$, $\gamma=O(\log(\tau))$, and $\lVert\boldsymbol{1}_{\feaDimen}\rVert_2=\sqrt{\feaDimen}$ gives
\begin{equation*}
\tau\left(\OPT_{\rewFnCommit}-\expect{\rewFnCommit\left(\frac{\sum\nolimits_{t\in[\tau]}\outcome_{t,\pickArm_t}}{\tau}\right)}\right)
\mathrel{=} O\left(\lipschitzConst\sqrt{\armNum\feaDimen\tau\log(\tau)}+R_c(\tau)\right)~.
\end{equation*}
The OGD specialization immediately following Theorem 5.1 of \cite{AD-2014} gives $R_c(\tau)=\widetilde{O}(GD\sqrt{\tau})$, where $G$ denotes a uniform upper bound on the Euclidean norm of a subgradient of the per-round OCO loss $\rewFnCommit^*(x_t)-x_t^T\tilde{\outcome}_{t,i}$.
$D$ is the radius of the dual feasible region, and in our paper is upper bounded by $\lipschitzConst$ (see Assumption 1 in \citealp{AD-2014} for the equivalence for Lipschitz concave functions).
The subgradient of the per-round OCO loss has the form $y-\tilde{\outcome}_{t,i}$, therefore in our setting, $G\mathrel{=} O\left(\sqrt{\feaDimen}\right)$ and the radius of the dual feasible region $D\leq\lipschitzConst$, so
\begin{equation}
\tau\left(\OPT_{\rewFnCommit}-\expect{\rewFnCommit\left(\frac{\sum\nolimits_{t\in[\tau]}\outcome_{t,\pickArm_t}}{\tau}\right)}\right)
\mathrel{=}\widetilde{O}\left(\lipschitzConst\sqrt{\armNum\feaDimen\tau}+\lipschitzConst\sqrt{\feaDimen\tau}\right)~.
\label{eqn:concave_commit_AD}
\end{equation}
We next relate the Stage I execution path to its induced portfolio. Let $\mathcal{F}_{t-1}$ be the history and algorithmic randomization available before arm $\pickArm_t$ is selected. Then
\begin{align*}
S_{\tau}=\sum\nolimits_{t\in[\tau]}\left(\outcome_{t,\pickArm_t}-\expect{\outcome_{t,\pickArm_t}\mid\mathcal{F}_{t-1}}\right)~,
\end{align*}
is a vector martingale with bounded differences. Pinelis' inequality (\cite{pinelis1994optimum} Theorem 3.5, the vector version of Azuma's inequality) and the Lipschitz property therefore give
\begin{align}
\left|\expect{\rewFnCommit\left(\frac{\sum\nolimits_{t\in[\tau]}\outcome_{t,\pickArm_t}}{\tau}\right)}-\expect{\rewFnCommit\left(\frac{\sum\nolimits_{t\in[\tau]}\expect{\outcome_{t,\pickArm_t}\mid\mathcal{F}_{t-1}}}{\tau}\right)}\right|
\mathrel{=}\widetilde{O}\left(\lipschitzConst\sqrt{\frac{\feaDimen}{\tau}}\right)~.
\label{eqn:concave_commit1}
\end{align}
For the commitment phase, the arms are sampled uniformly with replacement from the Stage I execution path. Conditional on $\mathcal{F}_{\tau}$, these selected arm indices are i.i.d., the subsequent outcomes are conditionally independent, and every outcome has a conditional mean equal to the average of the arm means along the Stage I path. Applying Pinelis' inequality and the Lipschitz property yields
\begin{align}
\expect{\rewFnCommit\left(\frac{\sum\nolimits_{t\in[\tau]}\expect{\outcome_{t,\pickArm_t}\mid\mathcal{F}_{t-1}}}{\tau}\right)}-\expect{\rewFnCommit\left(\frac{\sum_{t=\switchRound+1}^{\timeHorizon}\outcome_{t,\pickArm_t}}{\timeHorizon-\switchRound}\right)}
\mathrel{=}\widetilde{O}\left(\lipschitzConst\sqrt{\frac{\feaDimen}{\timeHorizon-\switchRound}}\right)~.
\label{eqn:concave_commit2}
\end{align}
Stage I incurs at most $\tau$ amount of experiment-phase regret, while Stage II incurs $O(\sqrt{\armNum(\switchRound-\tau)\log(\switchRound-\tau)})$. Combining these two bounds with \aiedit{Eqn.~\eqref{eqn:concave_commit_AD}}-\aiedit{Eqn.~\eqref{eqn:concave_commit2}}, we obtain
\begin{align*}
\Reg{\switchRound,\timeHorizon}
\mathrel{=}\widetilde{O}\Bigg(&\tau+\sqrt{\armNum(\switchRound-\tau)}
+(\timeHorizon-\switchRound)\left(\lipschitzConst\sqrt{\frac{\armNum\feaDimen}{\tau}}+\lipschitzConst\sqrt{\frac{\feaDimen}{\tau}}+\lipschitzConst\sqrt{\frac{\feaDimen}{\timeHorizon-\switchRound}}\right)\Bigg)~.
\end{align*}
Setting
\begin{align*}
\tau=\min\left\{\switchRound,\lipschitzConst^{\sfrac{2}{3}}\armNum^{\sfrac{1}{3}}\feaDimen^{\sfrac{1}{3}}(\timeHorizon-\switchRound)^{\sfrac{2}{3}}\log(\timeHorizon-\switchRound)^{\sfrac{1}{3}}\right\}~,
\end{align*}
and considering whether the minimum is attained by $\switchRound$ gives
\begin{align*}
\Reg{\switchRound,\timeHorizon}
\mathrel{=}\widetilde{O}\left(\timeHorizon\cdot\lipschitzConst\sqrt{\frac{\armNum\feaDimen}{\switchRound}}+\sqrt{\armNum\switchRound}+\lipschitzConst^{\sfrac{2}{3}}\armNum^{\sfrac{1}{3}}\feaDimen^{\sfrac{1}{3}}(\timeHorizon-\switchRound)^{\sfrac{2}{3}}\right)~,
\end{align*}
which proves the stated upper bound.
Finally, the lower bound in \Cref{thm:instance independent lower bound} also applies here since the basic model is a special case of the generalized concave model.
\end{proof}

\end{document}